%% file: main.tex
\documentclass{article}

\usepackage{microtype}
\usepackage{graphicx}
\usepackage{subcaption}
\usepackage{booktabs} % for professional tables

\usepackage{hyperref}

\usepackage[preprint]{icml2026}

\usepackage{amsmath}
\usepackage{amssymb}
\usepackage{mathtools}
\usepackage{amsthm}

\usepackage[capitalize,noabbrev]{cleveref}

\theoremstyle{plain}
\newtheorem{theorem}{Theorem}[section]
\newtheorem{proposition}[theorem]{Proposition}
\newtheorem{lemma}[theorem]{Lemma}

\theoremstyle{definition}
\newtheorem{definition}[theorem]{Definition}

\theoremstyle{remark}
\newtheorem{remark}[theorem]{Remark}

\newtheorem{example}{Example}
 
\usepackage{algorithm}
\usepackage{algorithmic} 
\usepackage{array}
\newcolumntype{C}[1]{>{\centering\arraybackslash}p{#1}}

\usepackage[textsize=tiny]{todonotes}

\icmltitlerunning{From LLM Conjectures to Lean Formalizations: Automated Inequality Proving via Sum-of-Squares Certificates}

\usepackage[most]{tcolorbox}
\tcbset{
  colback=gray!5!white, % 背景颜色
  colframe=blue!50!black, % 边框颜色
  coltext=black, % 字体颜色
  boxrule=0.8pt, % 边框粗细
  arc=3mm, % 圆角半径
  left=2mm, right=2mm, top=1mm, bottom=1mm, % 内边距
  fonttitle=\bfseries,
  title style={color=blue!40!black},
}

\usepackage{listings}
\usepackage{xcolor}

\lstdefinelanguage{lean4}{
    keywords={theorem, lemma, def, let, have, by, fun, match, 
              import, open, where, if, then, else, return,
              },
    keywordstyle=\bfseries\color{blue!80!black},
    morecomment=[l]{--},           % 单行注释
    morecomment=[s]{/--}{-/},      % 块注释
    commentstyle=\itshape\color{green!50!black},
    stringstyle=\color{red},
    basicstyle=\small\ttfamily,
    breaklines=true,
    frame=single,
    mathescape=true,
    sensitive=true
}

\usepackage[most]{tcolorbox}
\usepackage{multirow} 
\usepackage{pifont}
\usepackage{multirow}

\begin{document}
% 从LLM猜想到形式化证明：基于神经符号SOS表示的自动化不等式证明
\twocolumn[
% \icmltitle{From Heuristic Conjecture to Formal Proof: Automated Inequality Proving via Neuro-Symbolic SOS Recovery}
% \icmltitle{From LLM Conjecture to Formal Proof: Automated Inequality Proving via Neuro-Symbolic SOS Representation}
%%%% 
% \icmltitle{From LLM-Generated Conjectures to Lean Formalizations: Automated Proofs of Polynomial Inequalities via Sum-of-Squares Certificates}
%%%% 
\icmltitle{From LLM-Generated Conjectures to Lean Formalizations: Automated Polynomial Inequality Proving via Sum-of-Squares Certificates}
% \icmltitle{From LLM-Generated Conjectures to Lean Formalizations: Automated Inequality Proving via Neuro-Symbolic Sum-of-Squares Certificates}
% \icmltitle{From LLM-Generated Conjectures to Lean Formalizations: A Neuro-Symbolic Approach for Automated Polynomial Inequality Proving}



\icmlsetsymbol{equal}{*}

\begin{icmlauthorlist}
\icmlauthor{Ruobing Zuo}{ecnu}
\icmlauthor{Hanrui Zhao}{nudt}
\icmlauthor{Gaolei He}{ecnu}
\icmlauthor{Zhengfeng Yang}{ecnu}
\icmlauthor{Jianlin Wang}{hnu}

% \icmlauthor{Firstname1 Lastname1}{equal,yyy}
% \icmlauthor{Firstname2 Lastname2}{equal,yyy,comp}
% \icmlauthor{Firstname3 Lastname3}{comp}
\end{icmlauthorlist}

\icmlaffiliation{ecnu}{School of Software Engineering, East China Normal University, Shanghai, China}
\icmlaffiliation{nudt}{College of Computer Science and Technology, National University of Defense Technology, Changsha, China}
\icmlaffiliation{hnu}{School of Computer and Information Engineering, Henan University, Kaifeng, China}

\icmlcorrespondingauthor{Zhengfeng Yang}{zfyang@sei.ecnu.edu.cn}

\icmlkeywords{Machine Learning, ICML}

\vskip 0.3in
]

\printAffiliationsAndNotice{}  % no special notice (required even if empty)

\begin{abstract} 

Automated proving of polynomial inequalities is a fundamental challenge in automated mathematical reasoning, where rich algebraic structure and a rapidly growing certificate search space hinder scalability. Purely symbolic approaches provide strong guarantees but often scale poorly as the number of variables or the degree increases, due to expensive algebraic manipulations and rapidly growing intermediate expressions. In parallel, LLM-guided methods have made notable progress, particularly on competition-style inequalities with a small number of variables. To address the remaining scalability challenges, we propose 
NSPI,  a neuro-symbolic framework that combines the complementary strengths of LLMs and symbolic computation for polynomial-inequality proving. Concretely, an LLM proposes a conjecture in the form of an approximate polynomial Sum-Of-Squares (SOS) decomposition; we refine it via  symbolic computation to obtain an exact polynomial SOS representation, which directly proves the target inequality, and we further certify the proof in Lean, yielding an end-to-end pipeline from heuristic discovery to machine-checked proof. 
%Experiments on challenging high-dimensional benchmarks (up to 10 variables) show substantial improvements in success rate and efficiency over competitive baselines.
Experiments on challenging benchmarks involving polynomials with up to 10 variables demonstrate the effectiveness and scalability of the proposed method.
\end{abstract}

\input{Sec/Intro}

\input{Sec/Related}

\input{Sec/Pre}

\input{Sec/Method}
\input{Sec/Exp}

\section{Conclusion} 
% In this paper, we propose Neuro-Symbolic Approach for Polynomial Inequality Proving (NSPI), which combines the structural learning capabilities of LLMs with the precise computational advantages of symbolic methods to automatically generate end-to-end formal proof. Specifically, NSPI reformulates the inequality proving problem as a sum-of-squares (SOS) decomposition problem, ultimately generating a complete formal proof through three modules: Neural Conjecture, Symbolic Correction, and Lean Verification. In the Neural Conjecture module, we introduce four algebraic operation-based methods for synthesizing SOS data and train an LLM as an SOS structure conjecturer through a progressive two-stage training process, enabling the direct generation of SOS structure conjectures. The Symbolic Correction module obtains a verifiable and exact SOS representation through a Newton iteration and rational recovery process. Finally, the Lean Verification module generates a complete Lean proof. Extensive experimental results demonstrate the effectiveness of the proposed NSPI method, (****), extending the boundaries of polynomial inequality proving.

In this paper, we presented NSPI, a neuro-symbolic framework for automated polynomial-inequality proving that combines the complementary strengths of large language models and symbolic computation to provide an end-to-end pipeline from conjecture to certified proof. NSPI leverages an LLM to propose approximate Sum-Of-Squares (SOS) decompositions, and refines them via symbolic computation into exact SOS representations that directly prove the target inequalities; we then machine-check these proofs in Lean. 
Extensive experiments on challenging benchmarks with polynomials of up to 10 variables show that NSPI consistently improves the success rate and efficiency over competitive baseline methods, substantially broadening the practical scope of automated polynomial-inequality proving.


% Authors are kindly asked to make their submissions as accessible as possible
% for everyone including people with disabilities and sensory or neurological
% differences. Tips of how to achieve this and what to pay attention to will be
% provided on the conference website \url{http://icml.cc/}.

% \section*{Software and Data}

% If a paper is accepted, we strongly encourage the publication of software and
% data with the camera-ready version of the paper whenever appropriate. This can
% be done by including a URL in the camera-ready copy. However, \textbf{do not}
% include URLs that reveal your institution or identity in your submission for
% review. Instead, provide an anonymous URL or upload the material as
% ``Supplementary Material'' into the OpenReview reviewing system. Note that
% reviewers are not required to look at this material when writing their review.

% % Acknowledgements should only appear in the accepted version.
% \section*{Acknowledgements}

% \textbf{Do not} include acknowledgements in the initial version of the paper

% If a paper is accepted, the final camera-ready version can (and usually should)
% include acknowledgements.  Such acknowledgements should be placed at the end of
% the section, in an unnumbered section that does not count towards the paper
% page limit. Typically, this will include thanks to reviewers who gave useful
% comments, to colleagues who contributed to the ideas, and to funding agencies
% and corporate sponsors that provided financial support.

\section*{Acknowledgements}
% This work was supported in part by the National Key Re-
% search and Development Program of China under Grant
% 2023YFA1009402, the Strategic Priority Research Pro-
% gram of Chinese Academy of Sciences under Grant
% XDA0480501.

This work was supported in part by the National Key Research and Development Program of China under Grant 2023YFA1009402, the Strategic Priority Research Program of Chinese Academy of Sciences under Grant XDA0480501, and the Natural Science Foundation of Hunan Province under Grant 2026JJ70102.

% \textbf{Do not} include acknowledgements in the initial version of the paper

% If a paper is accepted, the final camera-ready version can (and usually should)
% include acknowledgements.  Such acknowledgements should be placed at the end of
% the section, in an unnumbered section that does not count towards the paper
% page limit. Typically, this will include thanks to reviewers who gave useful
% comments, to colleagues who contributed to the ideas, and to funding agencies
% and corporate sponsors that provided financial support.

\section*{Impact Statement}
This paper presents work whose goal is to the field of Machine Learning. 
The development of our approach holds potential for significant impact within the domains of formal verification and automated polynomial inequality proving. While there are various potential societal consequences of this work, none
which we feel must be specifically highlighted here. 

% Authors are \textbf{required} to include a statement of the potential broader
% impact of their work, including its ethical aspects and future societal
% consequences. This statement should be in an unnumbered section at the end of
% the paper (co-located with Acknowledgements -- the two may appear in either
% order, but both must be before References), and does not count toward the paper
% page limit. In many cases, where the ethical impacts and expected societal
% implications are those that are well established when advancing the field of
% Machine Learning, substantial discussion is not required, and a simple
% statement such as the following will suffice:

% ``This paper presents work whose goal is to advance the field of Machine
% Learning. There are many potential societal consequences of our work, none
% which we feel must be specifically highlighted here.''

% The above statement can be used verbatim in such cases, but we encourage
% authors to think about whether there is content which does warrant further
% discussion, as this statement will be apparent if the paper is later flagged
% for ethics review.

\nocite{langley00}

\bibliography{reference}
\bibliographystyle{icml2026}
\input{Sec/Appendix}

%%%%%%%%%%%%%%%%%%%%%%%%%%%%%%%%%%%%%%%%%%%%%%%%%%%%%%%%%%%%%%%%%%%%%%%%%%%%%%%
%%%%%%%%%%%%%%%%%%%%%%%%%%%%%%%%%%%%%%%%%%%%%%%%%%%%%%%%%%%%%%%%%%%%%%%%%%%%%%%
% APPENDIX
%%%%%%%%%%%%%%%%%%%%%%%%%%%%%%%%%%%%%%%%%%%%%%%%%%%%%%%%%%%%%%%%%%%%%%%%%%%%%%%
%%%%%%%%%%%%%%%%%%%%%%%%%%%%%%%%%%%%%%%%%%%%%%%%%%%%%%%%%%%%%%%%%%%%%%%%%%%%%%%
\end{document}

%% file: Sec/Intro.tex
\section{Introduction}\label{Intro}
% 语言待斟酌 参考文献待检查 换ref   %TODO
% 英文表述待精进
% Polynomial inequalities have widespread applications in fields such as optimization~\cite{2012Exact_certification_in_global}, control theory~\cite{2000Parrilo_Pablo_book}, and combinatorics~\cite{marechal2015polyhedral}. 
% The automation of competition-level inequality proving serves as a crucial benchmark for assessing the boundaries of artificial intelligence in mathematical reasoning~\cite{trinh2024alphageometry,wei2024AIPS,he-etal-2024-olympiadbench,liineqsearch}. 
% In recent years, %% TODO 近年来？
% works such as AlphaGeometry ~\cite{trinh2024alphageometry,chervonyi2025gold},  %%todo alphaGeometry 2
% % AlphaProof ~\cite{hubert2025olympiad} %结合RL AlphaZero启发 no 没有LLM
% and Seed-Prover ~\cite{chen2025seedprover, chen2025seedprover15} have 
% % achieved silver medal results 
% demonstrated medal-level %remarkable 
% performance in the 
% % IMO competition, 
% International Mathematical Olympiad (IMO) competition, 
% highlighting the substantial potential of large language models (LLMs) in theorem proving tasks. 
% However, automated inequality proving remains a challenging frontier, as it typically involves lengthy reasoning steps, an infinite search space, and highly complex computational processes, especially for high-dimensional and difficult inequalities with many variables.  %%
%%%%%%%%zhr-0129
Polynomial inequalities play a fundamental role in areas such as optimization, control, and combinatorics~\cite{2012Exact_certification_in_global,2000Parrilo_Pablo_book,marechal2015polyhedral}. 
The automation of 
% competition-level 
complex 
inequality proving has recently emerged as an important benchmark for evaluating the limits of AI in mathematical reasoning~\cite{trinh2024alphageometry,wei2024AIPS,he-etal-2024-olympiadbench,liineqsearch}. 
Recent systems such as AlphaGeometry~\cite{trinh2024alphageometry,chervonyi2025gold} and Seed-Prover~\cite{chen2025seedprover, chen2025seedprover15} have demonstrated impressive performance on Olympiad-level problems, highlighting the potential of large language models (LLMs) for theorem proving. 
Nevertheless, automated inequality proving remains highly challenging due to the long reasoning chains, vast search spaces, and substantial computational complexity involved, especially for high-dimensional and multivariate cases.

% Symbolic computation has long been a cornerstone for proving polynomial inequalities~\cite{1999overview,lasserre2002semidefinite,article,2023verifyrealroots}. A widely used approach is based on Sum-of-Squares (SOS) decomposition~\cite{2012Exact_certification_in_global}, which transforms the nonnegativity of a polynomial into a semidefinite programming (SDP) problem. Modern computer algebra systems~\cite{heck1993maple,de2008z3,meurer2017sympy} also provide symbolic operations such as polynomial equation solving. However, purely symbolic methods often struggle to produce human-readable reasoning steps and are prone to combinatorial explosion, limiting their scalability.
% Recently, large language model (LLM)-based approaches for automated formal theorem proving have achieved remarkable progress~\cite{lample2022hypertree, xin2025bfs,ren2025deepseek,lin2025goedel,wang2025kimina,lin2025goedelproverv2}, particularly 
% when integrated with proof assistants such as Lean~\cite{de2015lean} and Isabelle~\cite{paulson1990isabelle}. Methods based on complete proof generation~\cite{ren2025deepseek,lin2025goedel,wang2025kimina,lin2025goedelproverv2} and tree-structured proof search~\cite{lample2022hypertree, xin2025bfs} continually enhance the theorem-proving capabilities of LLMs. Despite these advances, the scarcity of formalized training data limits their performance on complex, competition-level algebraic inequalities.

Symbolic computation has long been a cornerstone for proving polynomial inequalities~\cite{1999overview,lasserre2002semidefinite,article,2023verifyrealroots}. A widely used approach is based on Sum-of-Squares (SOS) decomposition~\cite{2012Exact_certification_in_global}, which transforms the nonnegativity of a polynomial into a semidefinite programming (SDP) problem. 
Modern computer algebra systems further support such pipelines with basic algebraic operations~\cite{heck1993maple,de2008z3,meurer2017sympy}. 
However, purely symbolic approaches often suffer from poor scalability due to combinatorial explosion and typically fail to produce structured, human-readable proofs.
In parallel, recent LLM-based methods have substantially advanced automated formal theorem proving~\cite{lample2022hypertree,xin2025bfs,ren2025deepseek,lin2025goedel,wang2025kimina,lin2025goedelproverv2} integrated with proof assistants such as Lean~\cite{de2015lean} and Isabelle~\cite{paulson1990isabelle}. 
Nevertheless, their performance on complex 
% , competition-level 
algebraic inequalities remains limited by the scarcity formalized training data.

% To address the above challenges, a promising approach is to integrate symbolic methods, 
% thereby combining the strengths of structured reasoning and symbolic precision while mitigating dependence on large-scale formalized training data~\cite{2016Solving_and_Verifying,trinh2024alphageometry,wei2024AIPS,li2025LIPS}.
% However, most existing methods (e.g., AIPS~\cite{wei2024AIPS}) focus on ternary and quaternary polynomials, leaving significant scalability challenges for higher-dimensional cases. 
% Furthermore, in such hybrid systems, 
% LLMs are predominantly employed to optimize the strategy selection of symbolic solvers, 
% whereas their potential for direct symbolic conjecture remains largely unexplored. 
% In this paper,  we propose a novel neuro-symbolic approach for automated polynomial inequality proving,  which explores the scalability of combining LLMs with symbolic methods for automated theorem proving in unconstrained polynomial inequality scenarios, particularly in complex multivariate settings. 
% Our approach formulates inequality proving as an SOS-solving task, where LLM-generated conjectures are refined through symbolic computation and formal verification to produce rigorous, formally certified proofs.
% By combining the strengths of neural conjecturing and symbolic reasoning, our method provides a scalable and reliable solution for challenging polynomial inequalities, advancing the capabilities of neuro-symbolic methods in automated theorem proving.
To address the above challenges, a promising approach is to integrate neural and symbolic methods,
thereby combining the strengths of structured reasoning and symbolic precision while mitigating dependence on large-scale formalized training data~\cite{2016Solving_and_Verifying,trinh2024alphageometry,wei2024AIPS,li2025LIPS}.
% However, most existing methods (e.g., AIPS~\cite{wei2024AIPS}) focus on ternary and quaternary polynomials, leaving significant scalability challenges for higher-dimensional cases. 
% Furthermore, in such hybrid systems, 
% LLMs are predominantly employed to optimize the strategy selection of symbolic solvers, 
% whereas their potential for direct symbolic conjecture remains largely unexplored. 
% However, existing approaches~\cite{wei2024AIPS,li2025LIPS} %(e.g., AIPS~\cite{wei2024AIPS}) 
% are still limited in two key aspects: they mainly target low-dimensional cases (e.g., ternary or quaternary polynomials) and typically employ LLMs only for strategy guidance rather than for direct symbolic conjecturing. 
% In this paper,  we propose a novel neuro-symbolic approach for automated polynomial inequality proving,  which explores the scalability of combining LLMs with symbolic methods for automated theorem proving in unconstrained polynomial inequality scenarios, particularly in complex multivariate settings. 
% Our approach formulates inequality proving as an SOS-solving task, where LLM-generated conjectures are refined through symbolic computation and formal verification to produce rigorous, formally certified proofs.
% By combining the strengths of neural conjecturing and symbolic reasoning, our method provides a scalable and reliable solution for challenging polynomial inequalities, advancing the capabilities of neuro-symbolic methods in automated theorem proving.
However, existing approaches (e.g., AIPS~\cite{wei2024AIPS}) remain limited in two key aspects: they mainly target low-dimensional cases (e.g., ternary or quaternary polynomials), and typically restrict the role of LLMs to search guidance or strategy selection, since directly using LLMs to generate symbolic conjectures remains difficult to control and certify. 
In this paper, we propose a new neuro-symbolic framework for polynomial inequality proving that targets unconstrained polynomial inequality scenarios, high-dimensional multivariate problems and elevates LLMs to primary conjecture generators. 
By formulating inequality proving as SOS-based certification and tightly coupling LLM-driven hypothesis generation with symbolic refinement and formal verification, our approach establishes an end-to-end pipeline from heuristic discovery to certified proof, significantly extending the scope of neuro-symbolic automated theorem proving.

% In summary, o
% The main contributions can be summarized as follows:
% \begin{itemize}
%     \item We propose a neuro-symbolic approach for automated polynomial inequality proving, 
%     which automatically generates complete formal proofs of inequalities through a pipeline of neural conjecture, symbolic correction, and Lean verification. 
%     \item We explore a new paradigm in which LLM serve as a core symbolic conjecture engine. 
%     By leveraging structural priors generated by LLMs after progressive two-stage training, 
%     %%TODO 两阶段训练过程 命名方式是否合适？ 直接说？参考其他文章 
%     we extend automated inequality proving to cases involving more variables, thereby broadening the frontier of neuro-symbolic methods in this domain. 
%     % for proving polynomial inequalities. 
%     % We explore a new paradigm in which LLMs serve as the core symbolic conjecturer, leveraging models refined through a two-stage training process to provide essential structural priors. 
%     % This work extends existing research on automated inequality proving to cases involving more variables, effectively broadening  the boundaries of neuro-symbolic methods for proving polynomial inequalities. 
%     \item Extensive experimental results on 522 competition-level inequality problems validate the effectiveness of the proposed method. Compared to symbolic computation-based approaches and LLM-assisted methods, our method achieves the best results on inequality problems with up to 10 variables.
% \end{itemize}

The main contributions can be summarized as follows:
\begin{itemize}
    \item We propose a neuro-symbolic approach for automated polynomial inequality proving, 
    which automatically generates complete formal proofs of inequalities through a pipeline of neural conjecture, symbolic correction, and Lean verification. 
    \item We develop a principled reliability bridge that integrates LLM-based heuristic conjecture generation with symbolic exact certification, transforming neural conjectures into machine-checkable proofs and enabling automated inequality proving to scale to higher-dimensional multivariate cases.
    \item Extensive experiments on $522$ %competition-level
    challenging 
    inequality problems demonstrate the effectiveness of the proposed method, which outperforms both symbolic computation-based and LLM-assisted approaches, especially on problems with up to $10$ variables.

\end{itemize}

%% file: Sec/Related.tex
% %% TODO: 具体表述待修改 + 待精炼 + 待修改本文工作表述
% TODO  段末补充本文工作介绍
\section{Related Work}\label{Related}
% ATP   
% Ineq-Comp
% 和IneqSearch相似的是都希望通过将多项式分解成非负分量从而完成证明。 
%  ?（在他们的工作中，LLM承担策略搜索的工作。本文不是让LLM承担搜索工作，而是作为一个类人智能体在证明过程中承担重要角色，作为一个SOS分解猜想的提出者 ？ 区别在于，想探索真正LLM和符号计算结合能否提高不等式的自动证明能力 （希望 探索【通过大量数据训练后的LLM作为SOS结构的直接猜想器是否具有正确的直觉】？）。   
%  ?（本文方法更加端到端，没有以来演绎推理策略。LLM作为猜想者承担的工作要重一些，所以我们强调了神经模块的数据构造和训练过程）
% \textbf{Symbolic Methods for Inequality Proving.} 
% Polynomial inequality proving has long been regarded as a computationally challenging problem. A classical approach is based on the 
% sum-of-squares (SOS) methodology~\cite{2012Exact_certification_in_global}, 
% which reduces nonnegativity certification to the existence of an SOS decomposition and further to solving a semidefinite program (SDP), as exemplified by 
% RealCertify~\cite{magron2018realcertify} and 
% ValidSDP~\cite{martin2017validsdp}. 
% Meanwhile, computer algebra systems such as 
% Maple~\cite{heck1993maple}, 
% Z3~\cite{de2008z3}, 
% % , Mathematica~\cite{}, 
% and 
% SymPy~\cite{meurer2017sympy} provide fundamental symbolic capabilities—including factorization, elimination, solving polynomial equations, and algebraic simplification, that support algebraic preprocessing and manipulation in inequality-proving pipelines. Nevertheless, purely symbolic approaches typically struggle to produce human-readable reasoning steps, and they often suffer from combinatorial explosion as the problem dimension increases, particularly for competition-level mathematical inequalities.

% %%%%0121-zhr
\textbf{Symbolic Methods for Inequality Proving.} 
Polynomial inequality proving has traditionally been approached through symbolic computation. 
% where certifying polynomial nonnegativity remains a computationally challenging task. 
A classical method is based on the SOS methodology~\cite{2012Exact_certification_in_global,martin2017validsdp}, which reduces nonnegativity certification to the existence of an SOS decomposition and further to solving SDP problem. %%%TODO ref 
% , as exemplified by 
% % RealCertify~\cite{magron2018realcertify} and 
% ValidSDP~\cite{martin2017validsdp}. 
Meanwhile, computer algebra systems such as Maple~\cite{heck1993maple}, Z3~\cite{de2008z3} 
and SymPy~\cite{meurer2017sympy} 
provide fundamental symbolic capabilities, 
% including factorization, elimination, polynomial equation solving, and algebraic simplification, 
which support algebraic preprocessing and manipulation in inequality-proving pipelines. 
Nevertheless, purely symbolic approaches typically struggle to produce human-readable reasoning steps and often suffer from combinatorial explosion as the problem dimension increases, particularly for multivariate and algebraically intensive polynomial inequalities.
% % CAD
% % TODO : 待确认＋补充最新工作   
% % TODO: 添加 【RL相关】 motivation   
% % AlphaProof

\textbf{LLM-based Formal Theorem Proving.}   
In recent years, LLM-based automated theorem proving has advanced rapidly. 
% A dominant paradigm integrates 
Various approaches integrate LLMs 
% large language models 
with interactive proof assistants such as 
Lean~\cite{de2015lean} 
% , Coq~\cite{barras1997coq}, and 
% Isabelle~\cite{paulson1990isabelle} 
to produce machine-checkable formal proofs. 
One line of work fine-tunes models on large-scale corpora of formal proofs to generate proof strategies or local 
proof tactics
~\cite{ polu2020generative, lample2022hypertree, xin2025bfs}. %gpt-f  %hypertree  %  deepseek-prover-v1.5 % InternLM2.5-Step Prover % BFS prover  Goedel-prover v2?
% STP??
Another line explores end-to-end generation of 
complete formal proofs~\cite{ren2025deepseek,lin2025goedel,wang2025kimina,lin2025goedelproverv2},   
%% deepseek-prover v2   goedel-prover    kimina-prover
exemplified by Goedel-Prover~\cite{lin2025goedel}  % 可以展开写两句
,Kimina Prover~\cite{wang2025kimina} and DeepSeek-Prover-V2~\cite{ren2025deepseek}. 
% In addition, some studies treat the LLM as 
% % a value function to 
% the guider of  %%% TODO待修改 
% % guide 
% tactic selection~\cite{ xin2025bfs}, 
% often combined with search procedures such as Monte Carlo Tree Search (MCTS) to complete proofs. %%%TODO待修改
% Other approaches incorporate natural-language proofs or leverage error feedback returned by interactive theorem provers to iteratively refine candidate proofs ~\cite{jiang2022draft,wang2023lego,lin2407leanSTAR}. %%TODO 待检查
However, LLM-based methods are limited by the scarcity and uneven quality of formal-proof data, and thus remain weak on high-dimensional inequalities.
\textbf{Neuro-Symbolic Theorem Proving.}  
To bridge the scalability limitations of purely symbolic methods and the data bottleneck of purely neural approaches, recent work has explored neuro-symbolic integration for automated theorem proving~\cite{trinh2024alphageometry,wei2024AIPS,li2025LIPS,chervonyi2025gold}. 
These methods typically combine neural models with symbolic solvers, using learning-based components to guide or prioritize symbolic reasoning or derivation steps. 
Representative systems such as AlphaGeometry~\cite{trinh2024alphageometry}, AIPS~\cite{wei2024AIPS}, and LIPS~\cite{li2025LIPS} demonstrate the effectiveness of this paradigm in geometry and algebraic inequality proving. 
In contrast, while existing approaches mainly focus on learning-guided derivation or strategy selection, we treat LLMs as generators of symbolic conjectures and machine-checkable formal verification. 
Centered on verifiable SOS certificates, our framework establishes an end-to-end neuro-symbolic pipeline for tackling more complex multivariate polynomial inequalities.

%% file: Sec/Pre.tex
\section{Preliminaries}
\label{sec:preliminary}

%\subsection{Problem Statement}
%\label{sec:problem_statement}

\begin{figure*}[htbp]
  \centering
  \includegraphics[width=0.9\linewidth]{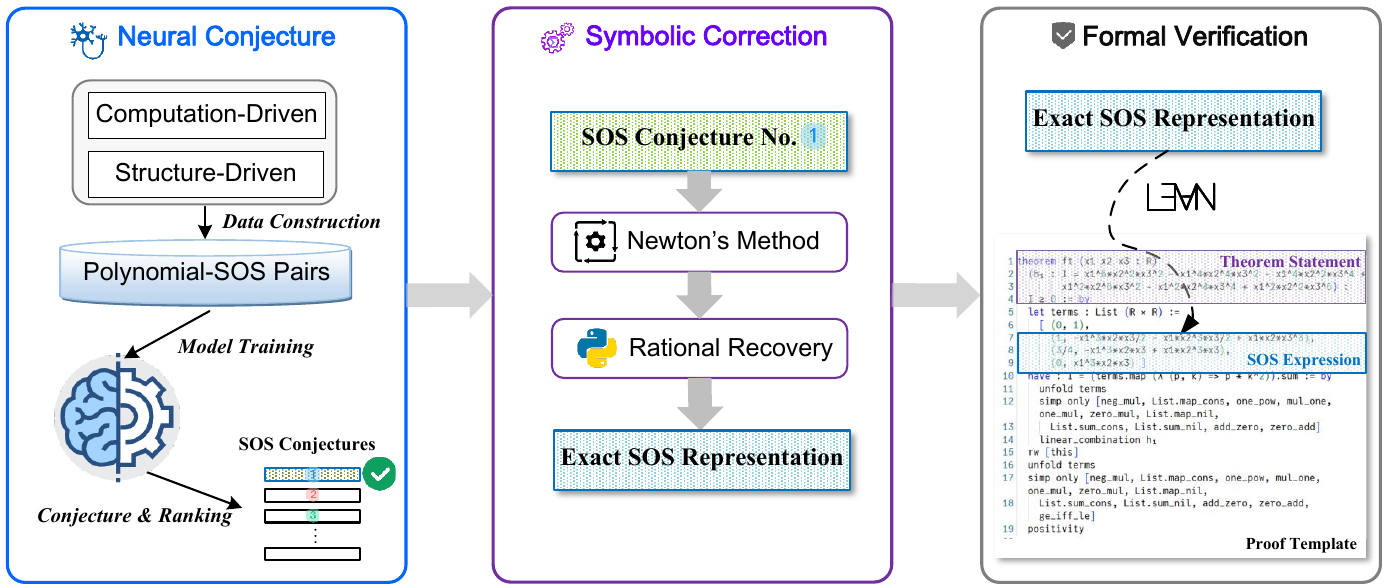}%{figures/overview/overview_251103_v1.pdf}
  \caption{\textbf{Overview of \textbf{Ne}uro-\textbf{Sy}mbolic \textbf{S}OS-based Polynomial \textbf{I}nequality \textbf{P}roving (\textbf{NSPI})}. (1) \textbf{Neural Conjecture Module}: 
  Non-negative polynomial-SOS representation pairs are constructed using 
  % two data construction methods based on 
  computation-driven and structure-driven approaches. 
  % Constructing non-negative polynomial-SOS representations for the data using two types of data construction methods: Computation-Driven and Structure-Driven;  
  A Large Language Model (LLM) is trained on the constructed data to function as an SOS structure conjecturer, 
  % The SOS structure conjecturer 
  which generates corresponding SOS representations based on the non-negative polynomials and ranks them according to the magnitude of the errors. 
  (2) \textbf{Symbolic Correction Module}: An exact SOS representation is derived from the top-ranked SOS structure conjectures through a symbolic computation process that involves Newton iteration and rational recovery. 
  (3) \textbf{Formal Verification Module}: Based on the exact SOS representations and predefined Lean proof templates, a complete Lean formal proof is automatically generated. 
  }  \label{fig:overview}
\end{figure*}
 
Let $\mathbb{R}[x]:=\mathbb{R}[x_1,\ldots,x_n]$ be the ring of polynomials in $n$ variables with coefficients in the real field $\mathbb{R}$.  A polynomial  $f(x) \in \mathbb{R}[x]$ is said to be nonnegative or positive semidefinite (PSD) if $f(x)\geq 0$ for all $x \in \mathbb{R}^n$.  
In this work, we focus on  automated proving of
unconstrained polynomial nonnegativity:
given $f \in \mathbb{R}[\mathbf{x}]$, our goal is to formally prove
\begin{align}
\label{eq:goal}
 f(x)\geq 0, \quad \forall x \in \mathbb{R}^{n}.
\end{align}
Beyond establishing~\eqref{eq:goal} mathematically,
automated theorem proving additionally requires certificates that are
rigorous and machine-checkable, rather than purely numerical or heuristic validations.

A widely-used sufficient certificate for~\eqref{eq:goal} is a \emph{sum-of-squares} (SOS) decomposition. A polynomial $f(x)$ is said to be an SOS, if there exist polynomials $f_1(x),\ldots,f_m(x)$ such that 
\begin{align}
\label{eq:sos_def}
    f(x) = \sum_{i=1}^{m} f_i(x)^2.
\end{align}
It is immediate $f(x)$ being SOS implies that is nonnegative over $\mathbb{R}^n$, 
hence an explicit SOS decomposition provides a constructive certificate of~\eqref{eq:goal}.
Notice that necessarily  $f(x)$ must be of even degree $2d$.
 Let $\mathbf{v}_{d}(x)$ be the vector 
$$\mathbf{v}_{d}(x)=[1,x_1,x_2,\ldots,x_n, x_1^2, x_1 x_2,\ldots,x_n^d]^{\mathsf T},$$
of all monomials in $x$          
 and whose degrees are at most $d$, which has dimension $s(d)=\binom{n+d}{d}$.
Then $f(x)$ is SOS if and only if there exists a symmetric positive semidefinite matrix $G   \succeq 0$ such that
\begin{align}
\label{eq:gram_form}
f(x) = \mathbf{v}(x)^{\mathsf T} G \mathbf{v}(x).
\end{align}
Equating coefficients in the identity~\eqref{eq:gram_form} yields a system of linear equations that the entries of $G$ must satisfy.
Therefore, determining whether $f(x)$ is SOS can
be formulated as a semidefinite feasibility problem:
\begin{align}
\label{eq:preliminary_sdp}
\left\{
\begin{aligned}
&\text{find} && G \in \mathbb{R}^{s(d)\times s(d)} \\
&\text{s.t.} && G \succeq 0, \quad G=G^{\mathsf T} ,\\
&&& f(x) =\mathbf{v}(x)^{\mathsf T} G \mathbf{v}(x).
\end{aligned}
\right.
\end{align}
However, SDP solvers typically return numerical solutions, while formal verification requires exact certificates.
This motivates our focus on constructing \emph{exact} SOS certificates that can be directly checked in a proof assistant.

%% file: Sec/Method.tex
\section{Methodology}\label{Method}  %++CL  %% TODO 补动机
In this section, we introduce \textbf{Ne}uro-\textbf{Sy}mbolic SOS-based \textbf{P}olynomial \textbf{I}nequality \textbf{P}roving (\textbf{NSPI}), 
a neuro-symbolic framework for automated proving of unconstrained polynomial inequalities. 
NSPI is designed as a pipeline that integrates LLM-based conjecture with symbolic computation and formal verification. 
The overall architecture of NSPI is illustrated in Fig.~\ref{fig:overview}, which depicts the following three main stages:

% \begin{itemize}
% \item  \textbf{Neural Conjecture Module}: The LLM acts as a conjecturer for the SOS structure. Initially, five methods are used to construct a diverse set of non-negative polynomial-SOS pair data. Based on this dataset, the SOS structure conjecturer is trained using reinforcement learning techniques, thereby enhancing the accuracy of the LLM in predicting the SOS structure of polynomials. (See \cref{sec:Neural Conjecture} for details)
% \item  \textbf{Symbolic Correction Module}: Symbolic computation methods serve as a precision bridge for SOS. This module leverages symbolic computation tools such as SymPy and SciPy to verify and refine both the original polynomial and the approximate SOS decomposition. By applying Newton's method and rational recovery techniques, floating-point solutions are converted into precise rational solutions, yielding an accurate SOS representation of the polynomial. (See  \cref{sec:Symbolic Correction} for details)
% \item  \textbf{Formal Verification Module}: The formal verification module integrates automated Lean proof templates to provide a complete formal Lean theorem proof based on the precise SOS decomposition. (See \cref{sec:Formal Verification} for details)
% \end{itemize}

% 神经猜想模块：LLM作为SOS结构的conjecturer。
% 符号修正模块：符号计算方法作为SOS的精度桥梁。
% 形式化验证模块：基于SOS精确分解给出形式化的完整Lean定理证明。

\text{\bf{[Neural Conjecture].}} 
% The Neural Conjecture module primarily leverages the structural conjecture capability of Large Language Models (LLMs) to predict the SOS decomposition of polynomials. The trained LLM can generate a potential SOS decomposition conjecture based on the input polynomial. 
%%%%%%%
% LLM acts as a \textbf{structure conjecturer} for the SOS. 
% % Initially, 
% % four methods 
% Computation-driven and structure-driven methods 
% are used to construct a diverse set of non-negative polynomial-SOS pair data. 
% % Based on this dataset, the 
% A SOS structure conjecturer is trained on synthetic data using a progressive two-stage training process, thereby enhancing the performance of the LLM in predicting the SOS structure of polynomials. 
% (See \cref{sec:Neural Conjecture} for details)
%%%%%zhr-0128
The LLM is employed as a \emph{structure conjecturer} for sum-of-squares (SOS) representations. By combining computation-driven and structure-driven strategies, we construct a diverse dataset of nonnegative polynomials and their corresponding SOS pair forms. The conjecturer is trained on synthetically generated data via a progressive two-stage training scheme, enabling the LLM to more accurately predict plausible SOS structures for given polynomials (see more details in Section~\ref{sec:Neural Conjecture}).

% \text{\bf{[Symbolic Correction].}} Symbolic computation methods serve as a \textbf{precision bridge} for SOS. 
% % This module leverages 
% In this stage, we use symbolic computation tools (e.g., SymPy) to refine 
% % both the original polynomial and 
% the approximate SOS decomposition. 
% By leveraging Newton's method and rational recovery techniques, 
% % floating-point solutions are converted into precise rational solutions, 
% an exact SOS representation of the polynomial is obtained. (See  \cref{sec:Symbolic Correction} for details)

\textbf{[Symbolic Correction].} 
Symbolic computation serves as a \emph{precision bridge} in SOS-based proving. 
At this point, symbolic computation tools are employed to refine the approximate SOS decomposition produced in the previous stage. 
By combining Newton-type iterative refinement with rational recovery techniques, 
numerical solutions are systematically converted into exact rational representations, 
yielding a precise SOS certificate of the target polynomial. (see more details in Section~\ref{sec:Symbolic Correction}).

% \text{\bf{[Formal Verification].}} The formal verification module integrates automated Lean proof templates to provide a complete formal Lean theorem proof based on the precise SOS decomposition. (See \cref{sec:Formal Verification} for details)

\textbf{[Formal Verification].} 
This module converts the precise SOS decomposition into machine-checkable Lean proof templates, 
thereby ensuring that the entire proving process, from conjecture to certificate, is fully verified by a trusted formal kernel. 
(see more details in Section~\ref{sec:Formal Verification}).

% 【数据构造部分框架】
% 问题： f 可写成 sos 当且仅当存在 半正定矩阵 ***
% 图示：
% 如何构造G ？两类方法：
%  1.计算方式：
%     （1）特征值计算方式： 随机Gram矩阵 -> 半正定
%     （2）随机f -》 求解LMI 最小扰动得到Gram G
%  2.特殊结构：对角占优
% 得到G之后：随机取 v 可以得到新数据

%%% TODO 补动机
\subsection{Neural Conjecture: LLM-Guided SOS Generation}
% \subsection{LLM-Guided SOS Structure Conjecturing}
\label{sec:Neural Conjecture}
% Neural Conjecture module leverages the structural conjecture capability of 
% % LLMs 
% Large Language Models (LLMs) 
% to predict the SOS decomposition of polynomials. The trained LLM can generate a potential SOS decomposition conjecture based on the input polynomial. 
% The Neural Conjecture module primarily consists of two stages: (1) constructing SOS data, and (2) training the SOS structure conjecturer. 
% % The LLM acts as a conjecturer for the SOS structure. Initially, five methods are used to construct a diverse set of non-negative polynomial-SOS pair data. Based on this dataset, the SOS structure conjecturer is trained using reinforcement learning techniques, thereby enhancing the accuracy of the LLM in predicting the SOS structure of polynomials. 

%%%zhr-0128
Given an input polynomial, the \emph{neural conjecture} component leverages the structural conjecturing capability of large language models to propose candidate the sum-of-squares decomposition. 
The overall procedure is organized into two complementary parts as follows:
\begin{itemize}
    \item [i)] {\it Constructing SOS training data}: we construct large-scale polynomial--SOS training pairs by generating SOS polynomials through Gram matrix synthesis, using both computation-driven and structure-driven mechanisms to obtain PSD matrices with controlled coefficient properties. (see Subsection~\ref{sec: SOS Data Construction Method})
    \item [ii)] {\it Training the SOS structure conjecturer}: we train an SOS structure conjecturer via a progressive two-stage scheme, where supervised fine-tuning provides a cold start on the synthetic corpus and curriculum-based reinforcement learning further improves conjecturing performance on harder multivariate instances. (see Subsection~\ref{subsec:Training the SOS Structure Conjecturer})
\end{itemize}

\subsubsection{SOS Data Construction Method}
\label{sec: SOS Data Construction Method}
% 反向构造易导致系数爆炸  多项式复杂 
% 如何构造与真实 ...  
% TODO: 数据构造写一句motivation 参考nips24

% \begin{figure}[htbp]
%   \centering
%   \includegraphics[width=0.95\linewidth]{figures/data_construction/data construvtion_251118_v3.pdf}%{figures/data_construction/data construvtion_251114_v2.pdf}%{figures/data construvtion_251112_v2.pdf}%{figures/data construvtion_251112_v1.pdf}
%   % \caption{\textbf{New SOS data construction methods}.  
%   % }  
%   % \caption{\textbf{Framework of SOS data construction methods}. The generation of polynomial-SOS pairs is based on the quadratic form $f(x) = {\mathbf v} (x)^T \widetilde{G} {\mathbf v}(x)$, where the Gram matrix $\widetilde{G}$ is synthesized through either computation-driven or structure-driven approaches.} 
%   \caption{SOS data construction. Polynomial–SOS pairs are generated 
%   % by synthesizing Gram matrices 
%   via computation-driven and structure-driven approaches.}
%   \label{fig:dataconstruction}
% \end{figure}

%%%%%%zhr-0128
Constructing nonnegative polynomials and and their corresponding SOS representations in a systematic and numerically well-behaved manner is a nontrivial task. 
A straightforward approach is to randomly sample polynomials $f_i(x)$ and construct
\(
f(x) = \sum_i f_i(x)^2.
\)
However, this naive strategy suffers from two major limitations: (1) the coefficients of the generated $f_i(x)$ are typically non-integer, which makes the resulting polynomial inconsistent with the typical symbolic representations that require integer coefficients, and (2) the expansion of randomly generated squared polynomials often leads to severe coefficient swelling, producing polynomials with unreasonably large or unbalanced coefficients.

To address these issues, we develop a novel data construction method grounded in the algebraic structure of SOS decomposition. 
As reviewed in \cref{sec:preliminary}, a polynomial $f(x)$ is SOS if and only if there exists a positive semidefinite (PSD) Gram matrix $\widetilde{G}$ such that
\(
f(x) = \mathbf{v}(x)^{\mathsf T} \widetilde{G} \mathbf{v}(x),
\)
where $\mathbf{v}(x)$ denotes the vector of monomials. 
Consequently, the problem of generating suitable polynomial--SOS pairs reduces to the problem of \textbf{\emph{constructing a PSD Gram matrix $\widetilde{G}$}}  with controlled coefficient structure.

Based on this observation, we propose two families of SOS data construction methods, namely \emph{computation-driven} and \emph{structure-driven} approaches. 
% as shown in Fig.~\ref{fig:dataconstruction}.
Once a Gram matrix $\widetilde{G}$ is obtained by either approach, the 
% corresponding 
SOS polynomial $f(x)$ is generated by instantiating a monomial basis $\mathbf{v}(x)$.

\textbf{(1) Computation-Driven Methods.} We first consider constructing the Gram matrix $\widetilde{G}$ through numerical procedures. The objective is to systematically control the magnitude and precision of the coefficients while ensuring that the resulting $~\widetilde{G}$ remains positive semi-definite. The computation-driven methods comprise two main approaches.

\begin{figure*}[htbp]
  \centering
    \includegraphics[width=\linewidth]{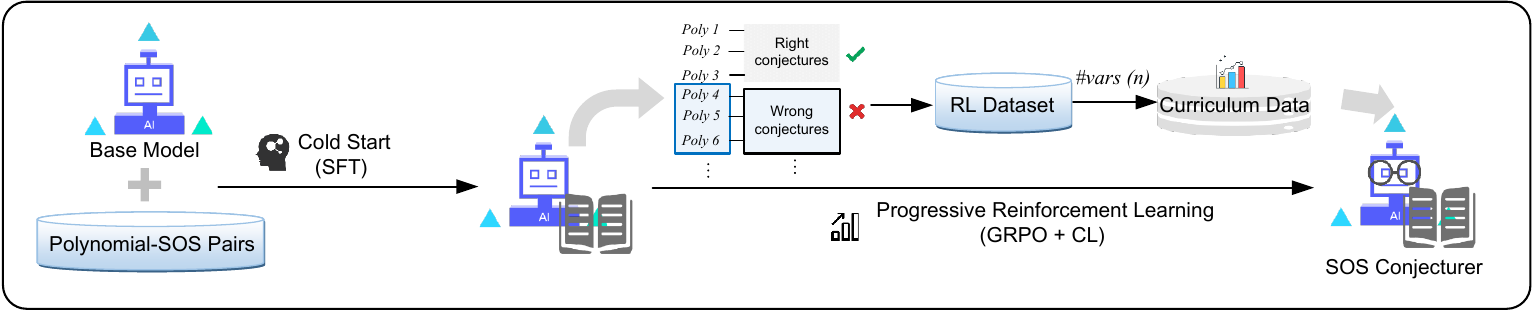} 
  % \caption{\textbf{Process of two-stage training of SOS Conjecturer}.\textbf{Stage 1: Cold Start}. \textbf{Stage 2: Progressive Reinforcement Learning.
  % } }  
  \caption{\textbf{Two-stage training of the SOS conjecturer}: (1) \textbf{Cold Start}: supervised fine-tuning (SFT) on large-scale synthetic polynomial–SOS pairs; (2) \textbf{Progressive reinforcement learning}: curriculum-based GRPO on challenging training data.}
  \label{fig:model_training}
\end{figure*}

%%%%%%%%%%%%%%%%%%%
% % \begin{itemize}
% %     \item 
%     \text{\bf{[Direct Construction Approach].}} A direct computation-driven approach to obtaining a PSD Gram matrix $\widetilde{G}$ is to adjust the eigenvalues of a randomly generated symmetric matrix. Specifically, a symmetric matrix $G \in \mathbb{S}^m$ with integer coefficients is first generated, and its minimum eigenvalue $\lambda_{\min}$ is computed. The PSD matrix $~\widetilde{G}$ is then obtained as: 
%       \begin{align}
%       \widetilde{G} = G - kI \succeq 0 \text{, where } k = \left\lfloor \lambda_{\min} \right\rfloor
%   \end{align}
%   % This ensures that $\widetilde{G}$ is positive semi-definite, as its smallest eigenvalue satisfies $\lambda_{\min} - k \geq 0$.

%   \textbf{Alternatively}, a PSD matrix can be directly constructed by generating a random sparse integer coefficient matrix $L \in \mathbb{Z}^{m \times k}$($k \leq m$)  together with a positive definite diagonal matrix $D$. Here, $k$ represents the number of squared polynomials in the SOS decomposition. The PSD matrix $\widetilde{G}$ is then formed as : %TODO 改表述  更清晰 此时
%   \begin{align}
%     \widetilde{G} = L^T D L.  
%   \end{align}
%     % Since $D$ is positive definite, $\widetilde{G}$ is guaranteed to be semi-positive definite.
%%%%%%%%%zhr-0128
% \textbf{[Direct Construction Approach].} 
\textbf{[Explicit Algebraic Construction].} 
One direct way to obtain a PSD Gram matrix $\widetilde{G}$ is based on spectral shifting. 
Specifically, we first generate a symmetric integer matrix $G \in \mathbb{S}^m$ and compute its smallest eigenvalue $\lambda_{\min}$. 
Setting
 \begin{align}
      \widetilde{G} = G - kI \succeq 0 \text{, where } k = \left\lfloor \lambda_{\min} \right\rfloor
  \end{align}
yields a matrix $\widetilde{G} \succeq 0$, since its smallest eigenvalue satisfies $\lambda_{\min} - k \geq 0$.

\textbf{Alternatively}, a PSD Gram matrix can also be constructed in factored form. 
Let $L \in \mathbb{Z}^{m \times k}$($k \leq m$) be a sparse integer matrix and let $D \in \mathbb{R}^{k \times k}$ be a positive definite diagonal matrix, where $k$ corresponds to the number of squared terms in the SOS decomposition. 
Then the matrix
\begin{align}
\widetilde{G} = L^{\mathsf T} D L
\end{align}
is symmetric positive semidefinite by construction.

% \text{\bf{[Optimization-Based Approach].}} 
{\bf{[Optimization-Based Approach].}} Another computation-driven method constructs the Gram matrix $\widetilde{G}$ through Linear Matrix Inequality (LMI) optimization. 
This approach employs LMI optimization to compute an approximate SOS representation, followed by controlled adjustment to obtain an exact integer-coefficient solution. 
% 待修改
Suppose $f(x) \in \mathbb{Z}[x]$ is  a randomly selected integer-coefficient polynomial.  
We first solve the following semidefinite program: 
\begin{align}
\begin{cases}  \text{max} \ \lambda \\ \text{s.t.} \ f(x) = \mathbf{v}(x)^{\mathsf T} G {\mathbf v}(x) \\ \ \ \ \ \ \ G - \lambda I \succeq 0, \quad G= G^{\mathsf T}. \end{cases}
\end{align}
      %where $G$ is the Gram matrix corresponding to $f(x)$. 
Let $(\lambda,G)$ be an optimal solution. Since $G-\lambda I\succeq 0$, for any integer $k\ge -\lambda$ we have
$G+kI\succeq 0$. Choosing $k:=\lceil -\lambda\rceil\in\mathbb{Z}$ yields $G+kI\succeq 0$, and therefore 
% $\tilde f(x)\ :=\ f(x)+\mathbf{v}(x)^{\mathsf T}(kI)\mathbf{v}(x)
% \ =\ \mathbf{v}(x)^{\mathsf T}(G+kI)\mathbf{v}(x)$ 
\[
\tilde f(x)\ :=\ f(x)+\mathbf{v}(x)^{\mathsf T}(kI)\mathbf{v}(x)
\ =\ \mathbf{v}(x)^{\mathsf T}(G+kI)\mathbf{v}(x)
\]
is SOS polynomial with integer coefficients. 
    Then the Gram matrix is updated as $\widetilde{G}= G +k I$. 
    % Since $G$ corresponds to an integer-coefficient polynomial and $\lambda$ is an integer, the resulting polynomial is guaranteed to have integer coefficients and constitute a valid SOS polynomial.  
%%% TODO 附录
    Appendix~\ref{appendix:foundation_of_Computation-Driven} provides the theoretical foundations for the aforementioned computation-driven methods.

\textbf{(2) Structure-Driven Methods.} This category constructs the Gram matrix $\widetilde{G}$ by exploiting algebraic matrix structures that guarantee positive semidefiniteness by design, most notably diagonally dominant (dd) matrices and scaled diagonally dominant (sdd) matrices.
% Another category of methods is the structure-driven approach, which primarily constructs the Gram matrix $\widetilde{G}$ based on the properties of diagonally dominant (dd) matrices. 
 
Specifically, a matrix $ G \in \mathbb{R}^{m \times m} $ is defined as diagonally dominant if it satisfies the following condition: $ G_{ii} \ge \sum_{j \ne i} |G_{ij}|, \forall i$. 
By Gershgorin’s circle theorem \cite{gershgorin1931uber}, any symmetric $G$ that satisfies the diagonal dominance condition is guaranteed to be positive semidefinite.
In this case, $ \widetilde{G}$ can be directly initialized as a diagonally dominant matrix. 
Furthermore, we \emph{construct} $\widetilde{G}$ in a structured diagonally dominant form by restricting it to a non-negative combination of $m$ rank-one generators inspired by \cite{barker1975cones}:
\begin{align}
\label{DD_Matrix}
\widetilde{G} = \sum_{i=1}^{m^2} \eta_i U_i,
\end{align}
where $\widetilde{G}$ is an $m \times m$ symmetric matrix, 
$\eta_i \ge 0$, 
and $U_i = \mathbf{u}_i \mathbf{u}_i^{\top}$, 
with $\mathbf{u}_i \in \mathbb{R}^{m}$ having at most two non-zero components, each equal to $\pm 1$.
This representation provides a convenient structural parameterization of PSD Gram matrices, and naturally extends to the scaled diagonally dominant case to introduce greater flexibility. Appendix~\ref{appendix:foundation_of_Structure-Driven} provides the theoretical foundations and more details of structure-driven methods.

\subsubsection{Training the SOS Structure Conjecturer}
\label{subsec:Training the SOS Structure Conjecturer}
To obtain an LLM-based SOS structure conjecturer with strong structural reasoning capability, we design a progressive two-stage training scheme. 
As illustrated in Fig.~\ref{fig:model_training}, the procedure consists of two phases.
(1) {\emph {Cold Start}:} {\bf\emph{endow}} the model with 
the ability to conjecture SOS structures 
via supervised fine-tuning (SFT) on the constructed polynomial-SOS pair data. 
(2) {\emph {Progressive Reinforcement Learning}:} {\bf\emph{further enhance}} the model’s SOS-structure conjecturing performance via curriculum-based reinforcement learning (RL) on challenging data. 

%%% TODO  强调大量合成数据
% \item \textbf{Data Utilization:}  去重
{\bf{Cold Start.}}We perform SFT on the base model using over one million synthetic data samples. 
These samples are generated via the method described in Section~\ref{sec: SOS Data Construction Method}, with each instance formulated as a pair $(f(x), S)$, where $f(x)$ denotes a non-negative polynomial and $S$ 
represents its corresponding SOS decomposition. This stage endows the base model with the foundational capability for SOS structural reasoning, serving as a critical precursor to the subsequent reinforcement learning phase. Further details regarding the synthetic dataset are provided in Appendix \ref{appendix:synthetic_data_statistics}.
% The SFT stage utilizes the diverse non-negative polynomial-SOS representation pairs constructed in Section \ref{sec: SOS Data Construction Method}. Each data sample takes the form of $(f(x), S)$, where $f(x)$ is the input polynomial and $S$ is its corresponding ground-truth SOS decomposition structure. 
% This stage equips the base model with the ability to reason about SOS structures, laying the foundation for the subsequent reinforcement learning phase. 
%%%TODO 具体用了多少训练数据  怎么训的  数据难度和质量怎么保证   附录给了更多关于训练数据的信息

% TODO 根据 新构造RL训练数据改
% \textbf{Reinforcement Learning (RL).}
% GRPO 强调数据来源 reward设计 + 优化目标
{\bf{Progressive Reinforcement Learning.}}We further optimize the model's SOS structural reasoning through progressive reinforcement learning on challenging tasks. Specifically, 
we curate the training instances that remain unsolved by the cold-started model into an easy-to-hard curriculum and conduct curriculum-based Group Relative Policy Optimization (GRPO) \cite{shao2024deepseekmath}. 
The reward function is designed with three critical components: 

\text{\bf{[Accuracy Reward].}} 
The accuracy reward encourages SOS structure conjectures with
smaller errors compared to the original polynomial. 
It measures the numerical fidelity of the conjecture and is computed as 
% It evaluates the accuracy of the SOS structure conjecture by measuring the error between the original polynomial and 
% % the SOS structure 
% conjecture, using the following formula:
    \begin{align}
        R_\text{Accuracy} = \frac{1}{1 + \alpha\|\mathbf{f}(x) -  \hat{\mathbf{f}}{(x)}\|_2},
    \end{align}
    where $ \mathbf{f(x)} $ represents the original polynomial, 
    % $\alpha$ is 
    $ \hat{\mathbf{f}}(\bf x) $ denotes the SOS structure conjecture generated by the model, and $\alpha$ is a scaling factor. 
    %The $ L_2 $-norm $| \cdot |_2 $ is employed to quantify the discrepancy between the two.
    
    \text{\bf{[Format Reward].}} 
     A binary indicator ensuring the output strictly adheres to the predefined SOS structural template and delimiters (e.g., \textit{\textless SOS Expression\textgreater}). 
    % % The format reward function 
    % Ensures that the conjecture generated by the SOS structure conjecturer follows an SOS form and is enclosed within the predefined delimiter \textit{\text{\textless SOS Expression\textgreater}}. 
    % % Regular expression matching is used to ensure that the conjecture adheres to the predefined structure. 
    % A reward of 1 is assigned if the format is correct; otherwise, the reward is set to 0.
    
    \text{\bf{[Algebraic Structure Penalty].}}  
     The algebraic structure consistency penalty is designed to ensure that the set of nonzero monomials in the SOS strucure conjecture closely matches that of the original polynomial, 
     which comprises a \textit{soft penalty} and a \textit{hard penalty}. 
     Appendix ~\ref{appendix:details_of_RL} details the specific computation procedure and 
     provides further details on the progressive reinforcement learning procedure.

\subsection{Symbolic Correction: Exact Rational Recovery}
\label{sec:Symbolic Correction}
In this part, the {\it symbolic correction} module serves as a precision bridge between 
% the approximate SOS structure conjectured by the Neural Conjecture Module 
model-generated SOS structure conjecture 
and the exact representation required for formal verification, 
which encompasses two stages: Gauss–Newton refinement and rational recovery.  
% The correction is achieved through a multi-step procedure that transforms the floating-point Gram matrix $\mathbf{Q}_{\text{float}}$ into an exact rational matrix $\mathbf{Q}_{\text{rational}}$.

% TODO 示例图  
% Gram x\ge 0  x  y
%              y  z   xz-y^2 \ge 0  锥
%  约束超平面    例子？
%\subsubsection{Numerical Refinement via Newton Iteration}
% \label{Newton's Method}
\label{sec:Newton}
{\bf Gauss–Newton Refinement.} After obtaining the SOS structural conjecture $\hat{f}{(x)}$ from the {\it neural conjecture} module, we perform Gauss–Newton refinement to enhance the numerical precision of the SOS representation. 

First, the corresponding monomial basis ${\mathbf v}(x)$ is extracted, and an initial floating-point Gram matrix $\mathbf{G}$ is constructed such that $\hat{f}{(x)} \approx {\mathbf v}(x)^{\mathsf T} \mathbf{G} {\mathbf v}(x)$. 
% \begin{align}
%     \hat{f(x)} \approx v(x)^T \mathbf{G} v(x).
% \end{align}
 To refine $\mathbf{G}$ using the Gauss–Newton iteration, we compute the Cholesky decomposition of $\mathbf{G}$:
\begin{align}
    \hat{f}(\mathbf{x}) \approx \mathbf{v}(\mathbf{x})^{\mathsf T} L L^{\mathsf T} \mathbf{v}(\mathbf{x})
 = \sum_{i=1}^{k} \left(\sum_{\alpha} c_{i,\alpha} \mathbf{x}^{\alpha}\right)^2,
\end{align}
where $k$ is the rank of the matrix $\mathbf{G}$, and $L L^{\mathsf T}$ denotes the Cholesky factorization of the Gram matrix $\mathbf{G}$.

The Gauss–Newton iteration is applied to compute the coefficient correction term $\Delta c_{i,\alpha} \mathbf{x}^{\alpha}$ such that
\begin{align}
\label{math:newton_expression}
     \hat{f}(\mathbf{x}) = \sum_{i=1}^{k}
 \left(\sum_{\alpha} c_{i,\alpha} \mathbf{x}^{\alpha} +
 \Delta c_{i,\alpha} \mathbf{x}^{\alpha}\right)^2,
\end{align}

 where $\Delta c_{i,\alpha} \mathbf{x}^{\alpha}$ represents the perturbation to the polynomial coefficients, and the Gram matrix is updated as $\mathbf{G} + \Delta \mathbf{G}$. 
 The optimization objective is to minimize the backward error:
\begin{align}
     \theta = || \hat{f}(\mathbf{x}) - {\mathbf v}(x)^T \mathbf{G} {\mathbf v}(x) ||.
\end{align}
 % The iterative process
 Gauss–Newton iteration terminates when $\theta$ falls below a predefined tolerance threshold $\tau$. 
 This process is crucial for the subsequent rational recovery process, %  todo 得到G_precise
 ensuring that the obtained SOS representation achieves high 
 % numerical stability and 
 precision. 
 Appendix \ref{appendix:newton_} provides further details of the procedure.
 % Newton iteration.
 
%\subsubsection{Exact Rational Recovery}
%\label{sec:RationalRecovery}
% After completing the Gauss–Newton refinement, our goal is to convert the 
% % high-precision floating-point Gram 
% refined 
% matrix into its exact rational form, thereby obtaining an exact SOS representation.

% After obtaining a high-precision numerical Gram matrix from the Gauss–Newton refinement, we convert it into an exact rational PSD matrix that satisfies the polynomial identity without numerical error.
{\bf Rational Recovery.} The numerical Gram matrix $G_N$ obtained from 
% high-precision 
Gauss–Newton refinement contains inherent floating-point errors. Our objective is to transform the matrix into an exact rational PSD matrix that rigorously satisfies the polynomial identity without numerical uncertainty, thereby yielding an exact SOS representation.

% TODO 改写 加逻辑 不够易懂
% According to the proposition in \cite{PEYRL2008269}, if the polynomial admits a Gram matrix lying strictly in the interior of the PSD cone, then there exists a threshold $\delta > 0$ such that any numerical Gram matrix within $\delta$ of the true solution can be converted into an exact rational PSD matrix by applying rational approximation and orthogonal projection. 

According to a classical result on rational recovery~\cite{PEYRL2008269},\footnote{%
If a polynomial admits a Gram matrix lying strictly in the interior of the PSD cone, then there exists a threshold $\delta > 0$ such that any numerical Gram matrix within $\delta$ of the exact solution can be converted into an exact rational PSD matrix by combining rational approximation with orthogonal projection
}
we distinguish two recovery regimes depending on the numerical rank of the refined solution.
(1) \textit{Interior-point case}: 
when the refined Gram matrix lies strictly in the interior of the PSD cone, we project it orthogonally onto the affine subspace defined by the SOS constraints and then rationalize the resulting matrix.
(2) \textit{Boundary case}: 
when the matrix is numerically rank-deficient, direct matrix rationalization is avoided. 
Instead, we perform a truncated $LDL^{\top}$ factorization followed by simultaneous Diophantine approximation to recover rational vectors while preserving the rank structure.
% When the refined Gram matrix lies in the interior of the PSD cone, we first project it orthogonally onto the affine hyperplane determined by the SOS constraints, and then rationalize the entries of the projected matrix. 
% When the matrix is numerically rank-deficient, we instead perform rational vector recovery based on a truncated $LDL^{T}$ factorization combined with simultaneous Diophantine approximation, thereby preserving the correct rank structure. 
% In practical implementation, we adopt a scaled rational recovery approach, converting each high-precision floating-point entry into a simple fraction with a bounded denominator. Owing to the extremely small backward error achieved by Newton refinement, the nearest such fraction can be guaranteed to coincide with the true underlying rational coefficient.  
Appendix \ref{appendix:theorey_symbolic} provides further details of Symbolic Correction module. 
% The theoretical foundations of the numerical refinement and rational recovery procedures are summarized in Appendix \ref{}. 

\subsection{Formal Verification: Lean Proof Generation}
% Lean 证明模板 
\label{sec:Formal Verification}
% \begin{tcolorbox}[colback=gray!5!white, colframe=blue!50!black, arc=3mm, title={\texttt{Lean Proof Template}}]
% \begin{minted}[fontsize=\small, breaklines]{lean}
% theorem <Theorem_Name> (x1 x2 x3 ... xn : ℝ)
%   (h₁ : I = <polynomial>) :
%   I ≥ 0 := by
%   let terms : List (ℝ × ℝ) :=
%     [ <Exact SOS Expression> ]
%   have : I = (terms.map (λ (p, k) => p * k^2)).sum := by
%     unfold terms
%     simp only [neg_mul, List.map_cons, one_pow, mul_one, one_mul, zero_mul, List.map_nil,
%       List.sum_cons, List.sum_nil, add_zero, zero_add]
%     linear_combination h₁
%   rw [this]
%   unfold terms
%   simp only [neg_mul, List.map_cons, one_pow, mul_one, one_mul, zero_mul, List.map_nil,
%     List.sum_cons, List.sum_nil, add_zero, zero_add, ge_iff_le]
%   positivity

% \end{minted}
% \end{tcolorbox}

% 模板 + 例子图
% 工具部分  外部得到SOS  下面在Lean中验证 。
% 需要证明以下引理：
%   F = 平方和  等号是对的
%   平方和形式怎么证明非负   
%   主要是两个引理  怎么证明   证明部分      介绍引理 + 给例子。

After obtaining the exact SOS certificates, the {\it formal verification} module generates a complete Lean proof. 
% using predefined proof templates together with the SOS certificate.  
% TODO 策略名待修改
% \textbf{Automated Lean Tactic Integration.} 
By integrating pre-defined proof templates with the {\it neural conjecture} and {\it symbolic correction} modules, we implement a callable Lean tactic, \textit{llm\_ineq}.  
% Applying this tactic directly within Lean automatically generates the complete Lean proof code. 
% Figure \ref{fig:lean_proof_generation} illustrates the 
% % template structure 
% structure of the proof template %utilized in this integration 
% and an example of automated proof generation. 
%%%%%%%%%
% Proof templates and generation examples are detailed in Appendix~\ref{apendix:more_about_lean_verification}.
Given a target polynomial and its exact SOS certificate, \textit{llm\_ineq} automatically constructs a complete Lean proof by discharging two obligations: 
% Serving as the blueprint for generating the formal proof, the \textit{llm\_ineq} tactic is employed to automate the verification of the polynomial inequality. 
% % Through this integration, upon providing a polynomial and its corresponding SOS decomposition, the system automatically invokes the requisite proof strategy, thereby minimizing human intervention and enhancing the efficiency of formal verification.
% Specifically, to establish the polynomial inequality from its exact SOS representation, two lemmas must be verified.

    \textbf{Equality between polynomial and SOS certificate.} Lean expands the polynomial and the SOS expression into canonical forms and checks their equality via the \textit{linear\_combination} tactic. 
    For example, the code below shows how Lean verifies the equality: %%%linenos,
%     \begin{minted}  [fontsize=\small,frame=single,breaklines]{lean}
% have h_eq : p = (terms.map (fun (q, k) => k • q^2)).sum := by
%   linear_combination
% \end{minted}   

% \begin{lstlisting}[language=lean4]
\begin{lstlisting}[
    language=lean4,
    xleftmargin=0.5em,
    xrightmargin=0.5em,
    aboveskip=1em,
    belowskip=0.8em,
    framexleftmargin=0.5em,
    framexrightmargin=0.5em,
    framextopmargin=3pt,
    framexbottommargin=3pt
]
have h_eq : p = (terms.map (fun (q, k) => k * q^2)).sum := by
  linear_combination
\end{lstlisting}
    
    % To verify that the polynomial coincides with the SOS expansion, Lean expands both sides into their canonical algebraic forms and checks their consistency. This is accomplished using the \textit{\text{linear\_combination}} tactic. 
    \textbf{Nonnegativity of the SOS expansion.} 
    Lean provides several built-in strategies for proving nonnegativity. 
    The nonnegativity of the SOS expression is proved by {\it Lean}’s \textit{positivity} tactic, which recursively applies standard rules (e.g., \textit{sq\_nonneg}, \textit{mul\_nonneg}, \textit{add\_nonneg}).
%     \begin{minted} [fontsize=\small,frame=single,breaklines]{lean}
% have h_nn : 0 <= (terms.map (fun (q, k) => k • q^2)).sum := by
%   positivity
% \end{minted}

% \begin{lstlisting}[language=lean4]
\begin{lstlisting}[
    language=lean4,
    xleftmargin=0.5em,
    xrightmargin=0.5em,
    aboveskip=1em,
    belowskip=0.8em,
    framexleftmargin=0.5em,
    framexrightmargin=0.5em,
    framextopmargin=3pt,
    framexbottommargin=3pt
]
have h_nn : 0 <= (terms.map (fun (q, k) => k * q^2)).sum := by
  positivity
\end{lstlisting}

    % We employ the automated \textit{positivity} tactic, which recursively applies rules such as \textit{"sq\_nonneg"}, \textit{"mul\_nonneg"}, and \textit{"add\_nonneg"} to establish the nonnegativity of the SOS expression.  
Further details on the proof templates and illustrative examples are provided in Appendix~\ref{apendix:more_about_lean_verification}.

%-- Example of Lean proof code generated by llm_ineq 
% \begin{lstlisting}[language=lean]
% -- Equality between polynomial and SOS certificate.
% have h_eq : I = (terms.map (fun (p, k) => p * k^2)).sum := by
%   linear_combination h
% positivity
% \end{lstlisting}

% Proof templates and generation examples are detailed in Appendix~\ref{apendix:more_about_lean_verification}.

% % \textbf{Automated Lean Proof Generation.} 
% Building upon the integrated tactic, we demonstrate the process of automatically generating the complete Lean proof based on the theorem statement and its corresponding SOS representation. 
% % Specifically, given a polynomial and its SOS decomposition, the \textit{llm\_ineq} tactic is employed to automatically generate the requisite Lean code, ensuring the proof adheres to Lean's formal verification standards. 
% This process begins with a polynomial input, followed by the application of the pre-integrated Lean tactic. The system then automatically synthesizes the Lean code to formally state and verify the inequality. 
% % This approach eliminates the need for manual proof construction, thereby significantly accelerating the formal verification workflow. 
% Figure $\text{\ref{fig:lean_proof_generation}}$ illustrates this process with an example, showcasing the automated generation of Lean proof code using the polynomial and its SOS decomposition.   

%% file: Sec/Exp.tex
\section{Experiments}\label{sec:Exp}
% \subsection{A Competition-Level Benchmark for Multivariate Polynomial Inequalities} 

\begin{table*}[htbp]
\centering
% \caption{
% % Method Comparison in PoliIneqBench Benchmark. DS-Prover-V2 stands for Deepseek-Prover-V2.
% Comparative Performance of NSPI and Baseline Methods on the PolyIneqBench Across Varying Numbers of Variables. 
% }
\caption{\textbf{Comparative Performance on PolyIneqBench (n=3 to 10).} \textit{Pass}: 
success rate within 1 hour; \textit{t(s)}: mean execution time of solved instances. \textit{DS-Prover-v2} denotes DeepSeek-Prover-v2. 
Best and second-best results are in \textbf{bold} and \underline{underlined}, respectively.
% The best results are \textbf{bolded}, and the second-best results are \underline{underlined}.
}
% \caption{Comparative Performance of NSPI and Baseline Methods on the PolyIneqBench Across Varying Numbers of Variables. DS-Prover-v2 denotes DeepSeek-Prover-v2; \textit{Pass} represents the success rate for formal proofs completed within a 1-hour time limit, and \textit{t(s)} indicates the average execution time in seconds for successfully proven instances.}
\label{tab:main_result}
\begin{small}
% \begin{tabular}{l l l l l l l l l l l l l l l l l}
% \begin{tabular}{l p{0.5cm} p{0.4cm} p{0.5cm} p{0.4cm} p{0.5cm} p{0.4cm} p{0.5cm} p{0.4cm} p{0.5cm} p{0.4cm} p{0.5cm} p{0.4cm} p{0.5cm} p{0.4cm} p{0.5cm} p{0.4cm}}
% \begin{tabular}{p{2.3cm} C{0.5cm} C{0.4cm} C{0.5cm} C{0.4cm} C{0.5cm} C{0.4cm} C{0.5cm} C{0.4cm} C{0.5cm} C{0.4cm} C{0.5cm} C{0.4cm} C{0.5cm} C{0.4cm} C{0.5cm} C{0.4cm}}
\begin{tabular}{p{2.3cm} C{0.5cm} C{0.5cm} C{0.5cm} C{0.5cm} C{0.5cm} C{0.5cm} C{0.5cm} C{0.5cm} C{0.5cm} C{0.5cm} C{0.5cm} C{0.5cm} C{0.5cm} C{0.5cm} C{0.5cm} C{0.5cm}}
\toprule
\multirow{2}{*}{\textbf{Method Name}} & 
\multicolumn{2}{c}{\textbf{n=3}} & 
\multicolumn{2}{c}{\textbf{n=4}} &
\multicolumn{2}{c}{\textbf{n=5}} &
\multicolumn{2}{c}{\textbf{n=6}} &
\multicolumn{2}{c}{\textbf{n=7}} &
\multicolumn{2}{c}{\textbf{n=8}} &
\multicolumn{2}{c}{\textbf{n=9}} &
\multicolumn{2}{c}{\textbf{n=10}} \\ 
\cline{2-17}
 &\textbf{Pass}  & \textbf{t(s)}  & 
 \textbf{Pass}  & \textbf{t(s)}   & 
 \textbf{Pass}  & \textbf{t(s)}   & 
 \textbf{Pass}  & \textbf{t(s)}   & 
 \textbf{Pass}  & \textbf{t(s)}   & 
 \textbf{Pass}  & \textbf{t(s)}   & 
 \textbf{Pass}  & \textbf{t(s)}   & 
 \textbf{Pass}  & \textbf{t(s)}   
 \\
\midrule
\multicolumn{17}{c}{\textbf{Symbol-based Method}} \\
\midrule
 Maple &  \textbf{97.6\%} & \text{24.6}  & \underline{39.0\%} & \text{107.4}  & \text{\underline{26.7\%}} & \text{451.2}  & \text{8.2\%} & \text{\underline{54.9}}  & \text{6.7\%} & \text{821.5}  & \text{6.67\%} & \text{456.3}  & \text{3.3\%} & \text{3176.6} & \text{1.7\%} &\text{1525.6}\\
Z3  &  \textbf{97.6\%} & \textbf{0.6}  & \text{32.5\%} & \text{21.8}  & \text{23.3\%} & \text{101.9}  & \text{19.7\%} & \text{196.4}  & \text{1.7\%} & \text{\underline{65.0}}  & \text{1.7\%} & \text{1598.6}  & \text{0.0\%} & \text{NA} & \text{0.0\%} &\text{NA} \\
     % validsdp %\cite{martin2017validsdp} 
    % & \text{--} & \text{--}  & \text{--} & \text{--}  & \text{--} & \text{--}  & \text{--} & \text{--}  & \text{--} & \text{--}  & \text{--} & \text{--}  & \text{--} \\
\midrule
\multicolumn{17}{c}{\textbf{LLM-based Prover}} \\
\midrule
 DS-Prover-v2 %\cite{ren2025deepseek}
 &\text{42.9\%} & \text{19.8}  & \text{2.6\%} & \text{\underline{17.9}}  & \text{0.0\%} & \text{NA}   & \text{0.0\%} & \text{NA}   & \text{0.0\%} & \text{NA}  & \text{0.0\%} & \text{NA}   & \text{0.0\%} & \text{NA}  & \text{0.0\%} & \text{NA}\\
 Goedel-Prover-v2 %\cite{lin2025goedel} 
     &  \text{20.2\%} & \text{189.6}  & \text{5.2\%} & \text{136.4}  & \text{0.0\%} & \text{NA}  & \text{0.0\%} & \text{NA}  & \text{0.0\%} & \text{NA}  & \text{0.0\%} & \text{NA}  & \text{0.0\%} & \text{NA}  & \text{0.0\%} & \text{NA} \\
Kimina-Prover %\cite{lin2025goedel} 
     &  \text{36.9\%} & \text{117.0}  & \text{5.2\%} & \text{111.7}  & \text{0\%} & \text{NA}  & \text{0.0\%} & \text{NA}  & \text{1.7\%} & \text{120.6}  & \text{0.0\%} & \text{NA}  & \text{1.7\%} & \text{115.8}  & \text{0.0\%} &\text{NA} \\
\midrule
\multicolumn{17}{c}{\textbf{General-purpose LLM}}  \\
\midrule 
 GPT-5.2 %gpt-5.2-high
 & \text{26.2\%} & \text{56.8}  & \text{10.4\%} & \text{86.7}  & \text{1.7\%} & \text{\underline{51.8}}  & \text{4.9\%} & \text{154.4}  & \text{3.3\%} & \text{75.3}  & \text{3.3\%} & \text{\textbf{55.9}}  & \text{3.3\%} & \text{\underline{68.3}} & \text{1.7\%} &\text{\textbf{49.2}} \\
 Gemini-3-Pro %% Gemini-3-pro-Preview
 & \text{22.6\%} & \text{91.2}  & \text{24.7\%} & \text{102.4}  & \text{\textbf{36.7\%}} & \text{109.8}  & \text{\underline{21.3\%}} & \text{138.9}  & \text{\underline{15.0\%}} & \text{140.7}  & \text{\underline{15.0\%}} & \text{126.7}  & \text{\underline{13.3\%}} & \text{143.5} &\text{\underline{6.7\%}} &\text{116.7}\\
 DeepSeek-V3.2 & \text{14.3\%} & \text{368.4}  & \text{6.5\%} & \text{420.9}  & \text{1.7\%} & \text{526.8}  & \text{1.6\%} & \text{495.8}  & \text{5.0\%} & \text{393.8}  & \text{5.0\%} & \text{561.5}  & \text{1.7\%} & \text{358.0} & \text{3.3\%} & \text{678.9} \\  %%% qwen3-max-preview-think
\midrule
\multicolumn{17}{c}{\textbf{Hybrid system}}  \\
\midrule
 LIPS %\cite{li2025LIPS} 
 & \underline{71.4\%} & \text{82.5}  & \text{27.3\%} & \text{131.0}  & \text{13.3\%} & \text{141.6}  & \text{11.5\%} & \text{195.7}  & \text{8.3\%} & \text{150.2}  & \text{8.3\%} & \text{153.1}  & \text{1.7\%} & \text{148.7} & \text{0.0\%}& \text{NA} \\
% \textbf{NSPI (ours)}  &  \text{37/84} & \text{16.3}  & \text{31/77} & \text{14.4}  & \text{22/60} & \text{20.0}  & \text{18/61} & \text{21.4}  & \text{16/60} & \text{29.6}  & \text{13/60} & \text{63.9}  & \text{10/60} & \text{20.1}&\text{7/60} &\text{46.3} \\
% \textbf{NSPI (ours)}  &  \text{37} & \text{16.3}  & \text{31} & \text{14.4}  & \text{22} & \text{20.0}  & \text{18} & \text{21.4}  & \text{16} & \text{29.6}  & \text{13} & \text{63.9}  & \text{10} & \text{20.1}&\text{7} &\text{46.3} \\
\textbf{NSPI (ours)}  &  \text{44.1\%} & \underline{16.3}  & \textbf{40.3\%} & \textbf{14.4}  & \textbf{36.7\%} & \textbf{20.0}  & \textbf{29.5\%} & \textbf{21.4}  & \textbf{26.7\%} & \textbf{29.6}  & \textbf{21.7\%} & \text{\underline{63.9}}  & \textbf{15.0\%} & \textbf{20.8}&\textbf{11.7\%} &\text{\underline{58.3}} \\
\bottomrule
\end{tabular}
\end{small}
\end{table*}

% 饼图只统计真实例子的来源
\begin{figure}[htbp]
  \centering
  \includegraphics[width=0.85\linewidth]
  {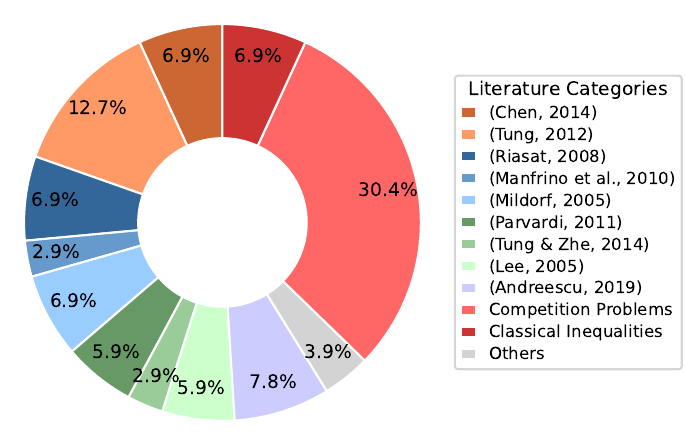}%{figures/experiments/problem_type_distribution_251210_v5.pdf}
  \caption{Distribution of sample counts across data source categories in 
  PolyIneq-Real.
  % PolyIneqBench.   
  }  \label{fig:benchmark_category}%  TODO: 可以加例子变元分布  （2-8）
\end{figure}

\begin{figure}[htbp]
  \centering
  \includegraphics[width=\linewidth]{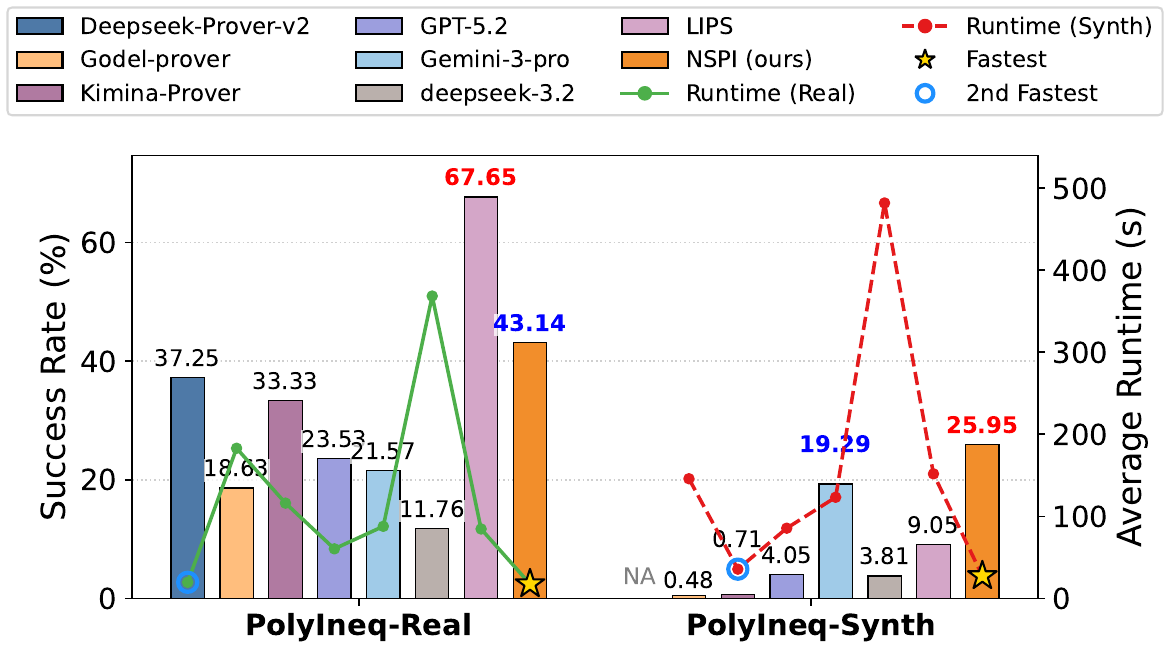}
  \caption{Performance of different methods on PolyIneqBench (PolyIneq-Real and PolyIneq-Synth).}
  % \caption{Comparative Analysis of Success Rates Across Various Methods on PolyIneqBench (PolyIneq-Real and PolyIneq-Synth).}
  \label{fig:poly_method_comparison}
\end{figure}

\begin{figure*}[htbp]
    \centering
    \includegraphics[width=\linewidth]{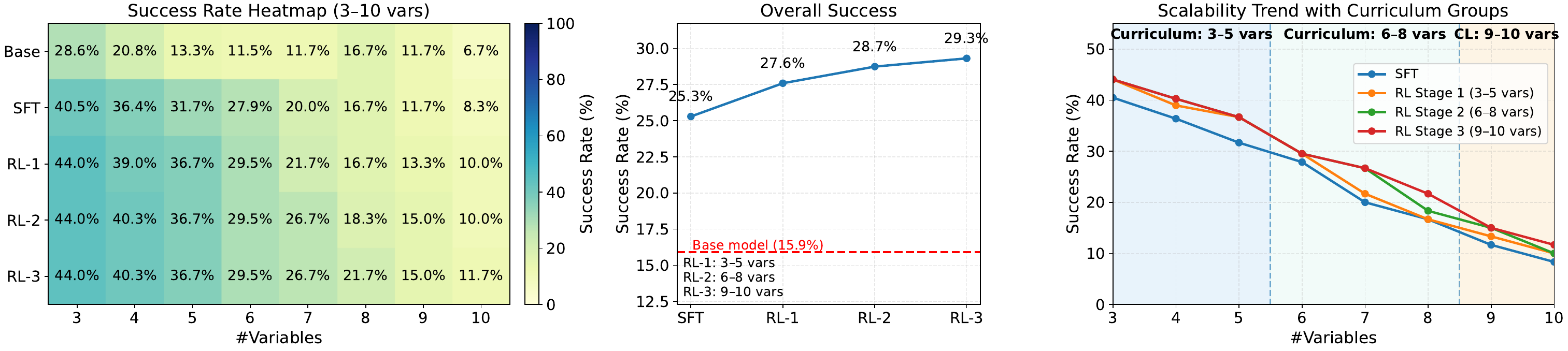}
    \caption{\textbf{Performance on PolyIneqBench across different training stages}. 
    % (3-10 variables)}.
    \textbf{Left:} Success rate heatmap showing capability boundary migration. 
    % Heatmap illustrating success rate transitions as the model progresses from the base state to RL-3 across different problem dimensions. 
    \textbf{Middle:} Overall performance gains from base model through RL stages. 
    \textbf{Right:} Scalability trends across curriculum groups. 
    }
    \label{fig:performance_diff_training}
\end{figure*}

\subsection{A Challenging Benchmark for Multivariate Polynomial Inequalities} 
% 对比实验确定 benchmark扩充 
% To validate the effectiveness of the proposed NSPI method, we conducted a series of experiments. 
% The NSPI was compared with several advanced baseline methods, and ablation studies were performed to investigate the performance of each component of the NSPI. 
% Finally, concrete case studies were provided to demonstrate the advantages of the NSPI.
Existing mathematical benchmarks for theorem proving \textbf{do not} encompass proofs of \textbf{high-dimensional multivariate} polynomial inequalities. 
To fill this gap, 
we construct \textbf{PolyIneqBench}, a challenging benchmark for unconstrained polynomial inequality proving. 
It contains \textbf{522} complex inequality problems with \textbf{3 to 10} variables, all transformed into the form of~\eqref{eq:goal} and formalized in Lean. 
Among them, 102 problems are collected from international mathematics competitions and other authoritative sources
\cite{chen2014brief,nguyen2012nice,riasat2008basics,manfrino2010inequalities,mildorf2005olympiad,Parvardi2011ProblemsVascArqady,NguyenZhou2014ThreeVariableInequalities,lee2005topics,Andreescu2019_118Inequalities}, 
forming \textbf{PolyIneq-Real}. 
Fig.~\ref{fig:benchmark_category} shows the distribution of data sources for PolyIneq-Real. 
Notably, existing competition problems are largely concentrated in the 3 or 4-variable setting. 
To extend the benchmark to more challenging high-dimensional cases, we synthesize problems with 4 to 10 variables to build \textbf{PolyIneq-Synth}. 
Appendix \ref{appendix:details_of_benchmark} provides further details on PolyIneqBench.

\subsection{Experimental Setup} 
% \textbf{Baselines.} 
% We compare the NPSI with various inequality proving methods, including the pure LLM method DeepSeek-Prover v2 \cite{ren2025deepseek} and Goedel prover \cite{lin2025goedel}, the symbolic computation-based method  Maple\cite{heck1993maple} and Validsdp \cite{martin2017validsdp}, and hybrid systems 
% % IneqSearch, AIPS, and    ???
% LIPS \cite{li2025LIPS}.   %% 可以提到一些方法代码不是公开可用的所以难以对比？ 参考AIPS类似表述

% We compare NSPI with several inequality-proving approaches. \textbf{Symbolic computation–based} baselines include Maple\cite{heck1993maple} and Z3\cite{de2008z3}. \textbf{State-of-the-art LLM-based provers} include DeepSeek-Prover V2\cite{ren2025deepseek}, Goedel-Prover-V2\cite{lin2025goedel}, and Kimina Prover\cite{wang2025kimina}. We also consider recent \textbf{general-purpose LLMs}, including GPT-5.2\cite{}, Gemini-3-Pro\cite{}, and DeepSeek-V3.2\cite{}, as well as the hybrid system LIPS\cite{li2025LIPS}. We evaluate all methods on PolyIneqBench using proof success rate and the average runtime of successful proofs, and further examine performance trends as the polynomial degree increases to assess scalability. Appendix \ref{appendix:more_details_of_exp} provides additional details on the experimental configurations.

We compare our method NSPI with several state-of-the-art inequality-proving approaches. Symbolic computation–based baselines include Maple~\cite{heck1993maple} and Z3~\cite{de2008z3}. State-of-the-art LLM-based provers include DeepSeek-Prover-V2~\cite{ren2025deepseek}, Goedel-Prover-V2~\cite{lin2025goedel}, and Kimina Prover~\cite{wang2025kimina}. We also consider recent general-purpose LLMs, including GPT-5.2, Gemini-3-Pro, and DeepSeek-V3.2~\cite{deepseekai2025deepseekv32}, as well as the hybrid system LIPS~\cite{li2025LIPS}. We evaluate all methods on PolyIneqBench using proof success rate and the average runtime of successful proofs, and further examine performance trends as the polynomial degree increases to assess scalability. 
% Appendix \ref{appendix:more_details_of_exp} provides additional details on the experimental configurations. 
Appendix \ref{appendix:more_details_of_exp} provides additional implementation details and experimental configurations.

\subsection{Main Results}

%% 时间复杂度角度分析符号方法 低维情景下擅长
% 

%%% TODO  一个实验结论对应一段的小标题
%%% TODO 结果换成 n<=3 
% 训练结果 + base model结果

% \paragraph{Superior Performance in Multivariate Scenarios.} 
\paragraph{Best performance under multivariate scenarios.} 
\cref{tab:main_result} presents the comparative performance of various baseline methods on the PolyIneqBench across different numbers of variables $n$. 
As illustrated, a general downward trend in pass rates is observed for all methods as the problem dimensionality increases. Overall, the proposed NSPI framework outperforms symbol-based methods, pure LLM-based 
approaches, and other hybrid systems, achieving best performance. 
% state-of-the-art results. 
Notably, while baseline methods struggle as the variable count scales to $n=10$, NSPI maintains a consistent pass rate of 11.7\%.  
% NSPI maintains a pass rate of 11.7\% even when the number of variables scales to $n=10$, demonstrating superior robustness. 
Our analysis further reveals several key insights: 
(1) symbolic methods, though proficient in low-dimensional settings, suffer from exponential computational complexity as dimensionality increases 
and lack the ability to generate human-readable formal proofs; (2) state-of-the-art (SOTA) LLM-based provers encounter significant performance bottlenecks 
% when polynomials exceed 5 variables; 
when tackling competition-level polynomial inequalities exceeding 5 variables; 
and (3) although general-purpose LLMs like Gemini-3-Pro exhibit commendable scalability in certain cases, their overall performance remains inferior to NSPI. 
% While symbolic methods excel in low-dimensional settings, their performance degrades significantly in scenarios with $n \geq 4$ due to the exponential growth in computational complexity. 
% Similarly, existing state-of-the-art LLM-based theorem provers encounter performance bottlenecks in complex polynomial inequality problems exceeding 5 variables. Among general-purpose LLMs, Gemini-3-Pro exhibits commendable scalability, yet its overall performance remains inferior to %the proposed 
% NSPI. 

\paragraph{Challenging synthetic data (PolyIneq-Real vs. PolyIneq-Synth).} 
% In addition, 
We compare the performance of various LLM-assisted methods on PolyIneq-Real and PolyIneq-Synth. 
Fig.~\ref{fig:poly_method_comparison} shows a marked decrease in both pass rates and computational efficiency (running time) observed across all methods when tested on synthetic data, validating that the synthesized high-dimensional instances represent more challenging reasoning scenarios. 
% In the more rigorous PolyIneq-Synth setting, NSPI achieves a 25.95\% pass rate with the lowest average execution time, further establishing its superior efficiency and effectiveness. 
% In the challenging PolyIneq-Synth, NSPI achieves a pass rate that represents a 2.87-fold improvement over the hybrid system LIPS. Regarding computational efficiency, NSPI consistently maintains the lowest execution time across both datasets, demonstrating a remarkable speedup of over 10x compared to several LLM-assisted methodologies.
Under the more rigorous PolyIneq-Synth benchmark, NSPI achieves a \textbf{25.95\%} pass rate while maintaining the minimum average execution time, thereby underscoring its superior efficacy and robustness. Notably, NSPI delivers a \textbf{2.87-fold} improvement in pass rate over the hybrid LIPS system. In terms of computational throughput, NSPI consistently outperforms its counterparts by maintaining the lowest latency across both benchmarks, yielding a remarkable speedup exceeding \textbf{10x} compared to several state-of-the-art LLM-based provers.

% \paragraph{Effectiveness of Progressive Reinforcement Learning.} 
\paragraph{Effectiveness of the progressive training process.} 
Fig.~\ref{fig:performance_diff_training} illustrates the performance of the SOS Structure Conjecturer across successive training stages on PolyIneqBench. 
As illustrated by the color transitions in the left panel of Fig.~\ref{fig:performance_diff_training}, progressive reinforcement learning (RL) facilitates the migration of the model’s \textbf{capability boundaries} toward higher-dimensional polynomials. 
Furthermore, the right panel of Fig.~\ref{fig:performance_diff_training} highlights ``jump-like" performance improvements at the curriculum boundary points for different data groups. 
% Notably, curriculum-specific training also promotes cross-group generalization, enhancing performance beyond the immediate training distribution. 
Moreover, training on specific curriculum groups yields cross-group generalization, partially enhancing performance on out-of-distribution instances.
% It is noteworthy, however, that while the progressive RL process is beneficial, its impact is relatively constrained; as shown in the middle panel of Fig.~\ref{fig:performance_diff_training}, the large-scale Supervised Fine-Tuning (SFT) during the cold-start phase yields more substantial performance gains for the base model.
Overall, large-scale supervised fine-tuning (SFT) in the cold-start phase establishes a strong performance foundation (middle panel), upon which progressive RL provides additional gains by extending generalization to more challenging, higher-dimensional instances.

\subsection{Ablation Studies}

To validate the effectiveness of the individual components of NSPI, 
% the proposed NSPI method, 
% we conducted ablation experiments on its various components.  
% we conducted ablation studies under three distinct configurations on PolyIneqBench: 
we conduct ablation studies on PolyIneqBench under three settings: 
(1) removing the synthetic SOS cold start and applying RL directly to the base model; (2) removing progressive RL; and (3) disabling curriculum-based data partitioning during RL. 
The results in~\cref{tab:ablation} show that the cold-start stage trained on large-scale synthetic data is crucial for SOS conjecturing performance, and that curriculum-based GRPO further improves the overall results.

\begin{table}[thbp]
\centering
\caption{Ablation study of NSPI components. PI-Real and PI-Synth denote PolyIneq-Real and PolyIneq-Synth, respectively.}
\label{tab:ablation} 
\begin{tabular}{lcccccc}
\toprule
\textbf{SOS Data} 
% \textbf{SOS Data Construction} 
& \textbf{GRPO} & 
% \textbf{ASP} & 
\textbf{CL} &  
% \textbf{ASP} & %% reward函数中 结构惩罚项
\textbf{PI-Real}&
\textbf{PI-Synth}\\
% \textbf{Pass Rate} \\
\midrule
\ding{55} & \ding{51}  & \ding{51}   &34.31\%  & 14.76\%\\
\ding{51} & \ding{55}  &  \ding{55}  &  40.20\% & 21.19\%\\
\ding{51} & \ding{51} & \ding{55} &  42.16\% & 24.29\%\\
\ding{51} & \ding{51} & \ding{51} & \textbf{43.14\%}& \textbf{25.95\%}\\
% \ding{51} & \ding{51} & \ding{51} & --& \\
\bottomrule
\end{tabular}
\end{table}

% % ASP惩罚项等数据可放附录
% % 可换一种方式  不同训练阶段的消融 CL不同轮次单独拿出来 是否SFT 
% \begin{table}[ht]
% \centering
% \label{tab:ablation}
% \caption{Ablation study of NSPI components.}
% \begin{tabular}{lcccccc}
% \toprule
% \textbf{SOS Data} 
% % \textbf{SOS Data Construction} 
% & \textbf{GRPO} & 
%  \textbf{ASP} & 
% \textbf{CL} &  \textbf{Pass Rate} \\
% \midrule
% \ding{55} & \ding{55} & \ding{55} & \ding{55}  & -- \\
% \ding{51} & \ding{55} & \ding{55} & \ding{55}  & -- \\
% \ding{51} & \ding{51} & \ding{55} & \ding{55} & -- \\
% \ding{51} & \ding{51} & \ding{51} & \ding{55}   & -- \\
% \ding{51} & \ding{51} & \ding{51} & \ding{51} & -- \\
% \ding{51} & \ding{51} & \ding{51} & \ding{51}  & -- \\
% \bottomrule
% \end{tabular}
% \end{table}

% \begin{table}[ht]
% \centering
% \caption{Pass Rate on PolyIneqBench under Different k Value Budgets}
% \begin{tabular}{ccccccc}
% \toprule
% % \textbf{k value} 
% \textbf{k Budget} 
% & 1 & 2 & 4 & 8 & 16 & 32 \\
% \midrule
% \textbf{Pass rate} & -- & -- & -- & -- & -- & -- \\
% \bottomrule
% \end{tabular}
% \end{table}
 
% \subsection{Case Analysis}

%% file: Sec/Appendix.tex
\newpage
\appendix
\onecolumn

% Additional Training Details
% case study
% reward details
% 部分前序知识 由于篇幅不一定放在正文，比如说GRPO可以加一个图来解释? 参考ICML 25 oral有没有类似做法
% 测试集示例表
% 数据分布

\section{Sum of Squares and Semidefinite Programming}
We introduce the relationship between sum-of-squares (SOS) polynomials and semidefinite programming (SDP), which provides the theoretical foundation for the methods employed in this study. 

A polynomial $f(\mathbf{x}) \in \mathbb{R}[\mathbf{x}]$ is called a sum-of-squares (SOS) polynomial if it can be expressed as
% \[
% f(\mathbf{x}) = \sum_{i} f_i(\mathbf{x})^2,
% \]
% where $f_i(\mathbf{x}) \in \mathbb{R}[\mathbf{x}]$.
\begin{align}
f(\mathbf{x}) = \sum_{i} f_i(\mathbf{x})^2, \quad \text{where } f_i(\mathbf{x}) \in \mathbb{R}[\mathbf{x}].
\end{align}

%%%%TODO 是否要加？ 并非所有非负多项式都可以表示为平方和，但 SOS 为多项式非负性提供了一个有效且可计算的充分条件。

\begin{example}
    Consider the polynomial
$f(\mathbf{x}) = 2x_1^4 + 2x_1^3x_2 - x_1^2x_2^2 + 5x_2^4,
\quad \mathbf{x}=(x_1,x_2)\in\mathbb{R}^2$. 
Define
$f_1(\mathbf{x})=\frac{1}{\sqrt{2}}\left(2x_1^2-3x_2^2+x_1x_2\right),%\qquad
f_2(\mathbf{x})=\frac{1}{\sqrt{2}}\left(x_2^2+3x_1x_2\right)$. 
Then $f(\mathbf{x})$ admits the SOS decomposition 
$f(\mathbf{x}) = f_1(\mathbf{x})^2 + f_2(\mathbf{x})^2$, 
which implies that $f(\mathbf{x})\ge 0$ for all $\mathbf{x}\in\mathbb{R}^2$. 
% Equivalently, letting the monomial basis be
% $\mathbf{v}(\mathbf{x}) = [x_1^2,\ x_1x_2,\ x_2^2]^\top$, 
% there exists a symmetric positive semidefinite Gram matrix $\widetilde{G}\succeq 0$ such that 
% $f(\mathbf{x})=\mathbf{v}(\mathbf{x})^\top \widetilde{G}\,\mathbf{v}(\mathbf{x})$. 
\end{example}

Semidefinite programming (SDP) is a class of convex optimization problems, whose standard form can be expressed as 
\begin{align}
\begin{aligned}
\text{minimize} \quad & \langle C, G \rangle \\
\text{subject to} \quad & \langle A_i, G \rangle = b_i, \quad i = 1,\dots,m, \\
& G \succeq 0,
\end{aligned}
\end{align}
where $G$ is a symmetric matrix variable and $G \succeq 0$ denotes that $G$ is positive semidefinite.
% The matrix inner product is defined as 
$\langle A, G \rangle = \mathrm{tr}(A^{\top} G)$  
denotes the matrix inner product.

%%% TODO 是否写成theorem？  ？？？？
\begin{theorem} \cite{parrilo2000structured}%\cite{parrilo2003semidefinite} 
\label{theorem:Gram_matrix_representation}
    A multivariate polynomial $f(x)$ in $n$ variables and of degree $2d$ is a sum of squares (SOS) if and only if there exists a symmetric positive semidefinite (PSD) matrix $\widetilde{G}$ such that
    \begin{equation}\label{eq:gram_representation}
         f(x) = {\mathbf v}(x)^\top \widetilde{G} {\mathbf v}(x), 
    \end{equation}
    where ${\mathbf v}(x)=[1,x_1,x_2,...,x_n,x_1^2,x_1x_2,...,x_n^d ]$ is the vector of monomials up to degree $d$. 
\end{theorem}
The matrix $\widetilde{G}$ is referred to as a \textit{Gram matrix} of $f(x)$ with respect to the monomial basis ${\mathbf v}(x)$.

\begin{remark}
The Gram matrix representation is generally not unique, as different choices of the monomial basis or orthogonal transformations of the SOS components yield different PSD matrices $G$.
\end{remark}

According to \cref{theorem:Gram_matrix_representation}, expanding the right-hand side of equation~\cref{eq:gram_representation} and matching the coefficients with those of $f(x)$ yields a system of linear equality constraints on the entries of the Gram matrix $G$. Consequently, the problem of finding an SOS decomposition of the polynomial can be equivalently reformulated as the SDP problem given in  ~\cref{eq:preliminary_sdp} of ~\cref{sec:preliminary}.  

% TODO 换个SOS项数大于一的例子
\begin{example}\cite{parrilo2003semidefinite}
% ref 2008 ex1 类似 选一个例子 （造数据的 3元例子
Consider the bivariate quartic polynomial
$f(\mathbf{x}) = 2x_1^4 + 2x_1^3x_2 - x_1^2x_2^2 + 5x_2^4,
\quad \mathbf{x}=(x_1,x_2)\in\mathbb{R}^2$. 
Let the monomial basis vector be
${\mathbf v}(\mathbf{x}) = [x_1^2,\ x_2^2,\ x_1x_2]^\top$.
Then $f(\mathbf{x})$ can be written in Gram form as
\begin{align}
f(\mathbf{x})
&= {\mathbf v}(\mathbf{x})^\top G\, {\mathbf v}(\mathbf{x})  \nonumber\\
&=
\begin{bmatrix}
x_1^2\\
x_2^2\\
x_1x_2
\end{bmatrix}^{\!\top}
\begin{bmatrix}
q_{11} & q_{12} & q_{13}\\
q_{12} & q_{22} & q_{23}\\
q_{13} & q_{23} & q_{33}
\end{bmatrix}
\begin{bmatrix}
x_1^2\\
x_2^2\\
x_1x_2
\end{bmatrix} \nonumber\\
&= q_{11}x_1^4 + q_{22}x_2^4 + (q_{33}+2q_{12})x_1^2x_2^2 + 2q_{13}x_1^3x_2 + 2q_{23}x_1x_2^3 .
\end{align}
Matching coefficients with the target polynomial $f(\mathbf{x})$ yields the following linear equalities:
\begin{equation}
q_{11}=2,\qquad q_{22}=5,\qquad q_{33}+2q_{12}=-1,\qquad 2q_{13}=2,\qquad 2q_{23}=0.
\label{eq:gram_linear_constraints_example}
\end{equation}
A positive semidefinite matrix $G\succeq 0$ satisfying~\eqref{eq:gram_linear_constraints_example} can be obtained, e.g., via semidefinite programming. One feasible choice is
\[
G=
\begin{bmatrix}
2 & -3 & 1\\
-3 & 5 & 0\\
1 & 0 & 5
\end{bmatrix}
= L^\top L,
\qquad
L=\frac{1}{\sqrt{2}}
\begin{bmatrix}
2 & -3 & 1\\
0 & 1 & 3
\end{bmatrix}.
\]
Consequently, $f(\mathbf{x})$ admits the explicit SOS decomposition
$f(\mathbf{x})
= \frac{1}{2}\left(2x_1^2-3x_2^2+x_1x_2\right)^2
+ \frac{1}{2}\left(x_2^2+3x_1x_2\right)^2$.
\end{example}

% A real symmetric matrix $G$ is positive semidefinite if and only if it admits a factorization
% \[
% G = L L^\top,
% \]
% where $L$ is a (possibly rectangular) real matrix.
% More generally, $G \succeq 0$ if and only if it admits an $LDL^\top$ decomposition with a diagonal matrix $D$ having nonnegative entries. %%%TODO ref

% TODO 注意更新年限的相关文章

\section{Theoretical Foundations of SOS Data Construction}
% \subsection{Theoretical Foundation and Derivations}
% Theorem 
% Proof 
% Example 
% TODO: 降a  确保严谨 完全正确
% SOS and SDP theorem

%  Gershgorin’s circle theorem \cite{gershgorin1931uber} 证明dd矩阵是半正定的

This part provides the formal theoretical guarantees for the data construction methods described in Section \ref{sec: SOS Data Construction Method}. We focus on the spectral properties and decomposition theorems that ensure the generated Gram matrices $\widetilde{G}$ are positive semidefinite.
\begin{figure}[htbp]
  \centering
  \includegraphics[width=0.5\linewidth]{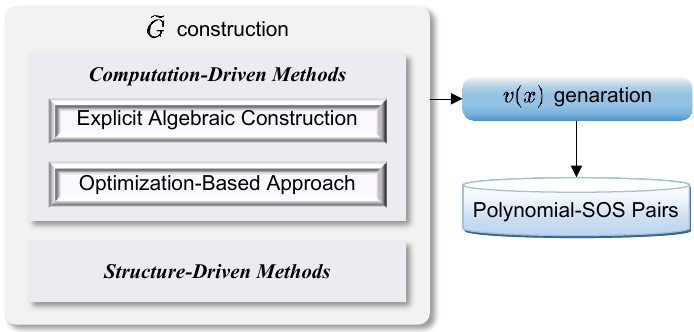}%{figures/data_construction/data construvtion_251114_v2.pdf}%{figures/data construvtion_251112_v2.pdf}%{figures/data construvtion_251112_v1.pdf}
  % \caption{\textbf{New SOS data construction methods}.  
  % }  
  \caption{\textbf{Framework of SOS data construction methods in \cref{sec: SOS Data Construction Method}}. The generation of polynomial-SOS pairs is based on the quadratic form $f(x) = {\mathbf v} (x)^{\mathsf{T}} \widetilde{G} {\mathbf v}(x)$, where the Gram matrix $\widetilde{G}$ is synthesized through either computation-driven or structure-driven approaches.} 
  % \caption{The procedure of polynomial-SOS pairs generation.}
  \label{fig:dataconstruction}
\end{figure}

\subsection{Spectral Foundations for Computation-Driven Methods} 
\label{appendix:foundation_of_Computation-Driven}
The computation-driven methods rely on the relationship between eigenvalues and the PSD property. 

\begin{lemma}[Eigenvalue Shift]
    Let $G \in \mathbb{S}^m$ be a symmetric matrix with real entries. For any $k \in \mathbb{R}$, the eigenvalues of $\widetilde{G} = G - kI$ are given by $\lambda_i(\widetilde{G}) = \lambda_i(G) - k$. Consequently, selecting $k \leq \lambda_{\min}(G)$ 
    is a sufficient condition to ensure %%%%% guarantees that 
    $\widetilde{G} \succeq 0$.
\end{lemma}

To generate Gram matrices with controlled sparsity and coefficient types, we employ a constructive approach based on quadratic forms.
\begin{lemma}%[Constructive Factorization Property]
     Let $L \in \mathbb{Z}^{k \times m}$ be an arbitrary matrix (representing the coefficients of $k$ basis polynomials) and $D = \text{diag}(d_1, \dots, d_k)$ be a diagonal matrix with $d_i > 0$. The resulting matrix
    \begin{equation}
    \widetilde{G} = L^\top D L
    \end{equation}
    is guaranteed to be positive semi-definite ($\widetilde{G} \succeq 0$).
\end{lemma}

% \subsection{Diagonally Dominant (dd) and Scaled Diagonally Dominant (sdd) Cones}
\subsection{Diagonally Dominant (dd) and Scaled Diagonally Dominant (sdd) Cones}
\label{appendix:foundation_of_Structure-Driven} 
% Appendix~\ref{appendix:foundation_of_Structure-Driven} provides the theoretical foundations and more details of structure-driven methods. 

Structure-driven methods restrict the search space to specific sub-cones of the PSD cone to simplify formal verification.

\begin{definition} \cite{ahmadi2019dsos}  %dsos
A matrix $ G \in \mathbb{R}^{m \times m} $  is defined as diagonally dominant (dd) if it satisfies the following condition:
  \begin{align}
         G_{ii} \ge \sum_{j \ne i} |G_{ij}|, \quad \forall i.
  \end{align}
% \end{definition} 
% \begin{definition} \cite{}
    A symmetric matrix $G$ is called scaled diagonally dominant (sdd) if there exists a diagonal matrix $D$ with strictly positive diagonal elements such that the matrix $DGD$ is diagonally dominant.
\end{definition}

% TODO
% 加 解释 为dd和sdd都是半正定矩阵
% \begin{theorem}[Gershgorin Circle Theorem \cite{gershgorin1931uber}]
    
% \end{theorem}

According to Gershgorin’s circle theorem \cite{gershgorin1931uber}, any symmetric diagonally dominant matrix with non-negative diagonal entries is positive semi-definite.  %%%TODO according to 

The following lemma provides the theoretical justification for the extreme ray decomposition used in our framework:
\begin{lemma} \cite{barker1975cones}%\cite{}
    The cone of $m \times m$ symmetric dd matrices is the set of all matrices that can be represented as:
\begin{equation}
\widetilde{G} = \sum_{i=1}^{m^2} \eta_i U_i, \quad \eta_i \geq 0
\end{equation}
where each $U_i = \mathbf{u}_i \mathbf{u}_i^\top$ is a rank-one matrix and $\mathbf{u}_i \in \mathbb{R}^m$ is a vector with at most two non-zero components, each belonging to $\{\pm 1\}$, that is, $\mathbf{u}_i \in \{\pm e_k : k \in [m]\} \,\cup\, \{\pm e_k \pm e_\ell : 1 \le k < \ell \le m\}$.
\end{lemma}

% For the scaled diagonally dominant (sdd) construction, we utilize the property of diagonal congruence to extend the flexibility of the dd cone: 
For the construction of scaled diagonally dominant (sdd) matrices, we utilize the properties of diagonal congruence transformation: 
\begin{lemma} 
    A matrix $G$ is sdd if there exists a positive definite diagonal matrix $D \succ 0$ such that $DGD$ is dd. 
    By Sylvester’s Law of Inertia, diagonal congruence transformations preserve the inertia of a matrix and, in particular, the signs of its eigenvalues. Since $DGD$ is dd, it follows that $DGD \succeq 0$. Consequently,
    $G = D^{-1}(DGD)D^{-1} \succeq 0$. 
\end{lemma}

% A symmetric matrix $G \in \mathbb{R}^{n \times n}$ is called sdd if there 
% exists a positive definite diagonal matrix 
% $D = \mathrm{diag}(d_1,\dots,d_n) \succ 0$ such that
% $G_{\mathrm{dd}} :=D G D$ is dd. 
Since $G$ is dd, it admits the decomposition $\sum_{i=1}^{m^{2}} \eta_i U_i $. 
Applying the diagonal congruence transformation $\widetilde{G} = D^{-1} G D^{-1}$ yields the sdd matrix $\widetilde{G}$ as a non-negative linear combination of scaled rank-one matrices: 
\begin{align}
    \widetilde{G} 
    % = D^{-1} G_{\mathrm{DD}} D^{-1} 
    = \sum_{i=1}^{m^{2}} \eta_i \widehat{U}_i,
\end{align}
where $\widehat{U}_i := D^{-1} U_i D^{-1} 
    = \mathbf{w}_i \mathbf{w}_i^{\top}$ and $\mathbf{w}_i := D^{-1} \mathbf{u}_i$.   
% $\widehat{V}_i := D^{-1} V_i D^{-1} 
%     = \mathbf{w}_i \mathbf{w}_i^{\top}, 
%     \quad \mathbf{w}_i := D^{-1} \mathbf{v}_i$. 
% \begin{align}
%     \widehat{V}_i := D^{-1} V_i D^{-1} 
%     = \mathbf{w}_i \mathbf{w}_i^{\top}, 
%     \quad \mathbf{w}_i := D^{-1} \mathbf{v}_i.
% \end{align}
Each vector $\mathbf{w}_i$ still has at most two nonzero components, whose
magnitudes are determined by the diagonal entries of $D^{-1}$.

% This result justifies the scaled rank-one decomposition
% $\widetilde{G} = \sum \eta_i \widehat{V}_i$
%  used in Section~\ref{sec: SOS Data Construction Method}.

% \section{Theoretical Foundations for Symbolic Correction}
\section{Further Details of Symbolic Correction Module}
\label{appendix:theorey_symbolic}
This section presents the theoretical foundations of the symbolic correction module described in Section~\ref{sec:Symbolic Correction}. The purpose of this module is to convert the SOS structural conjecture obtained from the SOS conjecture stage into an exact rational sum-of-squares certificate, ensuring suitability for formal verification.

%%% TODO 改表述 
% \subsection{Convergence of Gauss-Newton Refinement for SOS}
\subsection{Gauss-Newton Refinement for SOS}
\label{appendix:newton_}
Suppose a polynomial $f(x) \in \mathbb{R}[x]$ of degree $2d$ is conjectured to admit a sum-of-squares (SOS) decomposition. 
 Let ${\mathbf v}(x)$ denote the monomial basis consisting of all monomials of degree at most $d$. 
 The SOS condition is equivalently expressed in Gram matrix form as 
$f(x) = {\mathbf v}(x)^\top G {\mathbf v}(x)$, 
where $G \succeq 0$ is a symmetric positive semidefinite Gram matrix.
In practice, numerical solvers only provide an approximate Gram matrix $G$, such that 
$f(x) \approx {\mathbf v}(x)^\top G {\mathbf v}(x)$. 

To reduce numerical errors and improve stability prior to exact rational recovery, we apply a structure-preserving Gauss–Newton refinement.

\subsubsection{Numerical Newton Refinement of an Approximate Gram Matrix}
\label{app:Newton}

\paragraph{Factorization-based parameterization.}
Assume $G\succeq 0$ numerically and compute a factorization
\begin{equation}
G \approx L L^\top,
\label{eq:app_chol}
\end{equation}
where $L\in\mathbb{R}^{m\times k}$ and $k=\mathrm{rank}(G)$.
% (numerical rank).
Then
\begin{equation}
{\mathbf v}(\mathbf{x})^\top G\, {\mathbf v}(\mathbf{x})
\approx \sum_{i=1}^k \ell_i(\mathbf{x})^2,
\quad
\ell_i(\mathbf{x}) := \sum_{\alpha} c_{i,\alpha} \mathbf{x}^{\alpha}. 
% \ell_i(\mathbf{x}) := L_{:,i}^\top {\mathbf v}(\mathbf{x}).  %%% TODO
\label{eq:app_sos_from_L}
\end{equation}

We refine the coefficient vectors $L_{:,i}$ via a Gauss--Newton update so that the induced polynomial
$\sum_i \ell_i(\mathbf{x})^2$ matches $f(\mathbf{x})$ as accurately as possible.  %%% 换表述

\begin{proposition}%[Convergence of the Gauss–Newton Method] 
By applying a Cholesky factorization $G = LL^{\top}$ or an $LDL^{\top}$ decomposition to the Gram matrix $G$, and treating the coefficients of the SOS factors as optimization variables, the Gauss–Newton iteration exhibits rapid convergence when applied to SOS problems. Provided that sufficient numerical precision is employed, the backward error $\theta$ 
can be reduced to arbitrarily small values. 
% Let $c$ denote the coefficient vector of the SOS factor polynomials, such that
% $f(x) \approx \sum_{i=1}^{k} h_i(x,c)^2$.
% The Gauss–Newton iteration computes a correction $\Delta c$ by solving the following linearized least-squares problem:
% $$
% \min_{\Delta c}
% \left|
% f(x) - \sum_{i=1}^{k} \bigl(h_i(x,c) + \nabla_c h_i(x,c),\Delta c \bigr)^2
% \right|_2 .
% $$
% This iterative procedure is locally convergent with quadratic convergence to a local optimal solution $c^\ast$. 
\end{proposition}

This yields a high-accuracy numerical approximation, which serves as a reliable foundation for the subsequent rational recovery procedure.

\paragraph{Residual and stopping criterion.}
Let $\mathrm{coeff}(\cdot)$ denote the coefficient vector of a polynomial in the chosen monomial basis.
Define the residual
\begin{equation}
r(L) := \mathrm{coeff}\!\left(\sum_{i=1}^k \ell_i(\mathbf{x})^2 - f(\mathbf{x})\right),
\qquad
\theta := \|r(L)\|_2 .
\label{eq:app_residual}
\end{equation}

Here, $\theta$ denotes the backward error of the numerical Gram matrix $G$.
We apply the Gauss–Newton iteration to compute the coefficient correction terms $\Delta c_{i,\alpha}$ so as to minimize $\theta$, as described in \cref{math:newton_expression} of \cref{sec:Newton}. 

 Simultaneously, the Gram matrix is updated according to
$G \leftarrow G + \Delta G$, 
where the correction term $\Delta G$ can be expressed as
\[
\sum_{i=1}^{k} \left( \sum_{\alpha} \Delta c_{i,\alpha} x^{\alpha} \right)^2
= {\mathbf v}(x)^{\top} \Delta G\, {\mathbf v}(x).
\]

The iteration terminates
once $\theta<\tau$ for a prescribed tolerance $\tau>0$.

%%% TODO 应该只写牛顿迭代的具体理论？  不应该再介绍方法 ？

% \subsection{Existence of Rational SOS Certificates}
\subsection{Detailed Formulation of Exact Rational Recovery}
\label{app:RationalRecovery}

The following proposition establishes the fundamental equivalence between the existence of a rational sum-of-squares (SOS) decomposition and the existence of a Gram matrix with rational entries. 

\begin{proposition} \cite{PEYRL2008269}%[Rational SOS decomposition $\Leftrightarrow$ rational Gram matrix]
\label{prop:rational_sos_gram}
Let $f(\mathbf{x}) \in \mathbb{Q}[\mathbf{x}]$ be a polynomial, and let ${\mathbf v}(\mathbf{x})$ denote a fixed vector of monomials. 
Then the following statements are equivalent:
\begin{enumerate}
    \item The polynomial $f(\mathbf{x})$ admits a rational SOS decomposition, that is, there exist polynomials $f_i(\mathbf{x}) \in \mathbb{Q}[\mathbf{x}]$ such that
    \[
        f(\mathbf{x}) = \sum_{i=1}^{r} f_i(\mathbf{x})^2 .
    \]
    \item There exists a symmetric positive semidefinite Gram matrix $G \in \mathbb{S}^m \cap \mathbb{Q}^{m \times m}$ such that
    \[
        f(\mathbf{x}) = {\mathbf v}(\mathbf{x})^\top G\, {\mathbf v}(\mathbf{x}), \qquad G \succeq 0 .
    \]
\end{enumerate}
\end{proposition}

% \begin{proposition} \cite{PEYRL2008269}
%     The 
% \end{proposition}

For the numerical solution $G_N$ obtained after the Gauss–Newton refinement, when the backward error $\theta$ is sufficiently small, one can recover from $G_N$ an exact rational PSD matrix 
$\widetilde{G}$, which satisfies the polynomial identity exactly, thereby yielding a certified rational SOS certificate. Which satisfies the following identity:

\begin{equation}
f(x) - {\mathbf v}(x)^\top \widetilde{G} {\mathbf v}(x) = 0,
\qquad \widetilde{G} \succeq 0 .
\label{eq:exact_sos_certificate_goal}
\end{equation}

The matrix $G$ is projected onto the following affine hyperplane: 
\begin{equation}
\mathcal{X}
= \left\{
G \;\middle|\;
G^{\top} = G, f(x) - {\mathbf v}(x)^{\top} G\, {\mathbf v}(x) = 0\;
\right\}.
\label{eq:hyperplane}
\end{equation}
Suppose that the affine hyperplane defined by \cref{eq:hyperplane} can be represented by a linear system $Ay = b$, 
where $y$ consist of the entries of $G$. 
If the matrix $A$ has full row rank, then such a hyperplane is guaranteed to exist. 

The recovery strategy depends on whether the refined Gram matrix lies in the interior of the PSD cone. %% TODO

An exact solution satisfying \cref{eq:exact_sos_certificate_goal} can be obtained via the following two approaches:
\begin{itemize}
    \item Case 1: If the matrix $G_N$ is of full rank, the solution is recovered by applying an orthogonal projection method.
    \item Case 2: Otherwise, a rational vector recovery method is employed.
\end{itemize}

\paragraph{Case 1: Interior Point Solution ($G_N$ is full rank).} The solution of \cref{eq:exact_sos_certificate_goal} lies in the intersection of the affine hyperplane $\mathcal{X}$ and the positive semidefinite cone.
To compute $\widetilde{G}$, the corresponding orthogonal projection can be obtained by solving the following least-squares problem. 

\begin{equation}
\min_{G \succeq 0}
\;\; \lVert G_N - G \rVert_F^2
\quad
\text{s.t. }
f(x) = {\mathbf v}(x)^{\top} G\, {\mathbf v}(x).
\label{eq:problem}
\end{equation}

Next, to verify whether the recovered rational solution $\widetilde{G}$ is a symmetric positive semidefinite matrix, we compute the exact $LDL^{\top}$ decomposition of $\widetilde{G}$.

\begin{equation}
    f(x) = {\mathbf v}(x)^\top \widetilde{G} {\mathbf v}(x) = {\mathbf v}(x)^\top LDL^\top {\mathbf v}(x)
\end{equation}

\begin{theorem}
    Let $G_N$ be the refined numerical solution whose backward error satisfies $\theta < \tau$, and let $Ay = b$ be the linear system associated with the affine hyperplane defined by \cref{eq:hyperplane}.
Suppose that $\widetilde{G}$ is the optimal rational solution of the least-squares \cref{eq:problem}, and that the matrix $A$ has full row rank.
If the minimal eigenvalue $\lambda$ of $\widetilde{G}$ satisfies
\begin{equation}
    \lambda > \lVert G_N \rVert_F^2 \kappa_2^2(A)\tau^2,
\end{equation}

then $\widetilde{G}$ is an exact solution of \cref{eq:exact_sos_certificate_goal}. 
\end{theorem}

\textit{Proof.} 
Since $\widetilde{G}$ is a solution of Problem (\ref{eq:problem}), it clearly satisfies the polynomial identity $
f(x) = {\mathbf v}(x)^{\top} \widetilde{G} {\mathbf v}(x)$.
Let $y_N$ and $\widetilde{y}$ denote the vectors consisting of the entries of $G_N$ and $\widetilde{G}$, respectively. From \eqref{eq:hyperplane}, and under the assumption that the matrix $A$ has full row rank, we have
\begin{equation}
 \lVert A y_N - b \rVert_2^2 = \theta < \tau, \qquad A \widetilde{y} = b .
\end{equation}
According to the perturbation result \cite{golub2013matrix} for full-row-rank underdetermined linear systems, %% TODO check ref
the following estimate holds: 
\begin{equation}
 \lVert y_N - \widetilde{y} \rVert_2 \le (\kappa_2(A)   \tau)\lVert y_N\rVert_2 + O(\tau^2).
\end{equation}
Under the assumption that
$
\lambda >  \lVert G_N \rVert_F^2 \kappa_2^2(A) \tau^2,
$
it follows that
\begin{equation}
\lVert G_N-\widetilde{G} \rVert_F^2 \le  \lVert y_N-\widetilde{y} \rVert_2^2 < \lambda
\end{equation}
where the inequality $ \lVert y_N - \widetilde{y} \rVert_2^2 < \lambda $ holds because the higher-order term $O(\tau^2)$ is negligible when $\tau$ is sufficiently small.
Let $\widetilde{\lambda}$ denote an eigenvalue of $\widetilde{G}$. By the Wielandt–Hoffman theorem \cite{golub2013matrix}, %% TODO check
we obtain
% | \widetilde{λ}-λ |≤‖G-\widetilde{G} ‖_F^2≤λ
\begin{equation}
 | \widetilde{\lambda} - \lambda |\le  \lVert G_N - \widetilde{G} \rVert_F^2 \le \lambda,
\end{equation}
which implies that all eigenvalues of $\widetilde{G}$ are nonnegative. Therefore, we conclude that 
$\widetilde{G} \succeq 0$.

\paragraph{Case 2: Boundary Solution ($G_N$ is rank-deficient).} 
When the numerical Gram matrix $G_N$ obtained via Gauss–Newton refinement is rank-deficient or ill-conditioned (near-singular), the affine hyperplane $\mathcal{X}$ defined by the linear constraints is typically tangent to the boundary of the positive semidefinite (PSD) cone. In such instances, $G_N$ does not reside within the interior of the PSD cone; consequently, direct orthogonal projection may yield a rational matrix $\widetilde{G}$ that violates the semi-definiteness requirement. We provides a rigorous discussion on the structural origins of such cases and presents a theoretical framework for exact recovery based on rational vector reconstruction. 
\begin{itemize}
    \item \textbf{Redundant Monomials}: In the construction of the SOS decomposition, a monomial basis that is not strictly necessary may be used. This situation can be avoided by exploring the sparse structure of the polynomial or by removing entire rows and columns of the Gram matrix $G_N$ that correspond to numerically small entries, i.e., eliminating the monomials that should not appear in the polynomial's SOS representation.
    % Such cases can be regularized by analyzing the sparsity structure of the Gram matrix $G$ or by truncating the rows and columns corresponding to negligible eigenvalues, thereby transforming the problem into a full-rank formulation.
    \item \textbf{Intrinsic Singularity}: If the polynomial $f(x)$ attains its global minimum at a nonzero real point $(\xi_1, \dots, \xi_n)$, then the monomial vector ${\mathbf v}(\xi$) at this point is a zero vector in the Gram matrix $G_N$, implying that $G_N$ is singular. 
    When the global minimum is attained only at finitely many nonzero points, and the backward error $\theta$ is sufficiently small, the Gram matrix can be rendered full rank by performing Gauss–Newton refinement, after which an orthogonal projection can be applied. 
    In contrast, when the global minimum is attained on some manifolds, we apply the Gauss–Newton iteration to a truncated triangular decomposition of $G_N$. Once the residual is sufficiently small, a rational recovery of $G_N$ is carried out via a simultaneous Diophantine approximation algorithm ~\cite{lagarias1985computational}.
\end{itemize}

% \textit{Example.} 

% % After Newton refinement, we obtain a high-precision floating-point matrix $G_{\mathrm{N}}$. 
% % that is close to the exact feasible set $\mathcal{A}_f$. 
% The rational recovery step constructs an
% \emph{exact} rational Gram matrix,   
% % $G_{\mathbb{Q}}\in\mathbb{Q}^{m\times m}$, 
% which yields an exact SOS decomposition with rational coefficients.

% \paragraph{Case 1: $G$ is (numerically) full rank.}
% If $G_{\mathrm{num}}$ is well inside the PSD cone (full rank with a clear spectral gap), small perturbations
% caused by rationalization are unlikely to violate positive semidefiniteness. In this case, the pipeline
% ``projection $\rightarrow$ rationalization $\rightarrow$ exact PSD verification'' is typically stable.

% \paragraph{Case 2: $G$ is (numerically) rank deficient.}
% If $G_{\mathrm{num}}$ is close to the boundary (nearly singular), direct entrywise rationalization may destroy
% PSD-ness by introducing small negative eigenvalues. A more robust approach is to recover rational factors:
% we compute a truncated factorization $G_{\mathrm{num}}\approx L L^\top$ (or $LDL^\top$),
% rationalize the factor(s) (e.g., columns of $L$) using simultaneous Diophantine approximation,
% and reconstruct $G_{\mathbb{Q}} := L_{\mathbb{Q}} L_{\mathbb{Q}}^\top$ so that the intended rank structure is preserved.
% Finally, we enforce $G_{\mathbb{Q}}\in\mathcal{A}_f$ by solving a constrained rational least-squares step on the
% free parameters.

\section{Neural Conjectur Midule Configuration }
\label{appendix:neural_conjecture_configuration}
% \subsection{Detailed Experimental Setup}

\subsection{Prompt Template}
% \subsubsection{Prompt Template in Neural Conjecture Module}
Here we present the prompt template employed by the SOS structure conjecture module in the {\it neural conjecture} module. For a given nonnegative polynomial, the SOS structure conjecturer generates a candidate SOS structural conjecture based on an expanded-form representation, with the objective that the proposed SOS structure matches the target polynomial as closely as possible. 
% 定义一个带有圆角矩形背景的盒子环境
\begin{tcolorbox}[colback=gray!5!white, colframe=blue!50!black, arc=3mm, title={\texttt{Prompt Template for SOS Structure Conjecturer}}]
% \begin{tcolorbox}[colframe=blue!40!black, colback=blue!5, title=\textbf{Prompt Template}] 
    \textbf{Task:}
    \\
    You are given a polynomial that is the expanded form of a sum-of-squares (SOS) expression.
    Your task is to reconstruct a plausible SOS representation whose expanded form matches the given polynomial as closely as possible.
    
    \medskip
    \textbf{Instructions:} 
    
    1. Analyze the input polynomial carefully, focusing on the coefficients and the combinations of variables involved.  \\
    2. Infer possible linear or polynomial terms inside each square in the sum-of-squares expression.  \\
    3. Construct a sum of square terms without expanding the squares. Keep the output compact and well-structured.\\  
    4. Aim for the expanded form of your SOS expression to closely approximate the coefficients in the original polynomial. Minor numerical deviations are acceptable. \\ 
    5. If multiple valid SOS decompositions exist, prefer one that is simple, symmetric, and easy to interpret.\\  
    6. Include variables and constants inside parentheses when appropriate to match constant terms in the input polynomial.\\
    % \textbf{Instructions:}
    % \begin{enumerate}
    %     \item Analyze the input polynomial carefully, focusing on the coefficients and the combinations of variables involved.
    %     \item Infer possible linear or polynomial terms inside each square in the sum-of-squares expression.
    %     \item Construct a sum of square terms without expanding the squares. Keep the output compact and well-structured.
    %     \item Aim for the expanded form of your SOS expression to closely approximate the coefficients in the original polynomial. Minor numerical deviations are acceptable.
    %     \item If multiple valid SOS decompositions exist, prefer one that is simple, symmetric, and easy to interpret.
    %     \item Include variables and constants inside parentheses when appropriate to match constant terms in the input polynomial.
    % \end{enumerate}
    
    \medskip
    
    \textbf{Output format:}
    \\
    Please provide your response in the following structure:\\
    \texttt{<SOS Expression>: <sum\_of\_squares\_expression>}
    
    % \medskip
    
    Where \texttt{<sum\_of\_squares\_expression>} is a sum-of-squares term, expressed compactly.
    \\
    For example:
    \\
    \texttt{(x1 + 1)\textasciicircum 2 + (2*x1 + 3*x2)\textasciicircum 2}
    
    \medskip
    
    \textbf{Key considerations:}
    
    1. Do not simply rewrite the input polynomial or output expanded terms. The output must be a sum-of-squares expression.  \\
    2. Constants inside square terms can be fractional or decimal, as needed, to best approximate the original polynomial.  \\
    3. Avoid expanding the squares in the output; always keep terms inside parentheses squared.  \\
    4. Aim for clear and interpretable variable groupings. Symmetry and simplicity are preferred if multiple answers fit.  \\
    5. Your SOS expression should fully explain the input polynomial’s structure and coefficients to the best achievable extent. \\
    % \begin{enumerate}
    %     \item Do not simply rewrite the input polynomial or output expanded terms. The output must be a sum of squared expressions.
    %     \item Constants inside squared terms can be fractional or decimal, as needed, to best approximate the original polynomial.
    %     \item Avoid expanding the squares in the output; always keep terms inside parentheses squared.
    %     \item Aim for clear and interpretable variable groupings. Symmetry and simplicity are preferred if multiple answers fit.
    %     \item Your SOS expression should fully explain the input polynomial’s structure and coefficients to the best achievable extent.
    % \end{enumerate}
    
    \medskip
    
    \textbf{Example:}
    \\
    Input polynomial:
    \\
    \texttt{x1\textasciicircum 2 + 2*x1 + 1}
    
    \noindent Output:
    \\
    \texttt{(SOS Expression): (x1 + 1)\textasciicircum 2}
    
    \noindent Explanation:
    \\
    The polynomial is a perfect square trinomial; the SOS reconstruction is exact.
    
    \medskip
    
    \noindent Input polynomial:
    \\
    \texttt{5*x1\textasciicircum 2 + 12*x1*x2 + 6*x1 + 9*x2\textasciicircum 2 + 9}
    
    \noindent Output:
    \\
    \texttt{(SOS Expression): (x1 + 2.99)\textasciicircum 2 + (2*x1 + 3*x2)\textasciicircum 2}
    
    \medskip
    
    \noindent Now, please provide the SOS reconstruction for the following polynomial:
    \\
    \textbf{Original polynomial: \{polynomial\}}
\end{tcolorbox}

As illustrated by the prompt template above, it consists of a task description, multiple explicit instructions specifying the required output format and content constraints for the large language model, and several illustrative examples provided to guide the generation process.

% \subsection{Training Dataset}  
\subsection{Synthetic Dataset Statistics}   %%% 合成数据统计 即训练数据
\label{appendix:synthetic_data_statistics}
% 合成数据作为训练语料来源  
% 对于合成数据的说明
% 随机抽样人工检查

%%% TODO  加类似于 人工打分统计表 的内容 

During the data construction stage of the {\it neural conjecture} module, we generate more than 100w synthetic data instances using the four data generation strategies proposed in Section~\ref{sec: SOS Data Construction Method}, including both \textbf{computation-driven} and \textbf{structure-driven} approaches. These synthetic datasets, covering problems with varying numbers of variables ranging from 3 to 10, are used as the training corpus for the neural conjecturer.

Notably, we performed a \textbf{manual quality inspection} of the synthesized polynomial–SOS pair data via random sampling. For each variable setting from 3 to 10, we randomly sampled 50 instances for review. The inspection results indicate that the synthesized data contain no formatting errors and \textbf{encompass polynomial SOS decomposition cases of varying difficulty}, 
% including a substantial number of \textbf{competition-level challenging polynomials}. 
including a substantial proportion of \textbf{highly non-trivial instances} that pose significant challenges to conventional symbolic solvers. 
Additionally, we computed the degree and number of SOS terms of the synthesized data.  
% (i.e., the number of SOS terms / the expression length). 
The statistical results are reported in Fig.~\ref{fig:statistical_results_of_synthetic_data_degree} and Fig.~\ref{fig:statistical_results_of_synthetic_data_sos}.

\begin{figure}[htbp]
    \centering
    \includegraphics[width=0.7\linewidth]{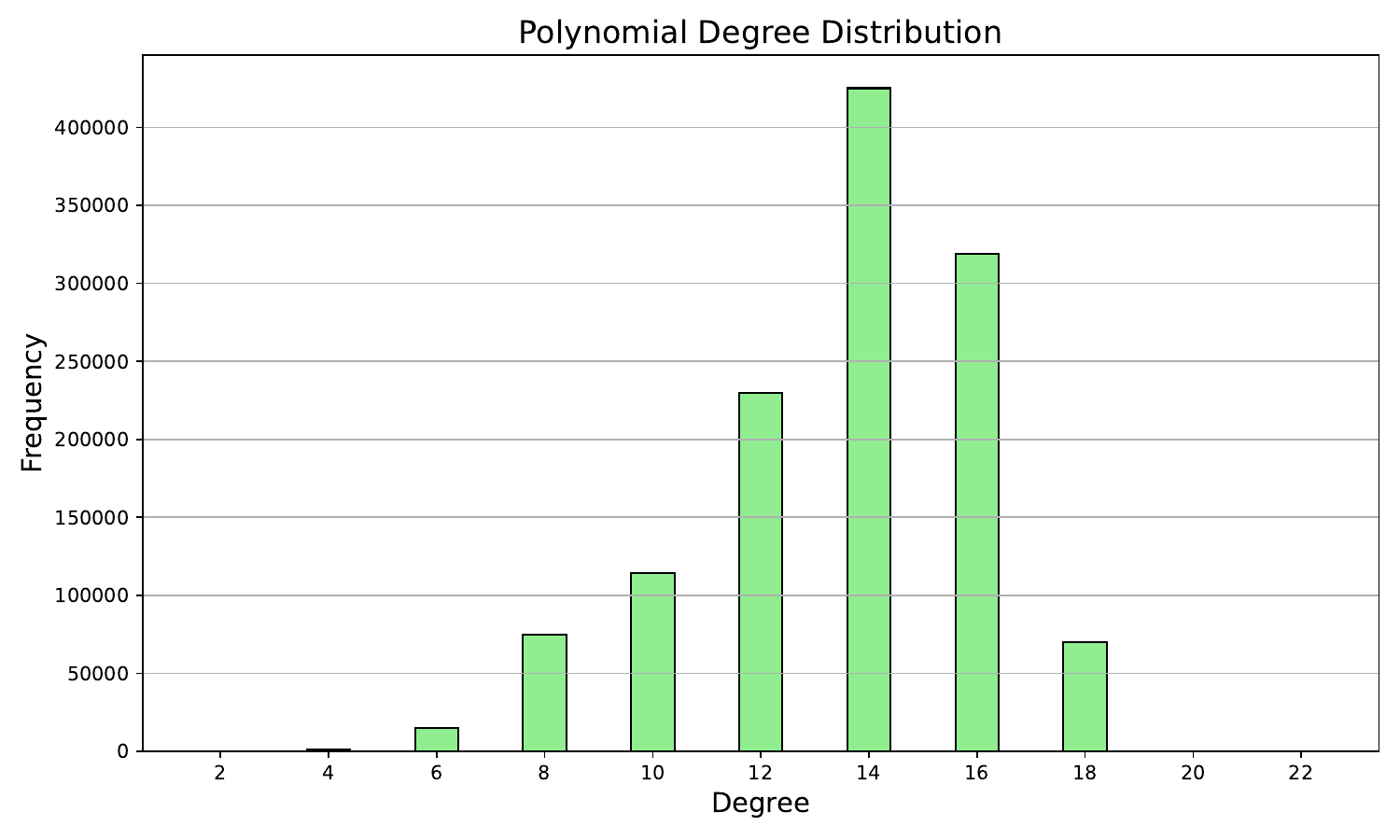}
    \caption{Statistical distribution of degrees for synthetic polynomial data.}
    \label{fig:statistical_results_of_synthetic_data_degree}
\end{figure}

\begin{figure}[htbp]
    \centering
    \includegraphics[width=0.7\linewidth]{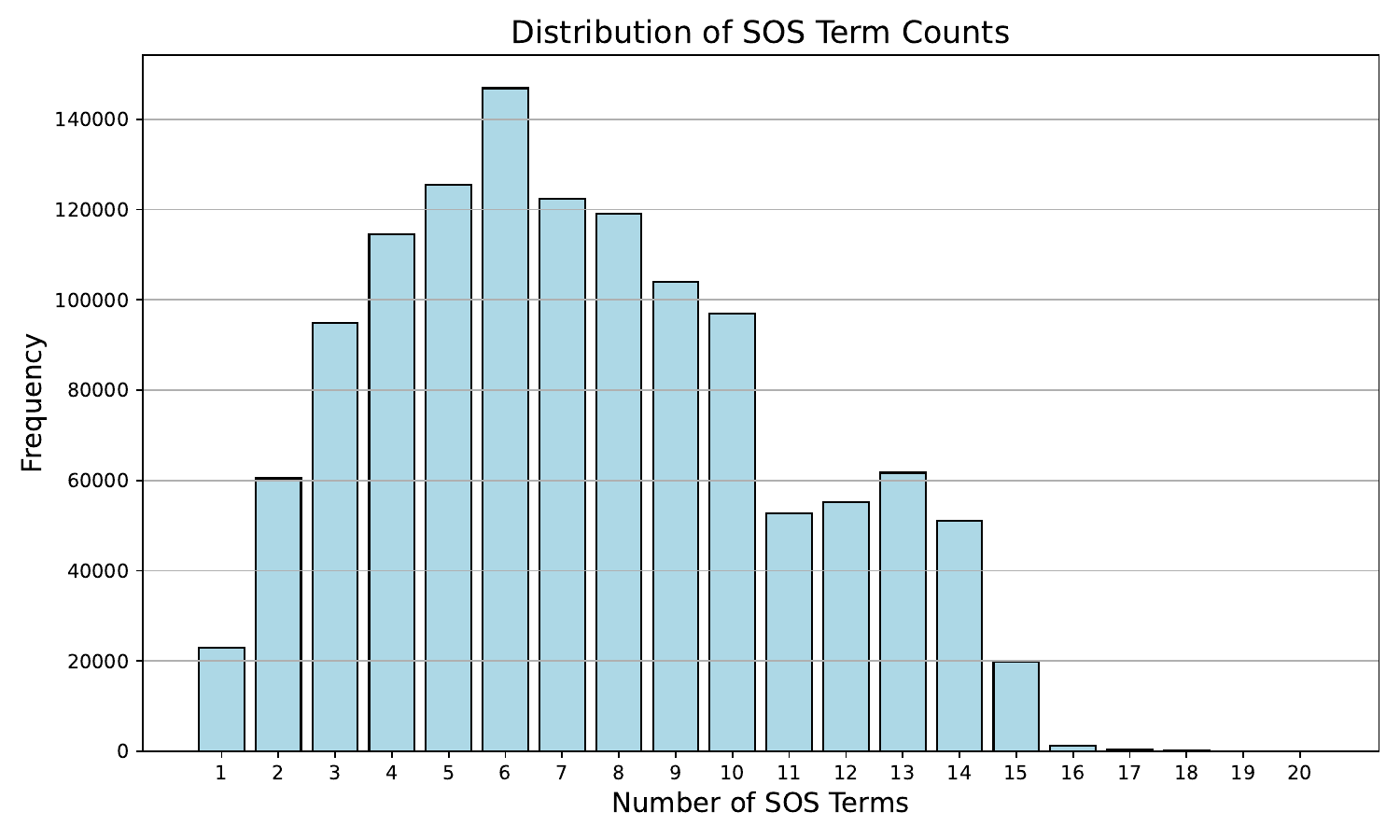}
    \caption{Statistical distribution of SOS term counts in synthetic polynomials.}
    \label{fig:statistical_results_of_synthetic_data_sos}
\end{figure}

As illustrated in Fig.~\ref{fig:statistical_results_of_synthetic_data_degree} and Fig.~\ref{fig:statistical_results_of_synthetic_data_sos}, the synthetic dataset exhibits significant diversity in both algebraic complexity and representation structure. Specifically, the degrees of the synthesized polynomials span a wide range from 2 to 22, ensuring the model is exposed to problems of varying algebraic depths; this reflects that the synthetic data encompasses both relatively simple low-degree polynomials and structurally sophisticated high-degree instances. Furthermore, the distribution of SOS term counts covers a spectrum ranging from basic single-term squares to complex decomposition instances involving up to 15 terms. Such structural diversity indicates that our construction strategies effectively cover a broad difficulty spectrum, preventing the model from over-fitting to specific sparse or low-degree patterns.

\subsection{Additional Details of Progressive Reinforcement Learning Process}  %%% 
\label{appendix:details_of_RL}
For the progressive two-stage training procedure of the SOS-structure conjecturer described in Section~\ref{subsec:Training the SOS Structure Conjecturer}, in the second stage we employ a \textbf{curriculum-style GRPO} reinforcement learning process. Here, we provide the theoretical foundations and several implementation details.

\paragraph{Group Relative Policy Optimization (GRPO).} 
% GRPO可放在附录 or 精简 // 参考其他文章
% \subsection{Group Relative Policy Optimization (GRPO).}
Group Relative Policy Optimization (GRPO) \cite{shao2024deepseekmath} is a reinforcement learning algorithm used for fine-tuning large language models (LLMs). 
Unlike the popular Proximal Policy Optimization (PPO), GRPO enhances efficiency by eliminating the need for a separate value function. It estimates the advantage by normalizing the rewards of a set of responses to the same prompt. Specifically, for each question $q$ in a given set $Q$, a set of responses $\{o_1, \dots, o_G\}$ is sampled from the old policy $\pi_{\text{old}}$. The reward model then evaluates these responses, generating rewards $\{r_1, \dots, r_G\}$, and the advantage is computed as follows:
\begin{align}
     \hat{A}i = \frac{r_i - \text{mean}({r_1, \dots, r_G})}{\text{std}({r_1, \dots, r_G})}
\end{align}
The policy model $\pi_\theta$ is then optimized by maximizing the following objective:
\begin{multline}
J_{\text{GRPO}}(\theta) =
 \mathbb{E}_{q \sim Q, \{o_i\}{i=1}^G \sim \pi_{\text{old}}(O|q)} 
 \frac{1}{G} \sum_{i=1}^G \frac{1}{|o_i|} \sum_{t=1}^{|o_i|}
 \{  
 \min  \left[
 \gamma_{i,t}(\theta) \hat{A}_{i,t}, 
 \text{clip}\left(\gamma_{i,t}(\theta), 1 - \epsilon,  1 + \epsilon\right)  \hat{A}_{i,t}   \right] 
- \beta D_{\text{KL}}(\pi_\theta || \pi_{\text{ref}})
\}
\end{multline}  %%正确

% \begin{multline}
%     J_{\text{GRPO}}(\theta) =
%  \mathbb{E}_{q \sim Q, \{o_i\}{i=1}^G \sim \pi_{\text{old}}(O|q)} \\
%  \frac{1}{G} \sum_{i=1}^G \frac{1}{|o_i|} \sum_{t=1}^{|o_i|}
%  \{  
%  \min  [
%  \gamma_{i,t}(\theta) \hat{A}_{i,t}, 
%  \text{clip} (\gamma_{i,t}(\theta), 1 - \epsilon, \\ 1 + \epsilon )  \hat{A}_{i,t}   ] 
% - \beta D_{\text{KL}}(\pi_\theta || \pi_{\text{ref}})
% \}
% \end{multline}
      
 where $\gamma_{i,t}(\theta) = \frac{\pi_\theta(o_{i,t}|q, o_{i,<t})}{\pi_{\text{old}}(o_{i,t}|q, o_{i,<t})}$ is the importance sampling ratio, $ \pi_{\text{ref}} $ represents the reference model, $\pi_{\text{old}}$ is the policy used to sample the responses, and $ D_{\text{KL}}(\pi_\theta || \pi_{\text{ref}})$ introduces a KL divergence constraint to limit the deviation of the model from the reference model.

\paragraph{Difficulty-Based Curriculum Data Partitioning.} 
As discussed in ~\cref{subsec:Training the SOS Structure Conjecturer}, we partition the training data based on the difficulty of polynomial SOS decomposition in the dataset, obtaining multiple groups of training data with varying difficulty levels. Each group is then trained using a Graduated Reinforcement Policy Optimization (GRPO) approach, progressing from easier to more challenging data. Specifically, we categorize the data into three tiers based on the ascending number of polynomial variables: Polynomial-SOS pairs data with 3 to 5 variables are used in the first stage, data with 6 to 8 variables in the second stage, and the more challenging data with 9 to 10 variables in the third stage. GRPO training is conducted sequentially across these three stages.

% \paragraph{Computation of Algebraic Structure Penalty.} 
\paragraph{Reward Functions Design.} As illstrated in \cref{subsec:Training the SOS Structure Conjecturer}, 
% the reaward function in GRPO is designed with three critical components: 
the reward function design includes three key components:
\begin{itemize}
    \item \textbf{[accuracy reward]} encourages SOS structure conjectures with smaller errors compared to the original polynomial.
    \item \textbf{[format reward]} ensures that the model-generated conjectures adhere to the required SOS structure.
    \item \textbf{[algebraic structure penalty]} penalizes the SOS structure hypothesis based on the degree of term matching with the original polynomial. 
\end{itemize}
The total reward is defined as:
\begin{align}
    R = w_{\text{acc}} R_{\text{acc}} + w_{\text{fmt}} R_{\text{fmt}} - P_{\text{struct}},
\end{align}
% The total reward $R$ for a generated SOS conjecture $\hat{f}(x)$ is defined as:
where $w_{\text{acc}} + w_{\text{fmt}} = 1$, $R_{\text{acc}} \in [0,1]$ denotes the accuracy reward based on the approximation error, $R_{\text{fmt}} \in [0,1]$ denotes the format consistency reward, $P_{\text{struct}}$ denotes the algebraic structure penalty.

% The individual components are defined as follows: 
% % After supervised fine-tuning, a rule-based reinforcement learning algorithm is employed to further optimize the model's ability to reason about SOS structures 
% % by learning from easy to difficult tasks. 
% % Specifically, the portion of the training data that the model cannot solve after supervised fine-tuning is divided into curriculum data, organized by increasing number of variables, allowing the model to learn progressively from easier to more difficult tasks. %CL数据划分
% % The GRPO algorithm is used for model optimization, and 
% % % two reward mechanisms are designed: 
% % the reward function design includes three key components: 
% \textbf{accuracy reward} encourages SOS structure conjectures with smaller errors compared to the original polynomial, while \textbf{format reward} ensures that the model-generated conjectures adhere to the required SOS structure, 
% and an \textbf{algebraic structure penalty} 
% penalizes the SOS structure hypothesis based on the degree of term matching with the original polynomial. 

\paragraph{Computation of Algebraic Structure Penalty.} 
The algebraic structure penalty consists of two components: a soft penalty and a hard penalty. 
 To accurately assess the algebraic structure, we express the original polynomial $ f $ and the model-generated SOS conjecture $\hat{f}$ in terms of a common monomial basis. A coefficient threshold $\tau $ (e.g., $ 10^{-5} $) is set, such that a monomial is considered present in the polynomial (i.e., a nonzero term) only if the absolute value of its coefficient exceeds $ \tau $.

    % 将原多项式 $f$ 与模型生成的 SOS 猜想展开式 $^\hat f$ 写成相同单项式基底下的系数形式后。给定阈值 $\tau$（如 $10^{-6}$），将绝对值小于 $\tau$ 的系数视为零。 TODO
     We define Structural Deviation Rate (SDR) to quantify the discrepancy between the algebraic structure of the SOS representation and that of the original polynomial:
     \begin{align}
          SDR = \frac{N_\text{miss} + N_\text{spur}}{N_\text{req} },
          % SDR = \frac{N_\text{miss} + N_\text{spur}}{N_\text{req} + \epsilon},
     \end{align}
     where $N_{\text{req}}$ denotes the number of nonzero monomials in the original polynomial, $N_{\text{miss}}$ is the number of monomials that appear in the original polynomial but are absent in $\hat f$, and $N_{\text{spur}}$ is the number of 
     \textbf{extraneous terms}
     % \textbf{spurious terms} 
     present in $\hat f$ but not belonging to the monomial set $\mathcal{V}_f$.
     
    The soft penalty is defined as
    \begin{align}
         P_{\text{struct-soft}} = \lambda \cdot \min(\text{SDR}, \rho_{\max}),
    \end{align}
     where $\lambda$ controls the strength of the penalty and $\rho_{\max}$ caps the maximum penalty.

    If the generated SOS structure exhibits severe algebraic inconsistencies (e.g., exceeding the degree of the original polynomial or introducing variables not present in it), a hard penalty is imposed:
    \begin{align}
          P_{\text{struct-hard}} =
         \begin{cases}
        C_\text{hard}, & \text{if a structural violation occurs} \\
         0, & \text{otherwise}
 \end{cases}
    \end{align}
    
 The overall algebraic structure penalty is the sum of the soft and hard penalties:
    \begin{align}
         P_{\text{struct}} = P_{\text{struct-hard}} + P_{\text{struct-soft}}.
    \end{align}
% \paragraph{Reward functions in GRPO} 

% The final reward is given by
% \begin{align}
%     R = w_{\text{acc}} R_{\text{acc}} + w_{\text{fmt}} R_{\text{fmt}} - P_{\text{struct}},
% \end{align}
% where $w_{\text{acc}} + w_{\text{fmt}} = 1$, $R_{\text{acc}} \in [0,1]$ denotes the accuracy reward based on the approximation error, $R_{\text{fmt}} \in {0,1}$ denotes the format consistency reward, $P_{\text{struct}}$ denotes the algebraic structure penalty.

\section{Further Details of Lean Verification Module}
\label{apendix:more_about_lean_verification}

\begin{figure*}[htbp]
  \centering
  \includegraphics[width=0.85\linewidth]{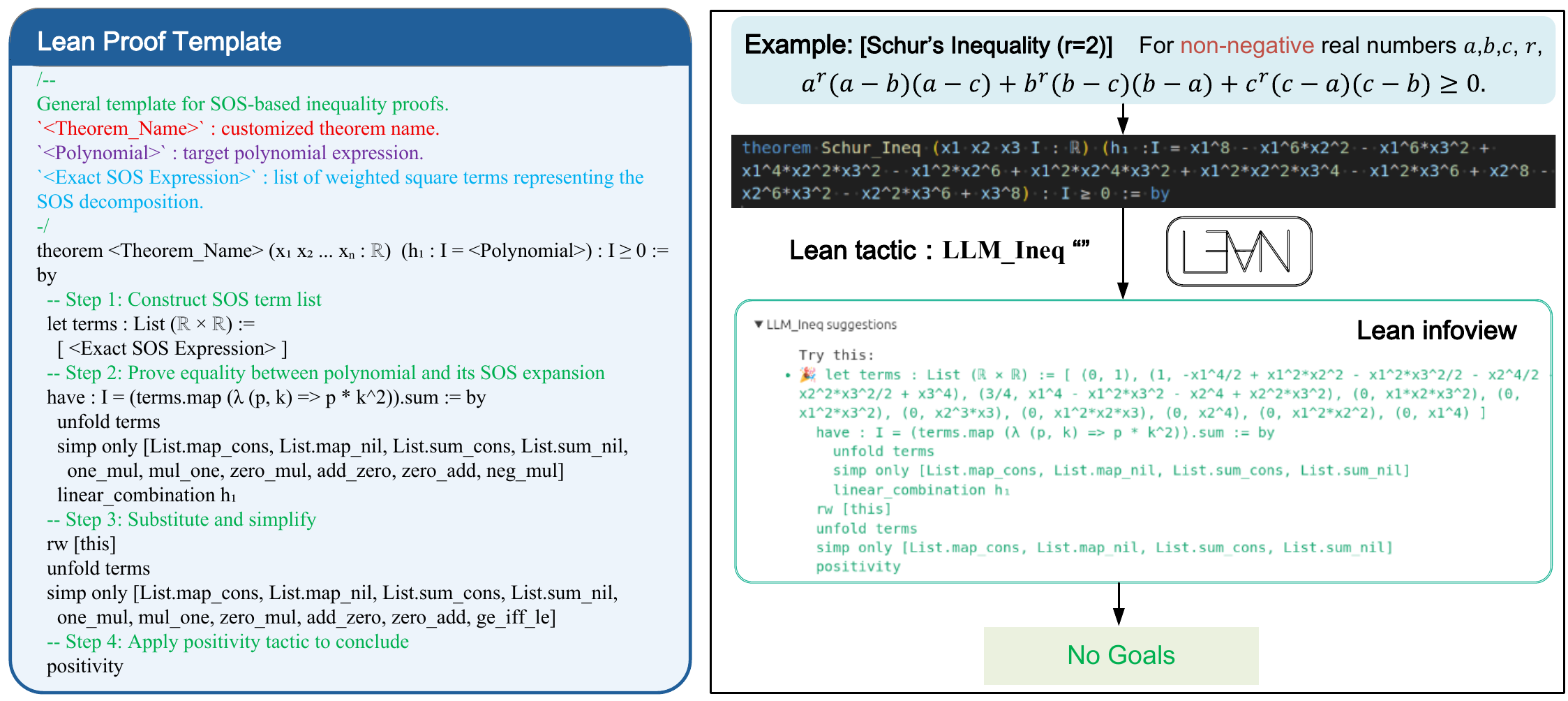}
  \caption{\textbf{Automated Lean Proof Generation Example.}   \textbf{left:} The pre-defined Lean Proof Template. 
  % for Sum-of-Squares (SOS)
  % SOS based verification. 
  \textbf{right:} An Example of automated proof generation within Lean4,  
  % (Schur's Inequality with $r=2$), 
  where the formal verification is completed using the integrated tactic \textit{``LLM\_Ineq"}.
  }  \label{fig:lean_proof_generation}
\end{figure*}  % TODO 这个图有点糙

\subsection{Lean Template of Proof}

Here presents the Lean4 proof template referenced in section ~\cref{sec:Formal Verification}, 
% presents the Lean proof template discussed in ~\cref{sec:Formal Verification}, 
which yields a verifiable formal proof based on the exact sum-of-squares (SOS) representation of a non-negative polynomial. 
% which provides a complete formalized proof in Lean4 based on the exact SOS representation of non-negative polynomials.

\begin{figure*}[htbp]
\begin{tcolorbox}[colback=gray!5!white, colframe=blue!50!black, arc=3mm, title={\texttt{Lean Proof Template}}]
% \begin{minted}[fontsize=\small, breaklines, mathescape=true]{lean}
\begin{lstlisting}[language=lean4, frame=none, mathescape=true, aboveskip=0pt, belowskip=0pt]
/--
General template for SOS-based inequality proofs.
`<Theorem_Name>` : customized theorem name.
`<Polynomial>` : target polynomial expression.
`<Exact SOS Expression>` : list of weighted square terms representing the SOS decomposition.
-/
theorem <Theorem_Name> (x1 x2 ... xn : Real)
  (h1 : I = <Polynomial>) :
  I >= 0 := by
  -- Step 1: Construct SOS term list
  let terms : List (Real $\times$ Real) :=
    [ <Exact SOS Expression> ]
  -- Step 2: Prove equality between polynomial and its SOS expansion
  have : I = (terms.map (fun (p, k) => p * k^2)).sum := by
    unfold terms
    simp only [List.map_cons, List.map_nil, List.sum_cons, List.sum_nil,
      one_mul, mul_one, zero_mul, add_zero, zero_add, neg_mul]
    linear_combination h1
  -- Step 3: Substitute and simplify
  rw [this]
  unfold terms
  simp only [List.map_cons, List.map_nil, List.sum_cons, List.sum_nil,
    one_mul, mul_one, zero_mul, add_zero, zero_add, ge_iff_le]
  -- Step 4: Apply positivity tactic to conclude
  positivity
\end{lstlisting}%\end{minted}
\end{tcolorbox}
\caption{Lean template for inequality proof based on SOS representation.}
\end{figure*}

The template uses the following placeholders: 
\begin{itemize}
% \textit{\text{\textless SOS Expression\textgreater}}
    \item \text{\textless Theorem\_Name\textgreater}: the customized theorem name;
    
    \item \text{\textless Polynomial\textgreater}: the target polynomial expression (often an expanded or normalized form equivalent to the original);

    \item \text{\textless Exact SOS Expression\textgreater}: a list of exact square terms constituting the SOS certificate, encoded as a Lean list of pairs (e.g., $(p,k)$) representing a squared polynomial $p$ together with its coefficient $k$, according to the chosen encoding.
\end{itemize}

%% TODO 待检查
The proof proceeds in four structured steps:
\begin{itemize}
    \item \textbf{Step 1: Construct the SOS term list.} The SOS representation is introduced as terms : 
    $\text{List}(\mathbb{R} \times \mathbb{R})=[...]$ 
    containing all exact squared terms required for the proof. 
    \item \textbf{Step 2: Prove the polynomial identity.} 
    A key lemma establishes that the target polynomial equals the sum of the expanded SOS terms. This step typically relies on simp to unfold list combinators \text{List.map\_cons, List.sum\_cons, etc.} and on \text{linear\_combination} to derive the exact identity from the provided equality hypothesis (e.g., \text{h1 : I = \textless Polynomial\textgreater}) together with auxiliary algebraic equalities. 
    \item \textbf{Step 3: Substitute and normalize the goal.} By rewriting with the established identity (rw [this]), the original goal is transformed into nonnegativity of a sum-of-squares, followed by additional simplifications to reach a form amenable to automation. 
    \item \textbf{Step 4: Conclude via positivity automation.} Finally, the positivity tactic is invoked to prove that a sum-of-squares is nonnegative, completing the formal verification of $I(\mathbf{x}) \ge 0$.
\end{itemize}

\subsection{An Example of Automated Lean Proof Generation}

Fig.~\ref{fig:lean_proof_generation} illustrates the 
% template structure 
structure of the proof template %utilized in this integration 
and an example of automated proof generation.

\section{Algorithm}
% 算法伪代码
 
Here provides the overall pseudocode for the proposed NSPI method (\cref{alg:overall}), along with the detailed pseudocodes for the progressive two-stage training process of the neural conjecture module (\cref{alg:nspi-train}), the Newton iteration (\cref{alg:newton}), and the rational recovery process (\cref{alg:rational}) within the symbolic correction module.

\begin{algorithm}[H]
\caption{Overall Proof Generation Process of NSPI}
\label{alg:overall}
\renewcommand{\algorithmicrequire}{\textbf{Input:}}
\renewcommand{\algorithmicensure}{\textbf{Output:}}
\renewcommand{\algorithmiccomment}[1]{\hfill $\triangleright$ #1}
\begin{algorithmic}[1]
\REQUIRE Target polynomial $f(x)$; SOS structure conjecturer $M$; candidate budget $K$. 
%Newton tolerance $\tau$; max Newton steps $T_N$; rank threshold $\epsilon$; rational precision $\rho$.
\ENSURE A verified Lean proof script $\mathcal{P}_{formal}$, or \textsc{Failure}.

\STATE \textbf{Stage 1: Neural Conjecture}
\STATE $\mathcal{S}_{approx} \leftarrow M(f(x), K)$ \COMMENT{Generate $K$ SOS-structure conjectures}
\STATE Rank $\mathcal{S}_{approx}$ by $\theta(s)=\|\hat f_s(x)-f(x)\|_2$ %\COMMENT{$\ell_2$ norm of the polynomial difference}

\STATE \textbf{Stage 2: Symbolic Correction}
\FOR{each structure $s$ in $\mathcal{S}_{approx}$  }
    \STATE $G_N \leftarrow \text{NewtonIteration}(f(x), s)$  \COMMENT{Gauss--Newton refinement to obtain $G_N$ with small backward error}
    \STATE $S_{rat} \leftarrow \text{RationalRecovery}(G_N)$ \COMMENT{Interior/boundary rational recovery}
    \IF{$\textsc{IsExactSOS}(f(x), S_{rat})$}
        \STATE \textbf{break}  \COMMENT{An exact rational SOS certificate is found}
    \ENDIF
\ENDFOR

\STATE \textbf{Stage 3: Formal Verification}
\IF{$S_{rat}$ is found}
    \STATE $\mathcal{P}_{formal} \leftarrow \textsc{TemplateFill}(f(x), S_{rat})$ \COMMENT{Generate Lean script from templates}
    \IF{$\textsc{LeanCheck}(\mathcal{P}_{formal})$}
        \STATE \textbf{return} $\mathcal{P}_{formal}$
    \ENDIF
\ENDIF
\STATE \textbf{return} \textsc{Failure}.
\end{algorithmic}
\end{algorithm}

\begin{algorithm}[htbp]%[htbp]
\caption{Progressive Two-Stage Training of the SOS Structure Conjecturer}
\label{alg:nspi-train}
\renewcommand{\algorithmicrequire}{\textbf{Input:}}
\renewcommand{\algorithmicensure}{\textbf{Output:}}
\renewcommand{\algorithmiccomment}[1]{\hfill $\triangleright$ #1}
\begin{algorithmic}[1]
\REQUIRE Base language model $M_0$; data construction methods $\mathcal{D}=\{\mathcal{D}_1,\dots,\mathcal{D}_4\}$; curriculum schedule $\mathcal{C}$; GRPO hyperparameters.
\ENSURE Trained SOS structure conjecturer $M$.

\STATE \textbf{Stage 1: Data Construction and Supervised Fine-Tuning}
\STATE Construct synthetic polynomial--SOS pairs $\mathcal{S} \leftarrow \bigcup_{j=1}^{4}\mathcal{D}_j$ \COMMENT{Four construction methods}
\STATE Train $M_0$ on $\mathcal{S}$ via teacher forcing to obtain $M_{\mathrm{SFT}}$ \COMMENT{Supervised fine-tuning}

\STATE \textbf{Stage 2: Curriculum Reinforcement Learning}
\STATE Evaluate $M_{\mathrm{SFT}}$ on $\mathcal{S}$ and collect unsolved samples $\mathcal{U}$
\STATE Partition $\mathcal{U}$ into curriculum buckets $\{\mathcal{U}_1,\dots,\mathcal{U}_L\}$ by $\mathcal{C}$ \COMMENT{Increasing difficulty}
\FOR{$\ell = 1$ to $L$}
    \FOR{each mini-batch $\mathcal{B} \subset \mathcal{U}_\ell$}
        \STATE Sample SOS conjectures $\{S_i\}$ from $M$ for each $(f,S^\star) \in \mathcal{B}$ \COMMENT{Policy rollout}
        \STATE Compute reward $R(S_i) \leftarrow R_{\text{acc}}(S_i) + R_{\text{fmt}}(S_i) - P_{\text{alg}}(S_i)$ \COMMENT{Accuracy, format, and algebraic penalty}
        \STATE Update $M$ via GRPO to maximize $\mathbb{E}[R]$ \COMMENT{Reinforcement learning}
    \ENDFOR
\ENDFOR

\STATE \textbf{return} $M$ \COMMENT{Trained SOS structure conjecturer}
\end{algorithmic}
\end{algorithm}

%%% TODO 待检查
\begin{algorithm}[htbp]
\caption{\textsc{NewtonIteration}: Gauss--Newton Refinement for SOS Certificates}
\label{alg:newton}
\renewcommand{\algorithmicrequire}{\textbf{Input:}}
\renewcommand{\algorithmicensure}{\textbf{Output:}}
\renewcommand{\algorithmiccomment}[1]{\hfill $\triangleright$ #1}
\begin{algorithmic}[1]
\REQUIRE Polynomial $f(x)$; SOS structure $s$; tolerance $\tau$; max steps $T_N$.
\ENSURE Refined numerical Gram matrix $G_N$ (or $\emptyset$).

\STATE Build monomial basis ${\mathbf v}(x)$ implied by $s$ and initialize floating Gram matrix $G^{(0)}$.
\FOR{$t=0$ \textbf{to} $T_N-1$}
    \STATE Compute backward error $\theta^{(t)}=\|f(x)-{\mathbf v}(x)^\top G^{(t)} {\mathbf v}(x)\|_2$.
    \IF{$\theta^{(t)}\le\tau$}
        \STATE \textbf{return} $G_N \leftarrow G^{(t)}$
    \ENDIF
    \STATE Update $G^{(t)}$ by one Gauss--Newton step (via Cholesky parameterization).
\ENDFOR
% \FOR{$t=0$ \textbf{to} $T_N-1$}
%     \STATE Compute backward error $\theta^{(t)}=\|f(x)-{\mathbf v}(x)^\top G^{(t)} {\mathbf v}(x)\|_2$.
%     \IF{$\theta^{(t)}\le\tau$}
%         \STATE \textbf{return} $G_N \leftarrow G^{(t)}$
%     \ENDIF
%     \STATE Compute a Cholesky factorization $G^{(t)} \approx L^{(t)}(L^{(t)})^\top$.
%     \STATE Form residual coefficients $r^{(t)} \leftarrow c\!\left(f(x)-{\mathbf v}(x)^\top L^{(t)}(L^{(t)})^\top {\mathbf v}(x)\right)$.
%     \STATE Build the Jacobian $J^{(t)} \leftarrow \partial r^{(t)}/\partial \mathrm{vec}(L^{(t)})$ by linearization.
%     \STATE Solve $\Delta^{(t)} \leftarrow \arg\min_{\Delta}\|J^{(t)}\Delta - r^{(t)}\|_2$.
%     \STATE Update $L^{(t+1)} \leftarrow L^{(t)} + \mathrm{reshape}(\Delta^{(t)})$ and set $G^{(t+1)} \leftarrow L^{(t+1)}(L^{(t+1)})^\top$.
% \ENDFOR

\STATE \textbf{return} $\emptyset$.
\end{algorithmic}
\end{algorithm}

\begin{algorithm}[htbp]
\caption{\textsc{RationalRecovery}: Exact Rational SOS Certificate Recovery}
\label{alg:rational}
\renewcommand{\algorithmicrequire}{\textbf{Input:}}
\renewcommand{\algorithmicensure}{\textbf{Output:}}
\renewcommand{\algorithmiccomment}[1]{\hfill $\triangleright$ #1}
\begin{algorithmic}[1]
\REQUIRE Polynomial $f(x)$; SOS structure $s$; refined Gram matrix $G_N$; rank threshold $\epsilon$; precision $\rho$.
\ENSURE Exact rational SOS certificate $S_{rat}$ (or $\emptyset$).

\STATE Form affine constraints from coefficient matching: $f(x)={\mathbf v}(x)^\top G {\mathbf v}(x)$.
\IF{$\textsc{NumRank}(G_N,\epsilon)$ is full}
    \STATE Project $G_N$ onto the affine constraint and rationalize entries with precision $\rho$.
\ELSE
    \STATE Perform truncated $LDL^\top$ and LLL-based rational vector recovery to preserve rank.
\ENDIF
\STATE Construct $S_{rat}$ from the recovered rational Gram matrix and \textbf{return} $S_{rat}$.
\end{algorithmic}
\end{algorithm}

\section{Experimental Details}
\label{appendix:more_details_of_exp}

% \subsection{Paremeters Setting} %%%可以和实验附录放一起
\subsection{Implementation Details} %++ 也可以放附录
% 数据构造细节
% 模型选择
% 重要参数设置
% TODO 检查具体模型型号

During the data construction phase of the neural conjecture module, over one million training instances were generated based on the proposed construction methods. Qwen3-8B \cite{qwen3technicalreport} was employed as the base model for SOS structure conjecturing. In the reinforcement learning phase, progressive multi-round training was conducted using curriculum data organized by the increasing number of variables: Stage 1 (3–5 variables), Stage 2 (6–8 variables), and Stage 3 (9–10 variables). The reward function parameters were configured with $\alpha = 0.5$, $\lambda = 0.5$, and $C_{hard} = 0.5$, with the accuracy weight $w_{acc}$ set to $0.9$. Within the symbolic correction module, the tolerance threshold $\tau$ for Newton iteration was predefined as \text{1e-15}. During the inference phase, the computational budget $k$ is set to 32, with a maximum time limit of 1 hour.

% In the data construction phase of the neural conjecture module, we generate more than 100w training data based on the four proposed data construction methods. Qwen3-8B \cite{qwen3technicalreport}
% is used as the base model for SOS structure conjecturing. During the reinforcement learning phase, progressive multi-round training is conducted using curriculum data, organized by increasing number of variables. %
% The $\alpha$ parameter of the reward function is set to **, the $\lambda$ parameter is set to **, $C_{hard}$ is set to **, and $w_{acc} = 0.9$. In the symbolic correction module, the tolerance threshold $\tau$ for the Newton iteration is predefined as **. 
% % More implementation details are provided in Appendix \ref{appendix:neural_conjecture_configuration}.
% % 候选猜想数量k 设置为 --

\subsection{Baseline Setups} 
To ensure a fair comparison, we evaluate our method against several categories of approaches, including symbol-based methods, LLM-based provers, general-purpose LLMs, and hybrid systems. Detailed configurations and specifications for each category are provided below. 
% Further details on the baseline approaches are presented here. 
\begin{itemize}
    \item \textbf{Symbol-based Method.} We evaluate the symbolic computation tools Maple~\cite{heck1993maple} and Z3~\cite{de2008z3}. 
    % For Maple, we ****; for Z3, we ****.
    For Maple, 
    % we utilize \textbf{RegularChains} to transform the non-negativity verification into a feasibility analysis of real algebraic systems; specifically, 
    we perform symbolic decomposition and sampling of the semi-algebraic set via \texttt{SamplePoints} to check for the absence of negative values. 
    For Z3, we employ its \textbf{SMT} solver to formulate the problem as a satisfiability task under real arithmetic. 
    % ; by asserting the constraint that the polynomial is strictly negative, the prover verifies non-negativity if it returns \texttt{unsat}. 

    \item \textbf{LLM-based Prover.} We consider state-of-the-art (SOTA) LLM-driven Lean automated theorem provers, including DeepSeek-Prover-V2-7B~\cite{ren2025deepseek}, Goedel-Prover-V2-8B~\cite{lin2025goedelproverv2}, and Kimina Prover~\cite{wang2025kimina}. Each prover is tested using its respective default prompt templates. In our main experiments, we compare their performance under a computational budget of pass@32. 
%%% GPT-5.2(reasoning=low) 不讲
    \item \textbf{General-purpose LLM.} 
    This group encompasses the latest closed-source foundation models, 
    % We evaluate recent proprietary general-purpose models, 
    including GPT-5.2, Gemini-3-Pro-Preview, and DeepSeek-V3.2~\cite{deepseekai2025deepseekv32}. We assess their ability to generate Lean proofs directly from Lean theorem statements. 

    \item \textbf{Hybrid system.} We compare against the recent hybrid inequality-proving system LIPS~\cite{li2025LIPS}, using its default experimental configuration. 
    It is worth noting that a comparison with several hybrid systems~\cite{wei2024AIPS,li2025ineqsearch} was not feasible, as their source code is not publicly available for replication. Furthermore, these methods primarily focus on addressing problems with 3 or 4 variables, as evidenced by the datasets described in their respective papers.
    % Note that some hybrid systems~\cite{wei2024AIPS} cannot be included due to the lack of publicly available implementations. 
    % It is worth noting that certain other hybrid systems were excluded from our comparison due to the inaccessibility of their source code.
\end{itemize}

% TODO 关于训练数据的质量 类似于---
% 为了评估数据的质量，我们选择了超过100个不同维度的例子，并邀请领域内的专家对例子的质量进行人工评估。评估表明，不同于直接构造SOS数据的朴素方法，所构造的SOS数据与真实例子具有高度相似性，并包含不同难度的例子，附录中给出了所有训练数据的更多细节（统计结果）。
% 正文

\subsection{Details of Benchmark}

%  TODO 加一个不同对应 
\label{appendix:details_of_benchmark}
% 关于训练集可能要强调一下 所有合成数据能够确保正确性 以及数据分布均匀，包含不同难度的多项式 (比如不同变元数量分布均匀   不同SOS项数 ---

% TODO:
% 来源
% 权威性强调 +  难度分析
% 数据示例（真实   合成）
% no % 数据转化方法 （针对真实例子） 并分析
% no %平方换元等操作提高了次数 是否能说不会导致题目难度下降？

Here we provide additional details on \textbf{PolyIneqBench}, the %competition-level 
challenging inequality proving benchmark constructed in this work. 

Specifically, PolyIneqBench consists of a total of \textbf{522} inequality problems, 
which primarily categorized into two subsets based on their origin and complexity: 
% \begin{itemize}
%     \item \textbf{Classical Inequalities}. These include 7 well-known benchmark families, including the Schur inequality, 
%     % Motzkin polynomial, Choi–Lam polynomial, 
%     Robinson polynomial, Delzell polynomial, Peyri–Parrilo polynomial, 
%     % Leep–Starr polynomial, 
%     Lax polynomial, Nesbitt inequality, and Voronoi polynomial. 

%     \item \textbf{Competition-Level Examples}. 
%     This category comprises 95 real-world problems drawn from authoritative sources and 420 synthesized problems designed for the multivariate setting.
% \end{itemize}

\begin{itemize}
\item \textbf{PolyIneq-Real}: This subset contains 102 inequality problems sourced from real-world domains. It comprises two main components:
\begin{itemize}
\item Classical benchmark inequalities, including the Schur inequality, Robinson polynomial, Delzell polynomial, Peyri–Parrilo polynomial, Lax polynomial, Nesbitt inequality, and Voronoi polynomial.
\item Problems derived from national and international mathematics competitions, as well as standard competition textbooks.
\end{itemize}
Representative examples of PolyIneq-Real are presented in Table~\ref{tab:polynomial_dataset_polyineq_real}.
\item \textbf{PolyIneq-Synth}: To extend the benchmark to higher-dimensional settings, we constructed this synthetic subset containing 420 inequalities spanning four to ten variables (with 60 problems per variable count). Each generated problem was validated by domain experts through manual quality assessment on randomly sampled instances (see \cref{appendix:synthetic_data_statistics}), ensuring non-triviality and challenge. Representative examples are provided in Table~\ref{tab:part_of_polyineq_synth_benchmark}.  
\end{itemize}

% % The \textbf{classical inequality} subset contains ten well-known benchmark families, including the Schur inequality, Motzkin polynomial, Choi–Lam polynomial, Robinson polynomial, Delzell polynomial, Peyri–Parrilo polynomial, Leep–Starr polynomial, Lax polynomial, Nesbitt inequality, and Voronoi polynomial. 
% The \textbf{competition-level} subset includes 110 real-world problems drawn from authoritative sources, such as national and international mathematics competitions and standard competition textbooks. Together with a small number of additional canonical instances, these problems form the \textbf{PolyIneq-Real} dataset, which contains 110 problems in total. ~\cref{tab:polynomial_dataset_polyineq_real} presents representative examples from PolyIneq-Real. 

% The real-world instances come from sources such as national and international mathematics competitions and standard competition textbooks. Together with a set of classical inequalities, these authoritative problems constitute the \textbf{PolyIneq-Real} dataset, totaling 102 problems. 
% % The real-world instances are sourced from national/international mathematics competitions and standard textbooks. Combined with the classical inequalities, these constitute the \textbf{PolyIneq-Real} dataset (102 problems). 
% Representative examples are presented in Table~\ref{tab:polynomial_dataset_polyineq_real}.
% % \cref{tab:polynomial_dataset_polyineq_real} presents representative examples from PolyIneq-Real. 

\textbf{Expanding to High-Dimensional Settings:} 
As shown in Table~\ref{tab:polynomial_dataset_polyineq_real}, most problems in PolyIneq-Real involve only \textbf{three} or \textbf{four} variables. PolyIneq-Synth was specifically designed to address this limitation by introducing systematically constructed problems in higher-dimensional settings.

% It is worth noting that, as shown in ~\cref{tab:polynomial_dataset_polyineq_real}, the vast majority of these authoritative and competition-derived problems involve only \textbf{three} or \textbf{four} variables. 
% To extend the benchmark to more challenging settings with \textbf{higher dimensionality}, 
% we constructed \textbf{PolyIneq-Synth}, containing 420 synthetic inequalities spanning four to ten variables (60 problems per variable count). 
% To ensure the rigor of the synthetic portion, domain experts conducted manual quality assessments on randomly sampled instances (\cref{appendix:synthetic_data_statistics}), confirming that the generated problems represent non-trivial and challenging inequalities. 
% Representative instances of PolyIneq-Synth are provided in Table~\ref{tab:part_of_polyineq_synth_benchmark}. 

% we further construct 420 difficult synthetic competition-level inequalities, covering inequalities with \textbf{four} to \textbf{ten} variables, with 60 problems for each variable count. For the synthetic portion, as described in ~\cref{appendix:synthetic_data_statistics}, we invited domain experts to conduct manual quality assessments on randomly sampled instances to ensure that the generated problems indeed correspond to nontrivial and challenging inequalities. These expert-validated instances constitute the \textbf{PolyIneq-Synth} dataset. ~\cref{tab:part_of_polyineq_synth_benchmark} provides representative examples from PolyIneq-Synth. 

\cref{tab:polyineqbench_distribution} presents the distribution of polynomial inequalities within the PolyIneqBench, categorized by the number of variables ($n$) and their respective sample sizes.

\begin{table}[htbp]
    \centering
    \caption{Distribution of Polynomial Inequalities by Variable Count in PolyIneqBench.}
    
    \begin{tabular}{c c}
    \toprule
         \textbf{Variable Count (n)}&\textbf{Sample Size}  \\
         \midrule
         3 & 84\\
         \midrule
         4 & 77\\
         \midrule
         5 & 60\\
         \midrule
         6 & 61\\
         \midrule
         7 & 60\\
         \midrule
         8 & 60\\
         \midrule
         9 & 60\\
         \midrule
         10 & 60\\
    \bottomrule
    \end{tabular}
    \label{tab:polyineqbench_distribution}
\end{table}

\begin{table}[H]%[htbp]
\centering
% \caption{Polynomial Dataset from Various Sources}
\caption{Part of PolyIneq-Real Benchmark}
\label{tab:polynomial_dataset_polyineq_real}
\begin{tabular}{l c c c p{7.5cm}}
% \begin{tabular}{@{}l>{\raggedright\arraybackslash}p{2.3cm}lcl@{}}
\toprule
\textbf{Source} & \textbf{\#Variables (n)} & \textbf{\#Degree} &\textbf{\#Terms} & \multicolumn{1}{l}{\textbf{\( f(x) \)}}  \\
\midrule
\cite{chen2014brief} %(Example 1.2) 
&3 & 2 & 6 & $x_1^2 - x_1x_2 - x_1x_3 + x_2^2 - x_2x_3 + x_3^2$ \\
\midrule
\cite{chen2014brief} %(Exercise 1.3) 
& 3& 6 &6 & $x_1^6 - x_1^4x_2^2 - x_1^2x_3^4 + x_2^6 - x_2^4x_3^2 + x_3^6$\\
\midrule
\cite{nguyen2012nice} %(Q 270) 
& 3& 8 &7 & $4x_1^8 - 4x_1^2x_2^2 - 4x_1^2x_3^2 + 4x_2^8 - 4x_2^2x_3^2 + 4x_3^8 + 3$\\
\midrule
\cite{nguyen2012nice} %(Q 453)
& 3& 4 & 6 & $-2x_1^4 + 2x_1^2x_2^2 + 2x_1^2x_3^2 + x_2^4 - 4x_2^2x_3^2 + x_3^4$  \\
\midrule
\cite{riasat2008basics} %(Example 2.1.1)
& 3& 2 & 6& $2x_1^2 - 2x_1x_2 - 2x_1x_3 + 2x_2^2 - 2x_2x_3 + 2x_3^2$  \\
\midrule
% \cite{riasat2008basics} %(Example 3.2.1)
% & & 4 & & $3(x_1^4 + x_2^4 + x_3^4) - (x_1^2 + x_2^2 + x_3^2)^2$  \\
% \midrule
% IMO 1964 %Q2
% & & 6 & & $x_1^6 + x_2^6 + x_3^6 + 3x_1^2x_2^2x_3^2 - x_1^2x_2^2(x_1^2 + x_2^2) -  x_2^2x_3^2(x_2^2 + x_3^2) - x_3^2x_1^2(x_3^2 + x_1^2)$ \\
% \midrule
% District Olympiad 2014 & & 4 &  & $x_1^2x_2^2 + x_1^2 - 12x_1x_2 - 4x_1 + x_2^2 - 6x_2 + 49$ \\
% \midrule
IMO Short List 1998& 3 & 8 & 7& $4x_1^8 + 4x_1^6 + 4x_2^8 + 4x_2^6 + 4x_3^8 + 4x_3^6 - 3$   \\
\midrule
\cite{manfrino2010inequalities} &3 & 8 & 6
% (Exercise 1.41) 
 & $x_1^4x_2^4 - x_1^4x_2^2x_3^2 + x_1^4x_3^4 - x_1^2x_2^4x_3^2 - x_1^2x_2^2x_3^4 + x_2^4x_3^4$ \\
\midrule
Canada 2002 & 3& 8 & 6 & $x_1^8 - x_1^4x_2^2x_3^2 - x_1^2x_2^4x_3^2 - x_1^2x_2^2x_3^4 + x_2^8 + x_3^8$  \\
\midrule
Spain 1996 & 3& 4 &  6& $x_1^4 - 4x_1^2x_2^2 + 2x_1^2x_3^2 + 4x_2^4 - 4x_2^2x_3^2 + x_3^4$ \\
\midrule
% Columbia 2001 & & 2&  6 & $3x_1^2 + 3x_1x_2 + 6x_1 + 3x_2^2 + 6x_2 + 4$  \\
% \midrule
Schur's Inequality ($r=2$) &3 & 8& 12 & $x_1^8 - x_1^6x_2^2 - x_1^6x_3^2 + x_1^4x_2^2x_3^2 - x_1^2x_2^6 + x_1^2x_2^4x_3^2 + x_1^2x_2^2x_3^4 - x_1^2x_3^6 + x_2^8 - x_2^6x_3^2 - x_2^2x_3^6 + x_3^8$ \\
\midrule
% Lax-Lax Polynomial & Classical Inequalities & $\begin{aligned} &x_1x_2x_3x_4 - x_1(x_2 - x_1)(x_3 - x_1)(x_4 - x_1) - \\ & x_2(x_1 - x_2)(x_3 - x_2)(x_4 - x_2) - x_3(x_1 - x_3)(x_2 - x_3)(x_4 - x_3)  - x_4(x_1 - x_4)(x_2 - x_4)(x_3 - x_4) \end{aligned}$ & 4   \\
Robinson Polynomial&3 & 6 &10  & $\begin{aligned} &x_3^6 - x_2^4x_3^2 - x_2^2x_3^4 + x_2^6 - x_1^2x_3^4  + 3x_1^2x_2^2x_3^2 - x_1^2x_2^4 - \\& x_1^4x_3^2 - x_1^4x_2^2 + x_1^6 \end{aligned}$  \\
\midrule
\cite{mildorf2005olympiad} %(Q 37) 
& 3 &4  & 9  & $x_1^4 - 3x_1^3x_2 + 2x_1^2x_2^2 + 2x_1^2x_3^2 - 3x_1x_3^3 + x_2^4 - 3x_2^3x_3 + 2x_2^2x_3^2 + x_3^4$\\
\midrule
% Problems of Vasc and Arqady
\cite{Parvardi2011ProblemsVascArqady}
% (Q56) 
&3 & 8 & 12 & $\begin{aligned} &x_1^8 + 4x_1^6x_2^2 - 4x_1^6x_3^2 + 2x_1^4x_2^4 + 2x_1^4x_3^4 - 4x_1^2x_2^6 + \\ &4x_1^2x_3^6 + x_2^8 + 4x_2^6x_3^2 + 2x_2^4x_3^4 - 4x_2^2x_3^6 + x_3^8 \end{aligned}$\\
\midrule
% The Interesting Around %%% TODO
% Technical Analysis Three Variable Inequalities 
\cite{NguyenZhou2014ThreeVariableInequalities}
% (Problem 12)  
& 3 & 6 & 10& $\begin{aligned} &4x_1^6 + 12x_1^4x_2^2 - 15x_1^4x_3^2 - 15x_1^2x_2^4 - 3x_1^2x_2^2x_3^2 + \\& 12x_1^2x_3^4 + 4x_2^6 + 12x_2^4x_3^2  - 15x_2^2x_3^4 + 4x_3^6 \end{aligned}$ \\
\midrule
\cite{lee2005topics} 
% (Exercise 9) 
& 3 &6 & 10 & $\begin{aligned} &x_1^2x_2^2x_3^2 - 2x_1^2x_2x_3 + x_1^2 - 2x_1x_2^2x_3  - 2x_1x_2x_3^2 + \\&2x_1x_2 + 2x_1x_3 + x_2^2 + 2x_2x_3 + x_3^2 \end{aligned}$\\
\midrule
% 118 Mathematical Competition Inequalities 
\cite{Andreescu2019_118Inequalities}
% (Q48)  
&4 &6  & 15& $\begin{aligned} &x_1^6 + 3x_1^4x_2^2 + 3x_1^2x_2^4 - 4x_1^2x_2^2 - 4x_1^2x_3^2 - 4x_1^2x_4^2 + x_2^6 \\&- 4x_2^2x_3^2 - 4x_2^2x_4^2 + x_3^6 + 3x_3^4x_4^2 + 3x_3^2x_4^4 \\&- 4x_3^2x_4^2 + x_4^6 + 8 \end{aligned}$\\
\midrule
\cite{el2008computing} &4 &8 & 20 & $\begin{aligned} &1 - 8x_3^2x_4^2 - 196608x_5^3x_1^2x_4^2x_3 + 1536x_5x_1x_4^4x_3^2 + \\&21504x_5^2x_1x_4^2x_3  - 4096x_5^2x_1x_3^3x_4^2 - 384x_5x_1x_4^2  + \\&1024x_5^2x_1x_3 + 16x_3^4x_4^4 - 72x_3^2x_4^4 + 1024x_3^2x_5^2 + \\&36864x_5^2x_1^2x_4^4 - 3456x_5x_1x_4^4 + 262144x_5^4x_1^2x_3^2 - \\&32768x_5^3x_1x_3^2 + 256x_3^3x_4^2x_5 - 576x_3x_5x_4^2 + 81x_4^4 \\ &+ 64x_3x_5 - 18x_4^2 \end{aligned}$ \\
\midrule
% \cite{mai2022practical} & &6 &  & $\begin{aligned} &1 + x_2^2 + x_2^4 + x_2^6 + x_1^2 - 9x_1^2x_2^2  + x_1^2x_2^4 + x_1^4 + x_1^4x_2^2 \\&+ x_1^6 \end{aligned}$\\
% \midrule
\cite{magron2023sum} & 3&4 &9 & $\begin{aligned} &x_1^4 + x_1x_2^3 + x_2^4 - 3x_1^2x_2x_3 - 4x_1x_2^2x_3 + 2x_1^2x_3^2 + x_1x_3^3 \\&+ x_2x_3^3 + x_3^4 \end{aligned}$ \\
\bottomrule
\end{tabular}
\end{table}

\begin{table}[H]%[htbp]
\centering
\caption{Part of PolyIneq-Synth Benchmark}
\label{tab:part_of_polyineq_synth_benchmark}
\begin{tabular}{c c c p{9.5cm}}
% \begin{tabular}{@{}l>{\raggedright\arraybackslash}p{2.3cm}lcl@{}}
\toprule
\textbf{\#Variables (n)}  & \textbf{\#Degree} & \textbf{\#Terms} & \multicolumn{1}{c}{\textbf{ \( f(x) \)}}  \\
\midrule
4 & 12& 7  & $6x_1^{12} - 2x_1^6x_2x_3^3 - 4x_1^6x_4^2x_3^4 + 5x_2^2x_4^6 + 4x_2^2x_3^6 - 2x_2x_4^5x_3^4 + 3x_4^4x_3^8$  \\
\midrule
% 4 & Competition-Level & $x_1^4x_2^2x_4^2 - 3x_1^2x_2^2x_3^2x_4^2 + x_1^2x_3^4x_4^2 + x_2^4x_3^2x_4^2 + x_3^8$ &  \\
5& 12 & 10 & $8x_1^{12} + 2x_1^6x_2x_3x_4 + 2x_1^6x_5^2 + 8x_2^2x_3^{10} + 2x_2^2x_3^6x_4 + 6x_2^2x_3^2x_4^2 - 2x_2x_3x_4x_5^2 + 7x_4^2x_5^6 - 2x_4x_5^5 + 8x_5^4$  \\
\midrule
6 & 12 & 13 & $5x_1^{12} + 2x_1^6x_2^6 + 2x_1^6x_3^3x_4^2 + 8x_2^{12} - 2x_2^6x_3^3x_4^2 + 2x_2^6x_3^2x_4^2x_5^2 + 2x_2^6x_4x_6^5 + 5x_3^6x_4^4 + 2x_3^5x_4^4x_5^2 + 5x_3^4x_4^4x_5^4 - 2x_3^3x_4^3x_6^5 - 2x_3^2x_4^3x_6^5x_5^2 + 7x_4^2x_6^{10}$   \\
\midrule
7 & 12 &14  &  $4x_1^{10}x_2^2 + 4x_1^5x_2^3x_3^2x_4^2 - 2x_1^5x_2x_3^3x_5 - 6x_1^5x_2x_6^3x_7^3 + 2x_1^5x_2x_6x_4^2x_7^3 + 5x_2^4x_3^4x_4^4 - 10x_2^2x_3^5x_5x_4^2 - 2x_2^2x_3^2x_6^3x_4^2x_7^3 + 4x_2^2x_3^2x_6x_4^4x_7^3 + 9x_3^6x_5^2 - 8x_3^3x_5x_6^3x_7^3 - 4x_3^3x_5x_6x_4^2x_7^3 + 8x_6^6x_7^6 + 2x_6^2x_4^4x_7^6$  \\
\midrule
8 & 16 & 23 & $9 x_{1}^{12} x_{7}^{2} + 22 x_{1}^{10} x_{3}^{2} x_{5}^{2} x_{8}^{2} + 16 x_{1}^{6} x_{3}^{2} x_{5}^{2} x_{6} x_{7} x_{8}^{2} - 44 x_{1}^{6} x_{3}^{2} x_{5} x_{8} + 4 x_{1}^{6} x_{4}^{2} x_{5}^{4} + 6 x_{1}^{5} x_{3}^{2} x_{5}^{2} x_{6} x_{8} + 8 x_{1}^{5} x_{3}^{2} x_{5} x_{7} x_{8} - 8 x_{1}^{4} x_{3} x_{4} x_{5}^{3} x_{6} x_{7} x_{8} + 11 x_{1}^{4} x_{4}^{2} x_{5}^{4} x_{8}^{2} - 16 x_{1}^{3} x_{3} x_{4} x_{5}^{3} x_{6} + 8 x_{1}^{3} x_{3} x_{4} x_{5}^{2} x_{7} + 23 x_{1}^{2} x_{3}^{2} x_{5}^{2} x_{6}^{2} x_{7}^{2} x_{8}^{2} - 32 x_{1}^{2} x_{3}^{2} x_{5} x_{6} x_{7} x_{8} + 28 x_{1}^{2} x_{3}^{2} + 32 x_{1} x_{3}^{2} x_{5}^{2} x_{6}^{2} x_{7} x_{8} - 8 x_{1} x_{3}^{2} x_{5} x_{6} x_{7}^{2} x_{8} - 16 x_{1} x_{3}^{2} x_{5} x_{6} - 8 x_{1} x_{3}^{2} x_{7} + 2 x_{2}^{4} x_{4}^{2} x_{5}^{2} x_{6}^{2} x_{7}^{2} x_{8}^{2} + 4 x_{2}^{2} x_{3} x_{4} x_{5} x_{6} x_{7}^{2} x_{8} + 21 x_{3}^{2} x_{5}^{2} x_{6}^{2} - 16 x_{3}^{2} x_{5} x_{6} x_{7} + 8 x_{3}^{2} x_{7}^{2}$\\
% 8 & 20 & 26 & $27x_1^{10}x_5^2 - 6x_1^9x_5^5x_6 + 6x_1^8x_5^8x_6^2 + 18x_1^7x_3^3x_4x_5x_7^2x_8^2 + 32x_1^7x_5 - 8x_1^6x_2x_4x_5 + 6x_1^6x_3^3x_4x_5^4x_6x_7^2x_8^2 + 24x_1^6x_3^2x_5x_7^3 + 28x_1^4x_3^6x_4^2x_7^4x_8^4 + 8x_1^4x_3^3x_4x_7^2x_8^2 + 25x_1^4 - 24x_1^3x_2x_3^3x_4^2x_7^2x_8^2 - 12x_1^3x_2x_4 + 40x_1^3x_3^5x_4x_7^5x_8^2 + 4x_1^3x_3^2x_7^3 + 28x_1^2x_2^4x_3^2x_6^{10} + 4x_1^2x_2^2x_3^6x_4x_6^3x_7^2x_8^2 + 4x_1^2x_2^2x_3^3x_6^3 + 24x_1^2x_2^2x_4^2 - 6x_1^2x_2x_3^2x_4x_7^3 + 23x_1^2x_3^4x_7^6 - 8x_1^2x_4^2x_5^5 - 6x_1x_2^3x_3^3x_4x_6^3 + 2x_1x_2^2x_3^5x_6^3x_7^3 + x_2^4x_3^6x_6^6 + 4x_4^4x_5^{10}$  \\  %%%  不在数据集
\midrule
9 & 18 &34 & $14x_1^{12}x_3^2x_9^2 - 2x_1^{12}x_3x_9 + 13x_1^{12} + 2x_1^8x_2x_3x_5x_9^2 - 2x_1^8x_2x_5x_9 - 2x_1^7x_2^2x_3^2x_4x_5x_6^2 - 4x_1^7x_2x_3^2x_5^2x_6x_7x_8x_9 + 2x_1^7x_5^2x_9 + 2x_1^6x_3^3x_6^2x_9 + 2x_1^6x_3^2x_5x_8x_9 + 2x_1^6x_3^2x_6^2 + 4x_1^6x_3x_4x_5x_8x_9^2 + 2x_1^6x_3x_5x_8 - 2x_1^6x_7x_8x_9^4 + 15x_1^4x_2^2x_5^2x_9^2 + 2x_1^3x_2^2x_3x_5^3x_6x_7x_8x_9 - 2x_1^3x_2x_5^3x_9^2 + 11x_1^2x_2^4x_3^4x_4^2x_5^2x_6^4 + 12x_1^2x_2^2x_3^2x_5^4x_6^2x_7^2x_8^2 - 4x_1^2x_2x_3x_5^2x_8x_9 - 2x_1^2x_2x_4x_5^2x_8x_9^2 + 2x_1^2x_2x_5x_7x_8x_9^5 + 13x_1^2x_5^4x_9^2 - 2x_1x_2^2x_3^3x_4x_5^2x_6^2x_8 - 2x_1x_2^2x_3^2x_4x_5x_6^2x_7x_8x_9^4 - 2x_1x_2x_3^2x_5^3x_6x_7x_8^2 + 2x_1x_3^2x_5^2x_6^2x_9 + 2x_1x_4x_5^3x_8x_9^2 + 16x_3^4x_6^4 - 2x_3^3x_5x_6^2x_8 + 13x_3^2x_5^2x_8^2 - 2x_3x_4x_5^2x_8^2x_9 + 12x_4^2x_5^2x_8^2x_9^2 + 14x_7^2x_8^2x_9^8$   \\
\midrule
% 10 & 18 & 33 & $13x_1^4x_3^2x_4^2x_5^2x_7^4x_9^2 + 13x_1^4x_5^4x_8^2 - 2x_1^3x_{10}^2x_2x_3x_4^3x_5x_7^3x_8^2x_9 - 2x_1^3x_{10}^2x_2x_4^2x_5^2x_7x_8^3 - 2x_1^3x_{10}x_2x_3^2x_4^2x_5x_7^2x_9^2 - 2x_1^3x_{10}x_2x_3x_4x_5^2x_8x_9 - 2x_1^3x_3x_4x_5x_6x_7^2x_8x_9 + 2x_1^3x_5^2x_6x_8^2 + 2x_1^2x_{10}^6x_5^2x_7^3x_8 + 13x_1^2x_{10}^4x_2^2x_4^4x_7^2x_8^4 - 4x_1^2x_{10}^3x_2^2x_3x_4^3x_7x_8^2x_9 - 4x_1^2x_{10}^3x_3x_4^2x_5^2x_6x_7^3x_9^3 + 2x_1^2x_{10}^3x_4x_5^3x_6x_7x_8x_9^2 + 10x_1^2x_{10}^2x_2^2x_3^2x_4^2x_9^2 + 2x_1^2x_{10}^2x_2x_4^2x_6x_7x_8^3 - 2x_1^2x_{10}x_2x_3x_4x_6x_8x_9 + 2x_1^2x_{10}x_2x_3x_5^4x_8^2 + 12x_1^2x_6^2x_8^2 + 2x_1x_{10}^7x_2x_3x_4x_7^3x_9 - 2x_1x_{10}^5x_2x_4^3x_5x_6x_7^2x_8^2x_9^2 + 2x_1x_{10}^4x_2x_3x_4^2x_5x_6x_7x_9^3 + 2x_1x_{10}^3x_4x_5x_6^2x_7x_8x_9^2 + 2x_1x_{10}^2x_2^2x_3x_4x_5x_6x_7x_8^2x_9 + 2x_1x_{10}x_2x_3x_5^2x_6x_8^2 - 2x_1x_{10}x_2x_5x_6^2x_7x_8^3 + 10x_{10}^{12}x_7^6 + 2x_{10}^9x_4x_5x_6x_7^4x_9^2 + 2x_{10}^7x_2x_5x_6x_7^4x_8^2 + 16x_{10}^6x_4^2x_5^2x_6^2x_7^2x_9^4 - 4x_{10}^4x_2x_3x_4x_5^3x_6x_7x_8x_9^2 + 13x_{10}^2x_2^2x_3^2x_5^4x_8^2 - 2x_{10}^2x_2^2x_3x_5^3x_6x_7x_8^3 + 12x_{10}^2x_2^2x_5^2x_6^2x_7^2x_8^4$  \\
10 & 18 & 42 & $10 x_{1}^{10} x_{4}^{6} + 2 x_{1}^{7} x_{4}^{3} x_{7}^{3} x_{8} - 2 x_{1}^{6} x_{3} x_{4}^{4} x_{5} x_{6} x_{7} - 4 x_{1}^{6} x_{3} x_{4}^{3} x_{5} x_{9} - 2 x_{1}^{5} x_{2} x_{4}^{3} x_{6}^{6} x_{7}^{2} - 2 x_{1}^{5} x_{4}^{3} x_{5}^{2} x_{8}^{3} x_{9} + 10 x_{1}^{4} x_{7}^{6} x_{8}^{2} + 2 x_{1}^{3} x_{2} x_{3} x_{6} x_{7}^{5} x_{8}^{2} x_{9} + 4 x_{1}^{3} x_{3} x_{4} x_{5} x_{6} x_{7}^{4} x_{8} + 14 x_{1}^{2} x_{2}^{2} x_{3}^{2} x_{6}^{2} x_{7}^{4} x_{8}^{2} x_{9}^{2} - 4 x_{1}^{2} x_{2} x_{3}^{2} x_{5} x_{6} x_{7}^{2} x_{8} x_{9}^{2} - 2 x_{1}^{2} x_{2} x_{3} x_{7}^{3} x_{8} - 2 x_{1}^{2} x_{2} x_{6}^{6} x_{7}^{5} x_{8} + 2 x_{1}^{2} x_{2} x_{7}^{3} x_{8} + 15 x_{1}^{2} x_{3}^{2} x_{4}^{2} x_{5}^{2} x_{6}^{2} x_{7}^{2} - 4 x_{1}^{2} x_{3}^{2} x_{4} x_{5}^{2} x_{6} x_{7} x_{9} + 16 x_{1}^{2} x_{3}^{2} x_{5}^{2} x_{9}^{2} + 2 x_{1} x_{10} x_{2}^{3} x_{3}^{2} x_{4} x_{6} x_{7}^{3} x_{8}^{2} x_{9}^{2} + 2 x_{1} x_{10} x_{2}^{2} x_{3}^{2} x_{4} x_{5} x_{7} x_{8} x_{9}^{2} - 2 x_{1} x_{2}^{2} x_{3}^{2} x_{6} x_{7}^{2} x_{8} x_{9} - 2 x_{1} x_{2}^{2} x_{3} x_{6}^{7} x_{7}^{4} x_{8} x_{9} - 2 x_{1} x_{2}^{2} x_{3} x_{6} x_{7}^{2} x_{8} x_{9} - 2 x_{1} x_{2} x_{3}^{2} x_{5} x_{9} - 4 x_{1} x_{2} x_{3} x_{4} x_{5} x_{6}^{7} x_{7}^{3} + 2 x_{1} x_{2} x_{3} x_{4} x_{5} x_{6} x_{7} - 4 x_{1} x_{2} x_{3} x_{5}^{2} x_{6} x_{7}^{2} x_{8}^{4} x_{9}^{2} - 2 x_{1} x_{2} x_{3} x_{5} x_{6}^{6} x_{7}^{2} x_{9} - 2 x_{1} x_{2} x_{3} x_{5} x_{9} + 2 x_{1} x_{3} x_{5}^{3} x_{8}^{3} x_{9}^{2} + 13 x_{10}^{2} x_{2}^{4} x_{3}^{2} x_{4}^{2} x_{7}^{2} x_{8}^{2} x_{9}^{2} + 2 x_{10} x_{2}^{3} x_{3}^{2} x_{4} x_{7} x_{8} x_{9} - 2 x_{10} x_{2}^{3} x_{3} x_{4} x_{6}^{6} x_{7}^{3} x_{8} x_{9} + 2 x_{10} x_{2}^{3} x_{3} x_{4} x_{7} x_{8} x_{9} - 2 x_{10} x_{2}^{2} x_{3} x_{4} x_{5}^{2} x_{7} x_{8}^{4} x_{9}^{2} + 10 x_{2}^{2} x_{3}^{2} - 2 x_{2}^{2} x_{3} + 14 x_{2}^{2} x_{6}^{12} x_{7}^{4} + 2 x_{2}^{2} x_{6}^{6} x_{7}^{2} + 15 x_{2}^{2} + 2 x_{2} x_{3} x_{5}^{2} x_{8}^{3} x_{9} - 4 x_{2} x_{5}^{2} x_{8}^{3} x_{9} + 12 x_{5}^{4} x_{8}^{6} x_{9}^{2}$\\
\bottomrule
\end{tabular}
\end{table}

\begin{figure}[htbp]
  \centering
  \includegraphics[width=\linewidth]{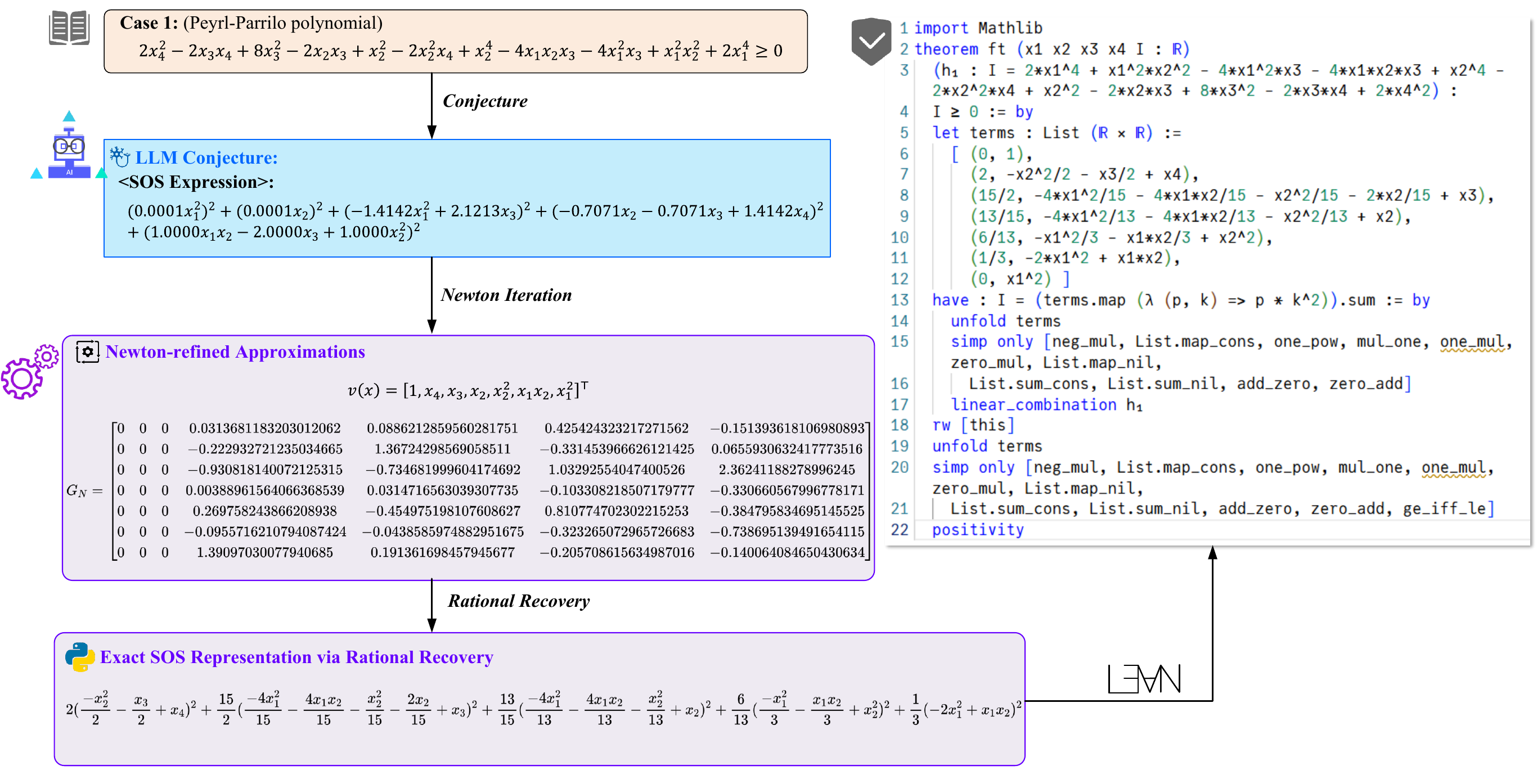}% 
  % \caption{Case Study of Peyrl-Parrilo polynomial.} 
  \caption{\textbf{Detailed workflow of the NSPI method applied to the Peyrl-Parrilo polynomial from PolyIneq-Real (Case 1)}.This example illustrates the complete pipeline from initial conjecture to formal certification.
  % :(1)LLM Conjecture:
  }
  \label{fig:case_study_1}
\end{figure}

%% 待补充caption
\begin{figure}[htbp]%[htbp]%[H]
  \centering
  \includegraphics[width=\linewidth]{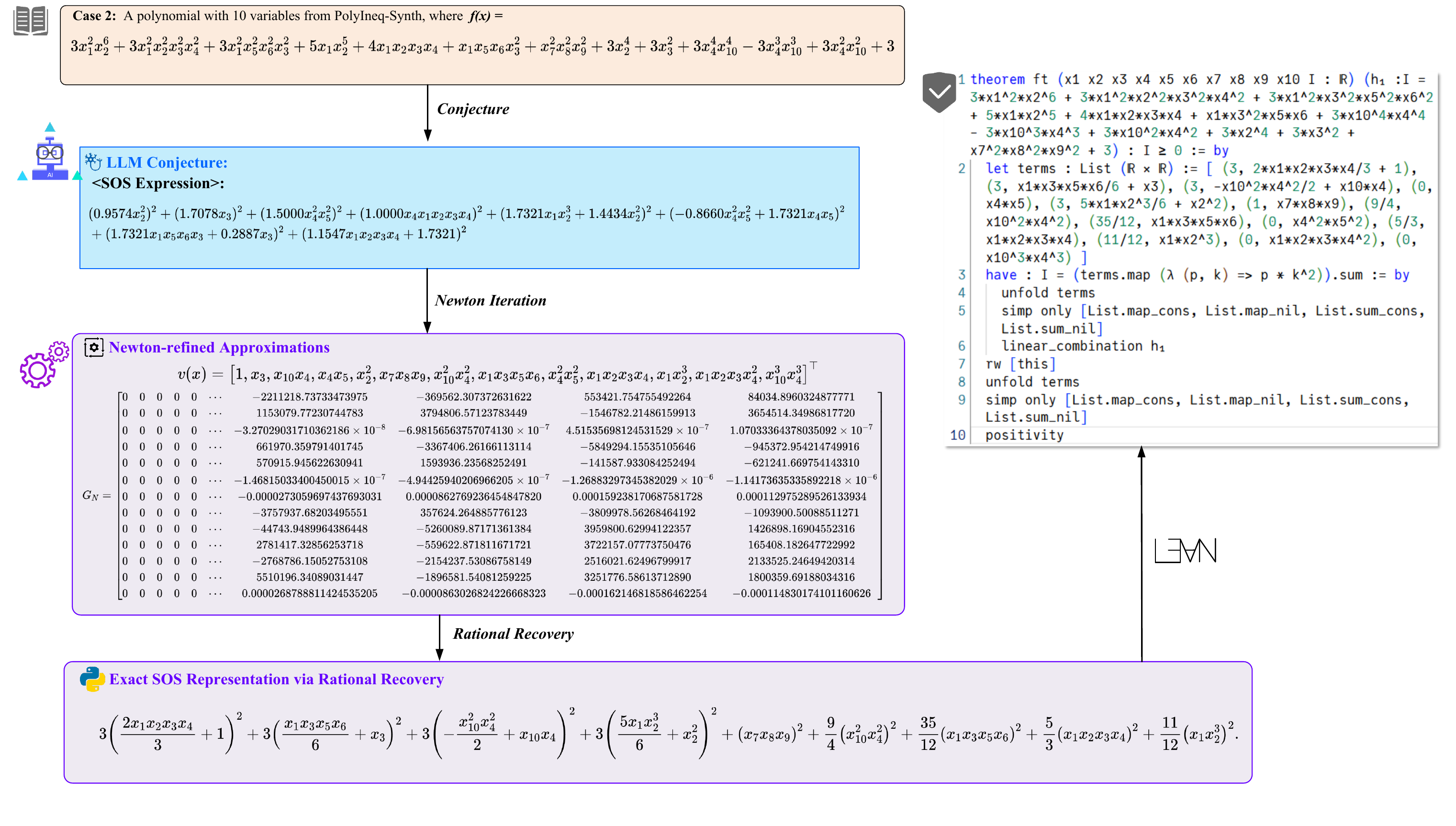}% 
  % \includegraphics[width=\linewidth]{figures/appendix/case_study_2_v1.pdf}% 
  % \caption{Case Study of a challenging synthetic polynomial.}   
  \caption{\textbf{Case study of a challenging synthetic polynomial from PolyIneq-Synth.} 
  This 10-variable instance represents a significant challenge in high-dimensional formal reasoning; notably, \textbf{NSPI is the sole method among all evaluated baselines (Case 2)} capable of successfully generating a verified formal proof. 
  % The workflow demonstrates how NSPI bridges the gap between neural conjectures and symbolic rigor: (1) an initial SOS candidate is conjectured by the LLM; (2) numerical coefficients are refined via Newton iteration; (3) exact rational recovery is performed to eliminate numerical errors; and (4) the final proof is formally certified within the Lean 4 environment using Mathlib tactics.
  }
  \label{fig:case_study_2}
\end{figure}

\section{Case Study}
% 完整过程示例详解
% 可以画一个类似IneqSearch的案例图 （数学公式）
% \subsection{Case Analysis of  Inequality}  %%% 经典例子的案例分析
Here we provide detailed case studies illustrating the complete proof process of NSPI on two representative examples drawn from PolyIneqBench. 

Case~1 is a classical inequality from the PolyIneq-Real subset, whereas Case~2 is a challenging competition-level synthetic inequality 
% with seven variables 
from the PolyIneq-Synth subset, corresponding to a higher-dimensional and more difficult setting. 
Notably, for Case~2, NSPI is the \textbf{only method} among all baselines that successfully completes the proof.  %%% TODO 待核查

\subsection{Case 1: Peyrl-Parrilo polynomial.} 
% 来自PolyIneq-Real的经典例子
Case~1 features a classic non-negative polynomial drawn from \cite{Peyrl2008ComputingSO,article2014Analgorithm}. %%%TODO
Fig.~\ref{fig:case_study_1} presents an example of generating a complete proof using the NSPI method.

\subsection{Case 2: A challenging synthetic polynomial.} 
% 来自合成数据的7元困难不等式（竞赛级别例子
                                   
%%% 段名：更加困难的例子
% 加一个合成数据的案例分析！！   （我们能解 其他方法解不了的情况
We present a challenging 10-variable synthetic polynomial instance from the PolyIneq-Synth dataset. Notably, for this specific case, NSPI is the only method among all considered baselines that successfully generates a complete formal proof.

\section{More Experimental Results} 

\subsection{Results under Different Computational Budgets} %% 不同计算budget下的结果
% 相比于不同LLM provers 不同计算budget下的结果
Here, we report the performance of various LLM-based provers and our NSPI method under different computational budgets. The results on PolyIneqBench are summarized in~\cref{tab:appendix_poly_method_comparison_diffK}.

\begin{table}[H]
\centering
\caption{Comparative results of multiple methods on PolyIneqBench under different computational budgets}
\label{tab:appendix_poly_method_comparison_diffK}
\begin{small}
\begin{tabular}{l c c c}
\toprule
\textbf{Method Name} & \textbf{Computational Budget} & \textbf{PolyIneq-Real} & \textbf{PolyIneq-Synth} \\
\midrule

\multirow{3}{*}{DeepSeek-Prover-V2}
& pass@8\;\;  & \text{28.43\%} & \text{0.00\%} \\
& pass@16 & \text{32.35\%} & \text{0.00\%} \\
& pass@32 & \text{37.25\%} & \text{0.00\%} \\
\midrule

\multirow{3}{*}{Goedel-Prover-V2}
& pass@8\;\;  & \text{15.69\%} & \text{0.48\%} \\
& pass@16 & \text{16.67\%} & \text{0.48\%} \\
& pass@32 & \text{18.63\%} & \text{0.48\%} \\
\midrule

\multirow{3}{*}{Kimina-Prover}
& pass@8\;\;  & \text{27.45\%} & \text{0.24\%} \\
& pass@16 & \text{30.39\%} & \text{0.24\%} \\
& pass@32 & \text{33.33\%} & \text{0.71\%} \\
\midrule

\multirow{3}{*}{\textbf{NSPI (ours)}}
& pass@8\;\;  & \text{36.27\%} & \text{23.57\%} \\
& pass@16 & \text{40.20\%} & \text{25.48\%} \\
& pass@32 & \text{43.14\%} & \text{25.95\%} \\

\bottomrule
\end{tabular}
\end{small}
\end{table}

% \subsection{Comparison of Average Polynomial Term Count Across Different Methods} 
\subsection{Detailed Comparative Study on Benchmark Difficulty Metrics}
% \subsection{Comparative Analysis of Different Polynomial Solving Methods}
% 变元数
% 项数
% 次数 三个方面统计

% \begin{figure}[htbp]
%   \centering
%   \includegraphics[width=0.5\linewidth]{figures/appendix/avg_terms_bar.pdf}% 
%   \caption{Average Number of Terms in Polynomials Solved by Different Methods.}   
%   \label{fig:appendix_Number_of_Terms}
% \end{figure}

% \begin{figure}
%     \centering
%     \includegraphics[width=0.5\linewidth]{figures/appendix/avg_degrees_bar.pdf}
%     \caption{Average t of Terms in Polynomials Solved by Different Methods.}
%     \label{fig:placeholder}
% \end{figure}

\begin{figure}[H]
  \centering
  \begin{subfigure}{0.47\textwidth}
    \centering
    \includegraphics[width=\linewidth]{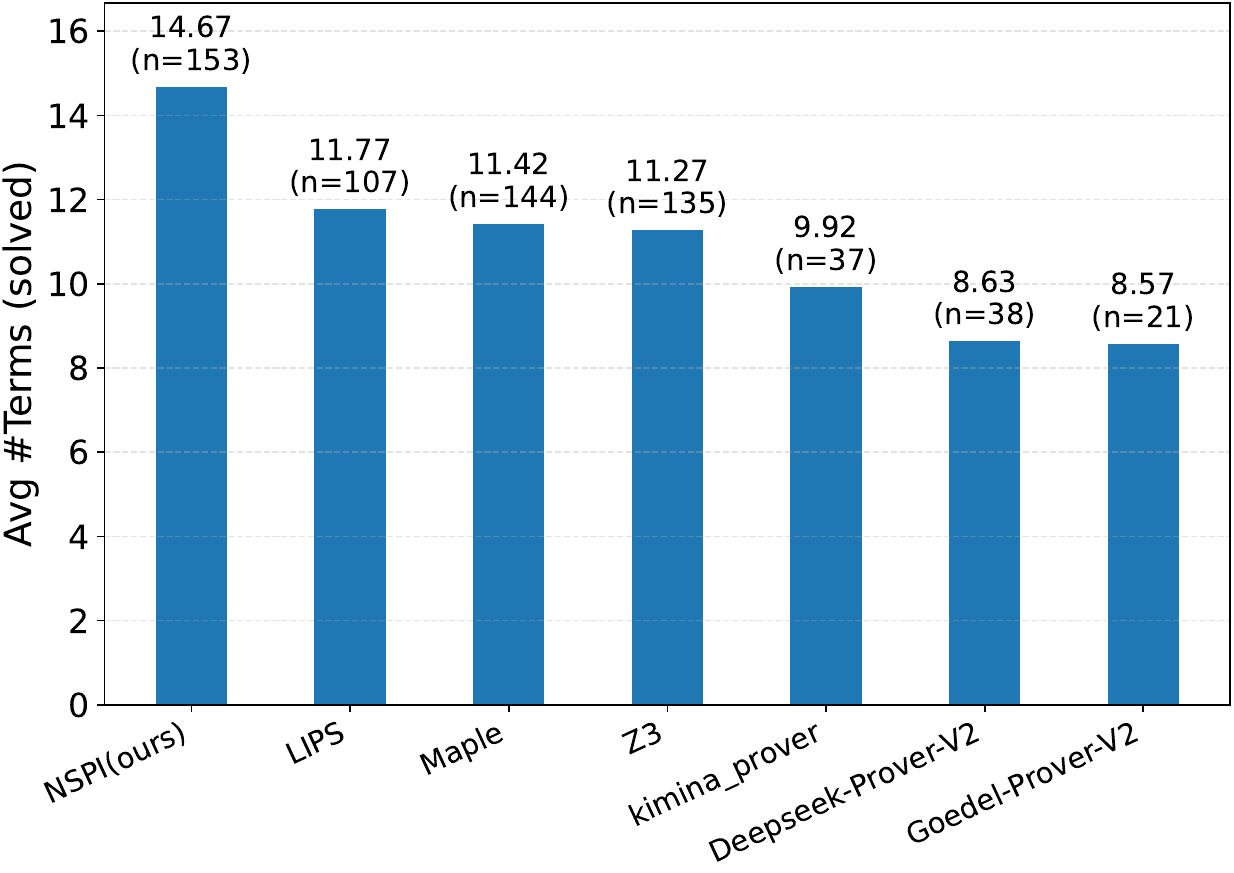}
    \caption{Average Number of Terms in Polynomials Solved by Different Methods.}
    \label{fig:appendix_Number_of_Terms}
  \end{subfigure}
  \hfill
  \begin{subfigure}{0.47\textwidth}
    \centering
    \includegraphics[width=\linewidth]{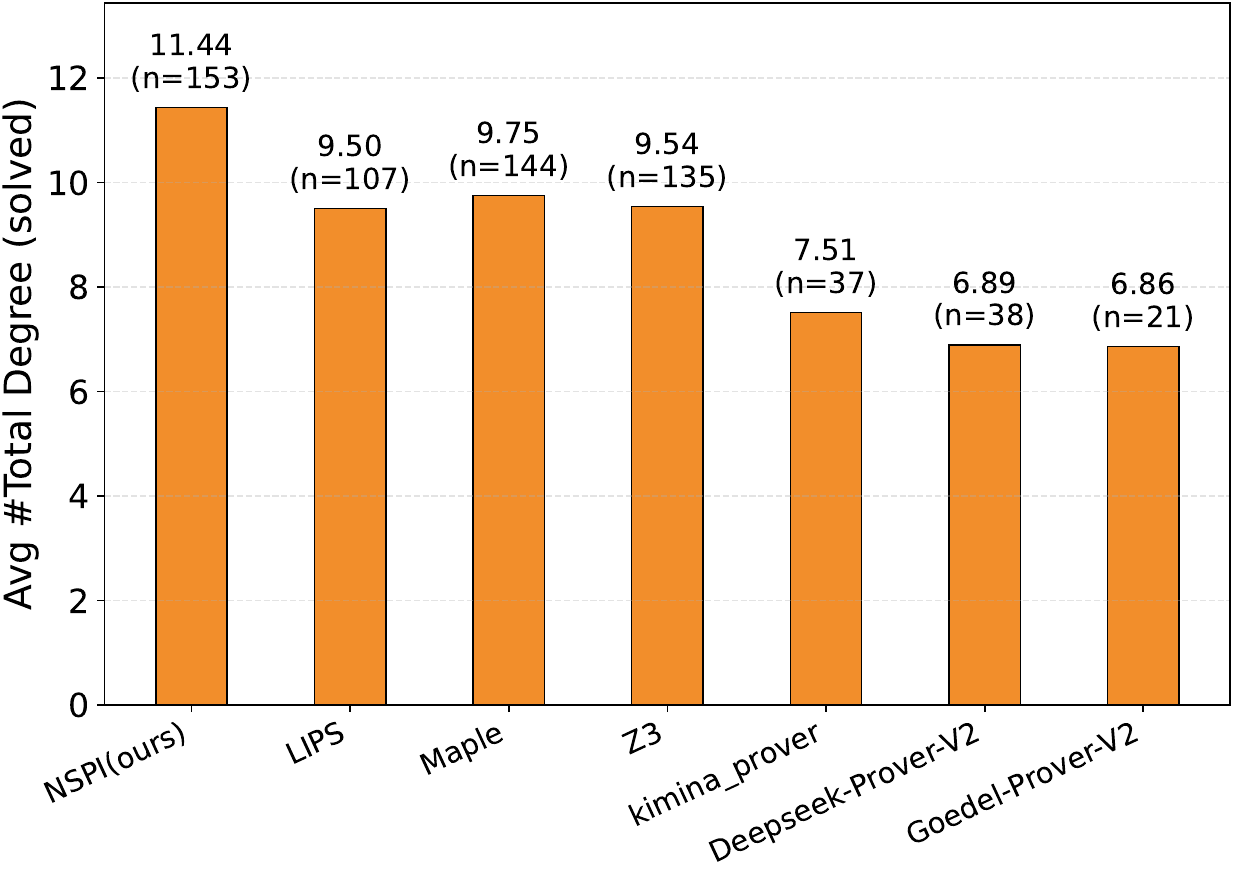}
    \caption{Average Total Degree of Terms in Polynomials Solved by Different Methods.}
    \label{fig:appendix_Total_Degree}
  \end{subfigure}
  \caption{Comparative analysis of different polynomial solving methods: (a) average number of terms; (b) average total degree.}
  \label{fig:appendix_combined_metrics}
\end{figure}

\begin{figure}[htbp]
    \centering
    \includegraphics[width=0.5\linewidth]{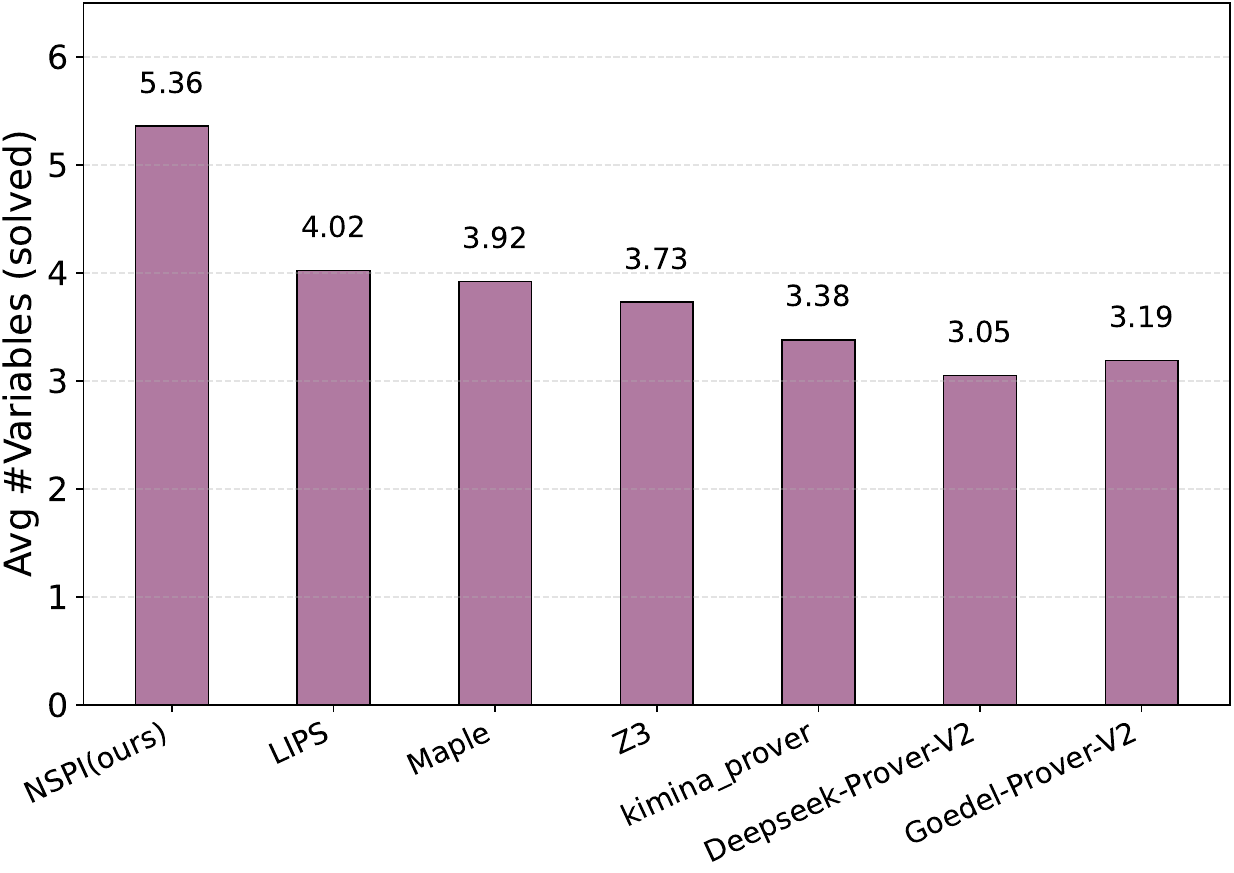}
    \caption{\textbf{Average number of variables in polynomial inequalities successfully verified by each method.} This metric illustrates the scalability of various provers, where NSPI (ours) demonstrates a superior capability in handling higher-dimensional problems compared to both symbolic baselines and LLM-based provers.}
    \label{fig:appendix_avg_vars}
\end{figure}

Fig.~\ref{fig:appendix_combined_metrics} and Fig.~\ref{fig:appendix_avg_vars} present a comparative analysis of different polynomial inequality solving methods in terms of average number of terms (Fig.~\ref{fig:appendix_combined_metrics} (a)), average total degree (Fig.~\ref{fig:appendix_combined_metrics} (b)), and average number of variables (Fig.~\ref{fig:appendix_avg_vars}) in successfully solved instances. 

Fig.~\ref{fig:appendix_combined_metrics} shows that our NSPI outperforms other methods in terms of the average number of terms and total degree in the polynomials it solves, suggesting its ability to handle more complex polynomials. 

Fig.~\ref{fig:appendix_avg_vars} highlights that NSPI also excels in handling higher-dimensional problems, as it successfully solves polynomials with the largest average number of variables, outperforming symbolic and LLM-based provers. 
These metrics together demonstrate NSPI’s superior scalability in solving complex and high-dimensional polynomial inequalities compared to both symbolic and LLM-based approaches.

\section{Extended Discussion}
% 再次强调现有方法的局限性以及NSPI的优势
% 需要讨论 现有的 根据不等式引理的证明方法 与本文方法的不同  再强调动机和贡献
% 有前景的研究方向

% \paragraph{Neuro-Symbolic Reasoning as a Promising Research Direction.} 
Neuro-symbolic reasoning systems represent a highly valuable and promising research direction. It is worth noting that the neuro-symbolic approach to inequality proving proposed in this work differs fundamentally from existing neuro-symbolic practices. In particular, we innovatively explore the feasibility of employing large language models (LLMs) as \emph{core symbolic conjecture engines} that directly provide end-to-end symbolic priors, thereby opening a promising new direction in which LLMs function as the central component for symbolic conjecturing. 
Moreover, we extend formalized theorem proving for polynomial inequalities to challenging high-dimensional cases with up to 10 variables, thereby \textit{extending the boundary of automated proof generation for polynomial inequalities}. 
To further contextualize the contributions of this work, we provide a comparative discussion with existing neuro-symbolic reasoning approaches.

\subsection{A New Paradigm: LLMs as Core Symbolic Conjecture Engines} 
In recent years, several neuro-symbolic hybrid reasoning approaches have demonstrated encouraging progress in automated theorem proving.
However, in most existing systems, neural modules primarily serve \emph{auxiliary} roles that guide or accelerate symbolic derivation processes by reducing the search space.

For example, AlphaGeometry~\cite{trinh2024alphageometry} uses an LLM to predict auxiliary constructions, while the proof is completed by a symbolic solver that enumerates derivation rules. 
AIPS~\cite{wei2024AIPS} employs a learned value network to evaluate intermediate subgoals during symbolic proof search. 
LIPS~\cite{li2025LIPS} leverages an LLM to generate candidate goal-rewriting steps and selects promising subgoals to prune the strategy space. 
Similarly, IneqSearch~\cite{li2025ineqsearch} leverages LLMs to perform inequality transformations or to select appropriate transformation rules, integrating an iterative learning mechanism.
In these systems, neural components assist decision-making within a symbolic engine, whereas the ability of neural models to directly provide symbolic reasoning content remains relatively underexplored.

In contrast, our work treats the LLM as a \emph{core symbolic conjecturing engine}. 
The LLM directly produces the symbolic priors required to complete a proof, in the form of SOS structure conjectures.
These conjectures are subsequently refined through Newton-style numerical refinement and rational recovery (Section~\ref{sec:Symbolic Correction}), and finally transformed into complete formal Lean proofs (Section~\ref{sec:Formal Verification}). 
This end-to-end pipeline demonstrates the feasibility of a new neuro-symbolic paradigm in which LLMs provide high-level symbolic structure, while symbolic computation and formal verification ensure correctness.

\subsection{Extending the Frontier of Formal Polynomial Inequality Proving}
\label{app:discussion:frontier}

% Automated formal proving of polynomial inequalities remains a challenging problem.
% Although automated theorem proving has made substantial progress, competition-level polynomial inequalities---especially those involving many variables and high-dimensional structures---are still difficult for existing systems.
Automated formal proving of polynomial inequalities remains a challenging problem. Although automated theorem-proving techniques have advanced rapidly, competition-level polynomial inequalities continue to pose significant difficulties, especially in settings involving many variables or high-dimensional structures.

Learning-based approaches to formal inequality proving often face a bottleneck in the availability of high-quality formal data.
Widely used theorem-proving benchmarks, such as miniF2F~\cite{zheng2021minif2f},  ProofNet~\cite{azerbayev2023proofnet} and Putnum~\cite{putnum} contain very few polynomial inequality problems; 
this scarcity is particularly pronounced for high-dimensional multivariate polynomials. 
% and almost none involving high-dimensional multivariate polynomials.
Existing neuro-symbolic inequality proving methods also tend to focus on low-dimensional cases with a small number of variables, and struggle to scale to more complex settings.

The neuro-symbolic approach proposed in this work extends Lean-based formal polynomial inequality proving to settings involving up to ten variables, and further introduces a competition-level benchmark covering inequalities with three to ten variables. As demonstrated by the results in Section~\ref{sec:Exp}, our method consistently outperforms all LLM-based and symbolic baselines on this 
% competition-level 
challenging 
benchmark, with particularly strong performance in cases involving a larger number of variables. These results indicate that the proposed approach significantly expands the frontier of formal polynomial inequality theorem proving. 

% The neuro-symbolic framework proposed in this work significantly extends the scope of formal polynomial inequality proving.
% By reformulating inequality proving as an SOS-based certification problem and leveraging LLM-generated symbolic structure priors, NSPI enables Lean-based formal verification for inequalities involving up to ten variables.
% As demonstrated by the experimental results in Section~\ref{sec:Exp}, our method consistently outperforms both LLM-based and purely symbolic baselines, with particularly strong gains in high-variable regimes.
% These results indicate that NSPI substantially expands the boundary of what can be achieved in formal polynomial inequality proving.

\section{More Examples} 
%%%% 多加一两个 其他方法失败 我们方法成功的例子的完整Lean代码
%%%% 来自真实例子 && 困难合成数据
%%% TODO 多放几个更困难合成数据完整证明例子。

% \subsection{Examples from PolyIneq-Real}
% \subsection{Examples from PolyIneq-Synth}

Here we provide additional examples that were successfully solved exclusively by our proposed NSPI method among all baseline approaches.

% \begin{figure}[htbp]
% \centering
% \begin{tcolorbox}[colback=gray!5!white, colframe=blue!50!black, arc=3mm, title=
% {\texttt{NSPI Achieves Solution While Baseline Methods \textbf{Fail}(1)}}
% % {\texttt{Example 1 (NSPI Output)}}
% ]
% \textbf{Polynomial:}\\
% $5x_1^6x_2^6 + 5x_3^4x_4^4x_5^4 - 2x_3^2x_4^2x_6x_7^4x_8x_5^2 - 2x_3^2x_4^2x_8^2x_5^6 + 5x_6^4x_5^4 - 2x_6^3x_7^4x_8x_5^2 + 5x_6^2x_7^8x_8^2 - 2x_6^2x_8^2x_5^6 + 6x_8^4x_5^8$
% % \\
% \medskip
% \textbf{Success on:} Gemini-3-Pro，Gpt-5.2，Deepseek-V3.2，NSPI(ours)
% \medskip
% \textbf{Lean Proof:} 
% \begin{minted}[fontsize=\small, breaklines]{lean}
% import Mathlib
% theorem ft (x1 x2 x3 x4 x5 x6 x7 x8 I : ℝ) (h1 :I = 5*x1^6*x2^6 + 5*x3^4*x4^4*x5^4 - 2*x3^2*x4^2*x5^6*x8^2 - 2*x3^2*x4^2*x5^2*x6*x7^4*x8 + 6*x5^8*x8^4 - 2*x5^6*x6^2*x8^2 + 5*x5^4*x6^4 - 2*x5^2*x6^3*x7^4*x8 + 5*x6^2*x7^8*x8^2) : I >= 0 := by
%   let terms : List (ℝ × ℝ) := [ (0, 1), (5, -x5^4*x8^2/5 + x5^2*x6^2 - x6*x7^4*x8/5), (29/5, -5*x3^2*x4^2*x5^2/29 + x5^4*x8^2 - x6*x7^4*x8/29), (139/29, -30*x3^2*x4^2*x5^2/139 + x6*x7^4*x8), (0, x4*x6*x7^4), (0, x4^2*x5^4), (640/139, x3^2*x4^2*x5^2), (5, x1^3*x2^3) ]
%   have : I = (terms.map (fun (p, k) => p * k^2)).sum := by
%     unfold terms
%     simp only [List.map_cons, List.map_nil, List.sum_cons, List.sum_nil]
%     linear_combination h1
%   rw [this]
%   unfold terms
%   simp only [List.map_cons, List.map_nil, List.sum_cons, List.sum_nil]
%   positivity
% \end{minted}
% \end{tcolorbox}
% \end{figure}

\begin{figure}[htbp]
\centering
\begin{tcolorbox}[colback=gray!5!white, colframe=blue!50!black, arc=3mm, title=
{\texttt{NSPI Achieves Solution While Baseline Methods \textbf{Fail}(1)}}
% {\texttt{Example 1 (NSPI Output)}}
]
\textbf{Polynomial:}\\
$8x_1^6 x_2^2 x_3^2 x_4^2 + 6x_1^4 x_5^2 x_6 x_2 x_3 x_4 + 8x_1^3 x_7^6 x_2 x_3 x_4 + 8x_1^3 x_6 x_2^3 x_3^2 x_4 + 6x_1^3 x_2 x_3 x_4^3 x_8^2 x_9^2 + 9x_1^2 x_5^4 x_6^2 - 6x_1 x_7^6 x_5^2 x_6 - 4x_1 x_5^2 x_6^2 x_2^2 x_3 + 10x_1 x_5^2 x_6 x_4^2 x_8^2 x_9^2 + 9x_7^{12} + 2x_7^6 x_6 x_2^2 x_3 + 4x_7^6 x_4^2 x_8^2 x_9^2 + 9x_6^2 x_2^4 x_3^2 - 8x_6 x_2^2 x_3 x_4^2 x_8^2 x_9^2 + 6x_4^4 x_8^4 x_9^4$
% $ $
% 8*x1^6*x2^2*x3^2*x4^2 + 6*x1^4*x5^2*x6*x2*x3*x4 + 8*x1^3*x7^6*x2*x3*x4 + 8*x1^3*x6*x2^3*x3^2*x4 + 6*x1^3*x2*x3*x4^3*x8^2*x9^2 + 9*x1^2*x5^4*x6^2 - 6*x1*x7^6*x5^2*x6 - 4*x1*x5^2*x6^2*x2^2*x3 + 10*x1*x5^2*x6*x4^2*x8^2*x9^2 + 9*x7^12 + 2*x7^6*x6*x2^2*x3 + 4*x7^6*x4^2*x8^2*x9^2 + 9*x6^2*x2^4*x3^2 - 8*x6*x2^2*x3*x4^2*x8^2*x9^2 + 6*x4^4*x8^4*x9^4

\medskip

\textbf{Lean Proof:} 
% \begin{minted}[fontsize=\small, breaklines]{lean}
\begin{lstlisting}[language=lean4, frame=none]
import Mathlib
theorem f1 (x1 x2 x3 x4 x5 x6 x7 x8 x9 I : Real) (h1 :I = 8*x1^6*x2^2*x3^2*x4^2 + 6*x1^4*x2*x3*x4*x5^2*x6 + 8*x1^3*x2^3*x3^2*x4*x6 + 6*x1^3*x2*x3*x4^3*x8^2*x9^2 + 8*x1^3*x2*x3*x4*x7^6 + 9*x1^2*x5^4*x6^2 - 4*x1*x2^2*x3*x5^2*x6^2 + 10*x1*x4^2*x5^2*x6*x8^2*x9^2 - 6*x1*x5^2*x6*x7^6 + 9*x2^4*x3^2*x6^2 - 8*x2^2*x3*x4^2*x6*x8^2*x9^2 + 2*x2^2*x3*x6*x7^6 + 6*x4^4*x8^4*x9^4 + 4*x4^2*x7^6*x8^2*x9^2 + 9*x7^12) : I >= 0 := by
  let terms : List (Real $\times$ Real) := [ (0, 1), (0, x2^2*x3), (9, x1^3*x2*x3*x4/3 + x1*x5^2*x6 - 2*x2^2*x3*x6/9 + 5*x4^2*x8^2*x9^2/9 - x7^6/3), (77/9, 6*x1^3*x2*x3*x4/11 + x2^2*x3*x6 - 26*x4^2*x8^2*x9^2/77 + 3*x7^6/77), (173/77, 224*x1^3*x2*x3*x4/173 + x4^2*x8^2*x9^2 + 291*x7^6/173), (282/173, -13*x1^3*x2*x3*x4/282 + x7^6), (0, x3^2*x4^4), (193/282, x1^3*x2*x3*x4), (0, x4^2*x6*x8^2*x9^2), (0, x6*x7^6) ]
  have : I = (terms.map (fun (p, k) => p * k^2)).sum := by
    unfold terms
    simp only [List.map_cons, List.map_nil, List.sum_cons, List.sum_nil]
    linear_combination h1
  rw [this]
  unfold terms
  simp only [List.map_cons, List.map_nil, List.sum_cons, List.sum_nil]
  positivity
\end{lstlisting}  %\end{minted}
\end{tcolorbox}
\end{figure}

\begin{figure}[htbp]
\centering
\begin{tcolorbox}[colback=gray!5!white, colframe=blue!50!black, arc=3mm, title=
{\texttt{NSPI Achieves Solution While Baseline Methods \textbf{Fail}(2)}}
% {\texttt{Example 1 (NSPI Output)}}
]
\textbf{Polynomial:}\\
$7x_1^2 x_2^2 x_3^2 x_4^2 - 2x_1 x_5^2 x_2 x_3 x_4 x_6 x_7 x_8 - 4x_1 x_5 x_2 x_3 x_4 x_6 x_9^4 - 6x_1 x_2 x_3 x_4 x_6^5 x_8 - 10x_1 x_2 x_3 x_4 x_6 x_7^3 x_8^2 + 8x_5^4 x_6^2 x_7^2 x_8^2 - 4x_5^2 x_6^2 x_7^4 x_8^3 + 8x_5^2 x_6^2 x_9^8 - 4x_5 x_6^6 x_8 x_9^4 - 2x_5 x_6^2 x_7^3 x_8^2 x_9^4 + 5x_6^{10} x_8^2 + 6x_6^6 x_7^3 x_8^3 + 6x_6^2 x_7^6 x_8^4$
 % 7*x1^2*x2^2*x3^2*x4^2 - 2*x1*x5^2*x2*x3*x4*x6*x7*x8 - 4*x1*x5*x2*x3*x4*x6*x9^4 - 6*x1*x2*x3*x4*x6^5*x8 - 10*x1*x2*x3*x4*x6*x7^3*x8^2 + 8*x5^4*x6^2*x7^2*x8^2 - 4*x5^2*x6^2*x7^4*x8^3 + 8*x5^2*x6^2*x9^8 - 4*x5*x6^6*x8*x9^4 - 2*x5*x6^2*x7^3*x8^2*x9^4 + 5*x6^10*x8^2 + 6*x6^6*x7^3*x8^3 + 6*x6^2*x7^6*x8^4
 
\medskip

\textbf{Lean Proof:} 
% \begin{minted}[fontsize=\small, breaklines]{lean}
\begin{lstlisting}[language=lean4, frame=none]
import Mathlib
theorem f1 (x1 x2 x3 x4 x5 x6 x7 x8 x9 I : Real) (h1 :I = 7*x1^2*x2^2*x3^2*x4^2 - 2*x1*x2*x3*x4*x5^2*x6*x7*x8 - 4*x1*x2*x3*x4*x5*x6*x9^4 - 6*x1*x2*x3*x4*x6^5*x8 - 10*x1*x2*x3*x4*x6*x7^3*x8^2 + 8*x5^4*x6^2*x7^2*x8^2 - 4*x5^2*x6^2*x7^4*x8^3 + 8*x5^2*x6^2*x9^8 - 4*x5*x6^6*x8*x9^4 - 2*x5*x6^2*x7^3*x8^2*x9^4 + 5*x6^10*x8^2 + 6*x6^6*x7^3*x8^3 + 6*x6^2*x7^6*x8^4) : I >= 0 := by
  let terms : List (Real $\times$ Real) := [ (0, 1), (0, x3*x4^2), (0, x4^3*x6), (7, x1*x2*x3*x4 - x5^2*x6*x7*x8/7 - 2*x5*x6*x9^4/7 - 3*x6^5*x8/7 - 5*x6*x7^3*x8^2/7), (55/7, x5^2*x6*x7*x8 - 2*x5*x6*x9^4/55 - 3*x6^5*x8/55 - 19*x6*x7^3*x8^2/55), (0, x3^2*x4*x6*x7), (408/55, x5*x6*x9^4 - 79*x6^5*x8/204 - 139*x6*x7^3*x8^2/408), (257/408, -110*x6^5*x8/257 + x6*x7^3*x8^2), (633/257, x6^5*x8), (0, x4*x6^5), (0, x3^4*x5*x6), (0, x5^2*x6^2*x7^4*x8^3), (0, x6^6*x7^3*x8^3) ]
  have : I = (terms.map (fun (p, k) => p * k^2)).sum := by
    unfold terms
    simp only [List.map_cons, List.map_nil, List.sum_cons, List.sum_nil]
    linear_combination h1
  rw [this]
  unfold terms
  simp only [List.map_cons, List.map_nil, List.sum_cons, List.sum_nil]
  positivity
\end{lstlisting} %\end{minted}
\end{tcolorbox}
\end{figure}

\begin{figure}[htbp]
\centering
\begin{tcolorbox}[colback=gray!5!white, colframe=blue!50!black, arc=3mm, title=
{\texttt{NSPI Achieves Solution While Baseline Methods \textbf{Fail}(3)}}
% {\texttt{Example 1 (NSPI Output)}}
]
\textbf{Polynomial:}\\
$11x_1^4x_2^4x_3^2 - 4x_1^3x_2^3x_3 - 12x_1^3x_2^2x_3^3x_5x_6 - 6x_1^3x_2^2x_3^2x_5^5x_7 + 2x_1^2x_2^3x_3^3x_5 - 2x_1^2x_2^3x_3x_4 - 8x_1^2x_2^3x_3x_5x_7 - 2x_1^2x_2^2x_3x_7^4 + 10x_1^2x_2^2 + 4x_1^2x_2x_3^2x_5x_6 + 4x_1^2x_2x_3x_5^5x_7 + 12x_1^2x_3^4x_5^2x_6^2 + 2x_1^2x_3^3x_5^6x_6x_7 + 12x_1^2x_3^2x_5^{10}x_7^2 - 12x_1x_2^2x_3^2x_5 - 4x_1x_2^2x_4 - 2x_1x_2^2x_5x_7 - 6x_1x_2x_3^4x_5^2x_6 + 2x_1x_2x_3^3x_5^6x_7 + 8x_1x_2x_3^2x_4x_5x_6 + 2x_1x_2x_3^2x_5^2x_6x_7 + 4x_1x_2x_3x_4x_5^5x_7 - 6x_1x_2x_3x_5^6x_7^2 - 14x_1x_2x_7^4 + 10x_1x_3^2x_5x_6x_7^4 + 2x_1x_3x_5^5x_7^5 + 13x_2^2x_3^4x_5^2 + 4x_2^2x_3^2x_4x_5 - 12x_2^2x_3^2x_5^2x_7 + 10x_2^2x_4^2 + 14x_2^2x_5^2x_7^2 + 8x_2x_3^2x_5x_7^4 + 2x_2x_4x_7^4 + 4x_2x_5x_7^5 + 12x_7^8$ 
% 11*x1^4*x2^4*x3^2 - 4*x1^3*x2^3*x3 - 12*x1^3*x2^2*x3^3*x5*x6 - 6*x1^3*x2^2*x3^2*x5^5*x7 + 2*x1^2*x2^3*x3^3*x5 - 2*x1^2*x2^3*x3*x4 - 8*x1^2*x2^3*x3*x5*x7 - 2*x1^2*x2^2*x3*x7^4 + 10*x1^2*x2^2 + 4*x1^2*x2*x3^2*x5*x6 + 4*x1^2*x2*x3*x5^5*x7 + 12*x1^2*x3^4*x5^2*x6^2 + 2*x1^2*x3^3*x5^6*x6*x7 + 12*x1^2*x3^2*x5^10*x7^2 - 12*x1*x2^2*x3^2*x5 - 4*x1*x2^2*x4 - 2*x1*x2^2*x5*x7 - 6*x1*x2*x3^4*x5^2*x6 + 2*x1*x2*x3^3*x5^6*x7 + 8*x1*x2*x3^2*x4*x5*x6 + 2*x1*x2*x3^2*x5^2*x6*x7 + 4*x1*x2*x3*x4*x5^5*x7 - 6*x1*x2*x3*x5^6*x7^2 - 14*x1*x2*x7^4 + 10*x1*x3^2*x5*x6*x7^4 + 2*x1*x3*x5^5*x7^5 + 13*x2^2*x3^4*x5^2 + 4*x2^2*x3^2*x4*x5 - 12*x2^2*x3^2*x5^2*x7 + 10*x2^2*x4^2 + 14*x2^2*x5^2*x7^2 + 8*x2*x3^2*x5*x7^4 + 2*x2*x4*x7^4 + 4*x2*x5*x7^5 + 12*x7^8

\medskip

\textbf{Lean Proof:} 
% \begin{minted}[fontsize=\small, breaklines]{lean}
\begin{lstlisting}[language=lean4, frame=none]
import Mathlib
theorem f1 (x1 x2 x3 x4 x5 x6 x7 I : Real) (h1 :I = 11*x1^4*x2^4*x3^2 - 4*x1^3*x2^3*x3 - 12*x1^3*x2^2*x3^3*x5*x6 - 6*x1^3*x2^2*x3^2*x5^5*x7 + 2*x1^2*x2^3*x3^3*x5 - 2*x1^2*x2^3*x3*x4 - 8*x1^2*x2^3*x3*x5*x7 - 2*x1^2*x2^2*x3*x7^4 + 10*x1^2*x2^2 + 4*x1^2*x2*x3^2*x5*x6 + 4*x1^2*x2*x3*x5^5*x7 + 12*x1^2*x3^4*x5^2*x6^2 + 2*x1^2*x3^3*x5^6*x6*x7 + 12*x1^2*x3^2*x5^10*x7^2 - 12*x1*x2^2*x3^2*x5 - 4*x1*x2^2*x4 - 2*x1*x2^2*x5*x7 - 6*x1*x2*x3^4*x5^2*x6 + 2*x1*x2*x3^3*x5^6*x7 + 8*x1*x2*x3^2*x4*x5*x6 + 2*x1*x2*x3^2*x5^2*x6*x7 + 4*x1*x2*x3*x4*x5^5*x7 - 6*x1*x2*x3*x5^6*x7^2 - 14*x1*x2*x7^4 + 10*x1*x3^2*x5*x6*x7^4 + 2*x1*x3*x5^5*x7^5 + 13*x2^2*x3^4*x5^2 + 4*x2^2*x3^2*x4*x5 - 12*x2^2*x3^2*x5^2*x7 + 10*x2^2*x4^2 + 14*x2^2*x5^2*x7^2 + 8*x2*x3^2*x5*x7^4 + 2*x2*x4*x7^4 + 4*x2*x5*x7^5 + 12*x7^8) : I >= 0 := by
  let terms : List (Real $\times$ Real) := [ (0, 1), (10, -x1^2*x2^2*x3/10 - x1*x2/5 + 2*x1*x3^2*x5*x6/5 + x1*x3*x5^5*x7/5 + x2*x3^2*x5/5 + x2*x4 + x7^4/10), (48/5, -11*x1^2*x2^2*x3/48 + x1*x2 + 7*x1*x3^2*x5*x6/24 + x1*x3*x5^5*x7/4 - 7*x2*x3^2*x5/12 - 5*x2*x5*x7/48 - 17*x7^4/24), (667/48, -7*x1^2*x2^2*x3/23 + 62*x1*x3^2*x5*x6/667 - 132*x1*x3*x5^5*x7/667 - 316*x2*x3^2*x5/667 + x2*x5*x7 + 62*x7^4/667), (0, x2*x3*x5), (9289/1334, -2755*x1^2*x2^2*x3/9289 + 8622*x1*x3^2*x5*x6/9289 + 3676*x1*x3*x5^5*x7/9289 + 594*x2*x3^2*x5/9289 + x7^4), (57461/9289, -18159*x1^2*x2^2*x3/57461 - 18281*x1*x3^2*x5*x6/57461 + 4839*x1*x3*x5^5*x7/57461 + x2*x3^2*x5), (0, x3^4), (163073/57461, -187911*x1^2*x2^2*x3/163073 + x1*x3^2*x5*x6 - 151490*x1*x3*x5^5*x7/163073), (670239/163073, x1^2*x2^2*x3 - 838684*x1*x3*x5^5*x7/670239), (0, x1*x3*x5^5), (292711/670239, x1*x3*x5^5*x7) ]
  have : I = (terms.map (fun (p, k) => p * k^2)).sum := by
    unfold terms
    simp only [List.map_cons, List.map_nil, List.sum_cons, List.sum_nil]
    linear_combination h1
  rw [this]
  unfold terms
  simp only [List.map_cons, List.map_nil, List.sum_cons, List.sum_nil]
  positivity
\end{lstlisting} %\end{minted}
\end{tcolorbox}
\end{figure}

% You can have as much text here as you want. The main body must be at most $8$ pages long.
% For the final version, one more page can be added.
% If you want, you can use an appendix like this one.  

% The $\mathtt{\backslash onecolumn}$ command above can be kept in place if you prefer a one-column appendix, or can be removed if you prefer a two-column appendix.  Apart from this possible change, the style (font size, spacing, margins, page numbering, etc.) should be kept the same as the main body.
%%%%%%%%%%%%%%%%%%%%%%%%%%%%%%%%%%%%%%%%%%%%%%%%%%%%%%%%%%%%%%%%%%%%%%%%%%%%%%%
%%%%%%%%%%%%%%%%%%%%%%%%%%%%%%%%%%%%%%%%%%%%%%%%%%%%%%%%%%%%%%%%%%%%%%%%%%%%%%%